\documentclass[journal]{IEEEtran}
\usepackage{amsmath,amsfonts}
\usepackage{algorithmic}
\usepackage{algorithm}
\usepackage{array}
\usepackage{textcomp}
\usepackage{stfloats}
\usepackage{url}
\usepackage{verbatim}
\usepackage{graphicx}
\usepackage{cite}
\usepackage{makecell}
\usepackage{amsmath, amssymb, epsfig, graphicx, bm}
\usepackage{cases}
\usepackage{float}
\usepackage{subfigure}
\usepackage{multirow}
\usepackage{graphicx}
\usepackage{array}
\usepackage{color}
\usepackage{epstopdf}
\usepackage{amsthm}
\usepackage{amsfonts}
\usepackage{cite, url}
\usepackage[all,cmtip]{xy}
\usepackage{flushend}
\usepackage[normalem]{ulem} % to use \sout
\usepackage{hyperref}
\usepackage{bbding}
\RequirePackage[
  color=blue!20!white,
  textsize=tiny,
  colorinlistoftodos,
  prependcaption,
]{todonotes}

\newtheorem{theorem}{Theorem}
\newcommand{\E}{{\mathbb{E}}}
\newcommand{\R}{{\mathbb{R}}}
\newcommand{\V}{{\mathbb{V}}}

\newtheorem{assumption}{Assumption}
\newtheorem{lemma}{Lemma}
\newtheorem{remark}{Remark}
\newtheorem{corollary}{Corollary}

\newcommand{\cC}{\mathcal{C}}
\newcommand{\cD}{\mathcal{D}}

\newcommand{\cF}{\mathcal{F}}
\newcommand{\cO}{\mathcal{O}}

\newcommand{\cN}{\mathcal{N}}

\newcommand{\sign}{\emph{sign}}
\newcommand{\bI}{\mathbb{I}}
\newcommand{\bN}{\mathbb{N}}
\newcommand{\smallCE}[1]{$<$1e-07}

\begin{document}

\title{Single-Timescale Multi-Sequence Stochastic Approximation Without Fixed Point Smoothness: Theories and Applications}

\author{
Yue~Huang,
Zhaoxian~Wu,
Shiqian~Ma,
and~Qing~Ling
\thanks{Yue Huang, Zhaoxian Wu, and Qing Ling are with the School of Computer Science and Engineering, Sun Yat-Sen
University.}
\thanks{Shiqian Ma is with the Department of Computational Applied Math and Operations Research, Rice University.}
}

\markboth{}%
{Shell \MakeLowercase{\textit{et al.}}: A Sample Article Using IEEEtran.cls for IEEE Journals}

\maketitle

\begin{abstract}
Stochastic approximation (SA) that involves multiple coupled sequences, known as multiple-sequence SA (MSSA), finds diverse applications in the fields of signal processing and machine learning. However, existing theoretical understandings {of} MSSA are limited: the multi-timescale analysis implies a slow convergence rate, whereas the single-timescale analysis relies on a stringent fixed point smoothness assumption. This paper establishes tighter single-timescale analysis for MSSA, without assuming smoothness of the fixed points. Our theoretical findings reveal that, when all involved operators are strongly monotone, MSSA converges at a rate of $\Tilde{\mathcal{O}}(K^{-1})$, where $K$ denotes the total number of iterations. In addition, when all involved operators are strongly monotone except for the main one, MSSA converges at a rate of $\mathcal{O}(K^{-\frac{1}{2}})$. These theoretical findings align with those established for single-sequence SA. Applying these theoretical findings to bilevel optimization and communication-efficient distributed learning offers relaxed assumptions and/or simpler algorithms with performance guarantees, as validated by numerical experiments.
\end{abstract}


\section{Introduction}
\label{sec:intro}

\IEEEPARstart{S}{tochastic} approximation (SA) aims at finding a zero of one or more operators, for which we can only observe noisy estimates. Its simplest form, called as single-sequence SA (SSSA), seeks a point $x^\ast$ satisfying $v(x^\ast)=0$ with $v:$ $\R^{d_0} \to\R^{d_0}$ being an operator, through
\begin{gather}\label{SSSA}
    x^{k+1}=x^k-\alpha^k(v(x^k)+\xi^k).
\end{gather}
Here $\alpha^k$ is a positive step size and $\xi^k$ is random noise.
Due to its broad applications in statistics, optimization, machine learning, and signal processing \cite{borkar2009stochastic,hastie2009elements,bottou2018optimization,sutton2018reinforcement,dieuleveut2023stochastic}, SA has gained widespread popularity since the 1950s \cite{robbins1951stochastic}. However, a number of emerging applications, including but not limited to bilevel optimization, meta learning, and reinforcement learning \cite{shen2022single,yang2019multilevel,sato2021gradient,sun2021optimal,chen2021solving}, involve the more complicated multi-sequence SA (MSSA) that aims at finding a zero of a system with two or more coupled operators. MSSA also appears when we combine the popular momentum acceleration techniques with SA updates \cite{deb2021n}, for example, in communication-compressed and momentum-accelerated distributed learning.

Consider $N+1$ operators of $v:\R^{d_0}\times\R^{d_1}\cdots\times\R^{d_N}\to\R^{d_0}$ and $h_n:\R^{d_0}\times\R^{d_1}\cdots\times\R^{d_n}\to\R^{d_n}$ for all $n\in[N]$. Our task is to find $x^\ast \in \R^{d_0}$, $y_1^\ast \in \R^{d_1}$, $\dots$, $y_N^\ast \in \R^{d_N}$ such that
\begin{subequations}\label{solution}
    \begin{align}
        &h_n(x^\ast,y_1^\ast,\dots,y_n^\ast)=0,~~~\forall n\in[N],\\
        &v(x^\ast,y_1^\ast,\dots,y_N^\ast)=0.
    \end{align}
\end{subequations}
To do so, the MSSA updates are given by
\begin{subequations}\label{MSSA}
    \begin{align}
        &y_n^{k+1}=y_n^k-\beta_{n}^k(h_n(x^k,y_{1}^k,\dots,y_n^k)+\psi_n^k),~\forall n\in[N],\\
        &x^{k+1}=x^{k}-\alpha^{k}(v(x^k,y_1^k,\dots,y_N^k)+\xi^k),
    \end{align}
\end{subequations}
where $\alpha^k$ and $\beta_{n}^k$ are two positive step sizes, while $\xi^k$ and $\psi_n^k$ are random noise. In this paper, $v$ is referred to as the main operator, while $h_n$ for all $n\in[N]$ are the secondary operators. Similarly, $\{x^k\}_k$ is referred to as the main sequence, while $\{y_n^k\}_k$ for all $n\in[N]$ are the secondary sequences.

\textbf{Related Works.}
Asymptotic convergence of SSSA has been established in the seminal work of \cite{robbins1951stochastic}. One of the popular analytical tools is to show that the SSSA iterates asymptotically converge to the solution of an associate ordinary differential equation (ODE) \cite{ljung1977analysis,harold1997stochastic}. Finite-time convergence of SSSA has also been studied in recent years \cite{moulines2011non,bhandari2018finite,chen2020finite}, showing that under the strong monotonicity assumption of the operator, SSSA converges at a rate of $\cO(K^{-1})$ where $K$ is the total number of iterations. Without the strong monotonicity assumption, the rate becomes $\cO(K^{-\frac{1}{2}})$.

For MSSA with $N=1$ that is also termed as two-sequence SA (TSSA), \cite{kaledin2020finite} establishes a convergence rate of $\cO(K^{-1})$ when the operators are linear. The analysis of nonlinear TSSA involves two scenarios: two-timescale with $\lim_{k\to\infty}{\alpha^k}/{\beta_1^k}$ $=0$, as well as single-timescale with $\beta_1^k=\Theta(\alpha^k)$. For the two-timescale scenario, \cite{mokkadem2006convergence} and \cite{doan2022nonlinear} establish a convergence rate of $\mathcal{O}(K^{-\frac{2}{3}})$, which is slower than the $\mathcal{O}(K^{-1})$ rates of SSSA and linear TSSA. For the single-timescale scenario, \cite{shen2022single} establishes a convergence rate of $\mathcal{O}(K^{-1})$. Nevertheless, \cite{shen2022single} requires a smoothness assumption on the fixed points, which is uncommon in applications such as bilevel optimization. The recent work of \cite{doan2024fast} applies the momentum acceleration technique to both sequences in TSSA and achieves a convergence rate of $\mathcal{O}(K^{-1})$ without assuming the smoothness of the fixed points. All of these results are established assuming the strong monotonicity of the operators.

% and under an assumption of strong monotonicity

For general MSSA with $N>1$, \cite{deb2021n} analyzes its asymptotic convergence based on the ODE approach. There are only few works to analyze its finite-time convergence.
The work of \cite{yang2019multilevel} considers a stochastic gradient method to solve a multi-level stochastic composite optimization problem, which is a special case of MSSA. The multi-timescale scenario is investigated, namely, $\lim_{k\to\infty}{\alpha^k}/{\beta_1^k}=0$ and $\lim_{k\to\infty}{\beta^k_{n}}/{\beta^k_{n+1}}=0$ for all $n\in[N-1]$, leading to a suboptimal convergence rate of $\cO(K^{-4/(4+N)})$ that depends on $N$. For the single-timescale scenario with $\beta_n^k=\Theta(\alpha^k)$ for all $n\in[N]$, \cite{shen2022single} establishes a convergence rate of $\cO(K^{-1})$. However, this convergence guarantee requires the stringent smoothness assumption on the fixed points, in addition to the strong monotonicity assumption on the operators. When the main operator is not strongly monotone, the convergence rate is $\cO(K^{-\frac{1}{2}})$.

\begin{table}[tb!]
    \centering
    \caption{Comparison with the existing works. The third column marks whether the guarantee is for MSSA or TSSA. The fourth column marks whether the analysis requires the smoothness assumption on the fixed points.
    }
    \label{tab:related_work}
    \begin{tabular}{c|c|c|c}
        \hline\hline
         & Convergence Rate & $N>1$ & \makecell[c]{Smoothness of Fixed Points} \\
        \hline
        \makecell[c]{\cite{doan2024fast}} & \makecell[c]{$\cO(K^{-1})$} & \makecell[c]\XSolidBrush & \XSolidBrush\\
        \hline
        \makecell[c]{\cite{deb2021n}} & \makecell[c]{Asymtotic covergence} & \makecell[c]\Checkmark & \XSolidBrush\\
        \hline
        \makecell[c]{\cite{yang2019multilevel}} & \makecell[c]{$\cO(K^{-4/(4+N)})$} & \makecell[c]{\Checkmark} &  \XSolidBrush \\
        \hline
        \cite{shen2022single} & $\cO(K^{-1})$ & \Checkmark & \Checkmark\\
        \hline
        Ours & $\tilde\cO(K^{-1})$ & \Checkmark & \XSolidBrush\\
        \hline\hline
    \end{tabular}
\end{table}

Considering the fact that the existing multi-timescale analysis of MSSA implies a slow convergence rate, whereas the single-timescale analysis requires the smoothness assumption on the fixed points, we ask:
\textit{Is it possible to establish a fast convergence rate for single-timescale MSSA without assuming smoothness of the fixed points?} We give an affirmative answer to this question in this paper. We compare our results with those of the existing works in Table \ref{tab:related_work}.

\textbf{Our contributions.} Our contributions are as follows.

\noindent \textbf{C1)} We establish tighter single-timescale analysis for MSSA, without assuming smoothness of the fixed points. Our theoretical findings indicate that, when all involved operators are strongly monotone, MSSA can converge at a rate of $\Tilde{\cO}(K^{-1})$. When all involved operators are strongly monotone except for the main one, MSSA converges at a rate of $\cO(K^{-\frac{1}{2}})$. These theoretical findings align with those established for SSSA.

\noindent \textbf{C2)} We apply our tighter single-timescale analysis for MSSA in a variety of applications. First, applying it to the popular bilevel optimization algorithm SOBA \cite{dagreou2022framework}, we can remove the high-order Lipschitz continuity assumption that corresponds to the smoothness assumption on the fixed points in MSSA, but obtain the same convergence rate. Second, we apply our analysis to communication-compressed and momentum-accelerated distributed learning, providing a novel perspective {on} understanding a class of communication-efficient distributed learning algorithms.

\noindent \textbf{C3)} We also conduct numerical experiments on a data hyper-cleaning task that can be formulated as a bilevel optimization problem, and a communication-efficient distributed learning task with communication compression and momentum acceleration, validating our theoretical findings.

% \todo{We can add a table to compare our analysis with others}

Compared to the short, preliminary conference version \cite{huang2024convergence}, this paper provides convergence analysis for MSSA under a milder assumption on random noise, which enlarges the class of applications of MSSA. This improvement allows us to apply our theoretical findings to the new problem of communication-compressed and {momentum-accelerated} distributed learning; see section \ref{sec:compress}. We also conduct numerical experiments {to} verify our theoretical findings. All proofs are provided in the appendices of this paper.

\noindent \textbf{Notation.} Given functions $f, g:\mathcal{X} \to [0,\infty)$, where $\mathcal{X}\subset\R$, we say $f = \mathcal{O}(g)$ if there exists constants $c<\infty$ and $M$ $>0$ such that $f(x) \leq c  g(x)$ for all $x\ge M$; $f = \Omega(g)$ if there exists a constant $c > 0$ and $M>0$ such that $f(x)$ $\geq c g(x)$ for all $x\ge M$; $f=\Theta(g)$ if $f=\mathcal{O}(g)$ and $f=\Omega(g)$ hold simultaneously. We denote $f=\cO(g\max\{1,\log g\})$ as $f=\Tilde\cO(g)$. We use $\|\cdot\|$ to denote the $\ell_2$-norm. We define $\cF^k:=\sigma\left(\bigcup_{l=0}^k\{x^k,\,y_1^k,\dots,y_N^k\}\right)$, in which $\sigma(\mathcal{X})$ is the $\sigma$-algebra generated by $\mathcal{X}$. We use $y_{1:n}$ to collect $y_{1}, \dots, y_n$, $[N]$ to collect the integers from $1$ to $N$, and so as $[K]$.

\section{Finite-Time Convergence of MSSA}
\label{sec:main}
We start from several standard assumptions.
\begin{assumption}[Strong monotonicity of any $h_n$]\label{asp-sub-sm}
    For any $n\in[N]$, given $x$ and $y_{1:n-1}$, there exists a positive constant $\mu_n$ such that for any $\hat{y}_n$ and $\bar{y}_n$, it holds that
    \begin{align*}
        &\langle h_n(x,y_{1:n-1},\hat{y}_n)\!-\!h_n(x,y_{1:n-1},\bar{y}_n),\hat{y}_n\!-\!\bar{y}_n\rangle\!\geq\!\mu_n\|\hat{y}_n\!-\!\bar{y}_n\|^2.
    \end{align*}
\end{assumption}

\begin{assumption}[Lipschitz continuity of any $h_n$]\label{asp-lc-sub}
    For any $n\in[N]$, $h_n(\cdot)$ satisfies one of the following conditions:\\
    $(a)$ $h_n(x,y_{1:n})$ is $\ell_n$-Lipschitz continuous w.r.t. any variable among $x$ and $y_{1:n}$.\\
    $(b)$ $h_n(x,y_{1:n})=A_n(x,y_{1:n-1})y_n+b_n(x,y_{1:n-1})$, in which $A_n:$ $\R^{d_0\times d_1\times\cdots\times d_{n-1}}\to\R^{d_n\times d_n}$ is $\ell_{A,n}$-Lipschitz continuous w.r.t. any variable among $x$ and $y_{1:n-1}$, and $\ell_{A,n}^\prime$-bounded in terms of the spectral norm;
    $b_n:\R^{d_0\times d_1\times\cdots\times d_{n-1}}\to\R^{d_n}$ is $\ell_{b,n}$-Lipschitz continuous w.r.t. any variable among $x$ and $y_{1:n-1}$, and $\ell_{b,n}^\prime$-bounded in term of $\ell_2$-norm.
\end{assumption}

% \todo{Explain condition (b)? }
Condition $(a)$ in Assumption \ref{asp-lc-sub}
requires Lipschitz continuity of $h_n(x,y_{1:n})$ w.r.t. any variable among $x$ and $y_{1:n}$, which is mild and widely adopted in existing works \cite{shen2022single,doan2022nonlinear,doan2024fast}. However, this condition might not hold in some applications, for example, bilevel optimization; see Section \ref{sec:BO}. Therefore, we relax the assumption and add Condition $(b)$ to enlarge the investigated function class.
When $y_n$ is bounded, Condition $(b)$ implies $(a)$ trivially.

\begin{lemma}[Existence, uniqueness and Lipschitz continuity of fixed points]\label{lma:Lip-fp}
    Suppose Assumptions \ref{asp-sub-sm} and \ref{asp-lc-sub} hold. For any $n$, given any $x,y_{1:n-1}$, there exists a unique fixed point $y_n^\ast(x,y_{1:n-1})$ satisfying
    \begin{align}
        h_n(x,y_{1:n-1},y_n^\ast(x,y_{1:n-1}))=0.
    \end{align}
    In addition, $y_n^\ast(x,y_{1:n-1})$ is $L_{y,n}$-Lipschitz continuous w.r.t. any variable among $x$ and $y_{1:n-1}$, where $L_{y,n}=\frac{\ell_n}{\mu_n}$ for Condition (a) in Assumption \ref{asp-lc-sub}, and $L_{y,n}=\frac{\ell_{b,n}}{\mu_n}+\frac{\ell_{b,n}^\prime\ell_{A,n}}{\mu_n^2}$ for Condition (b) in Assumption \ref{asp-lc-sub}.
\end{lemma}
\begin{remark}\label{rmk-lip-fp}
    As shown in Lemma \ref{lma:Lip-fp}, Assumptions \ref{asp-sub-sm} and \ref{asp-lc-sub} imply the existence and uniqueness of the fixed point of the operator $h_n(x,y_{1:n-1},\cdot)$, denoted by $y_n^\ast(x,y_{1:n-1})$, for any given $x$ and $y_{1:n-1}$. Letting $y^\diamond_1(x) = y^\ast_1(x)$, we define
    \begin{align}\label{eq:y-diamond}
        y^\diamond_n(x)&:=y^\ast_n(x,\,y^\diamond_1(x),\dots,\,y^\diamond_{n-1}(x)),\quad \forall n\in[N].
    \end{align}
    Finding the fixed points satisfying \eqref{solution} is equivalent to finding $x^\ast$ and $y^\diamond_n(x^\ast)$ for all $n \in[N]$ such that $v(x^\ast,y^\diamond_{1:N}(x^\ast))=0$, where $y^\diamond_{1:N}(x)$ collects $y^\diamond_{1}(x),\cdots,y^\diamond_N(x)$.

    Lemma \ref{lma:Lip-fp} also shows that Assumptions \ref{asp-sub-sm} and \ref{asp-lc-sub} ensure Lipschitz continuity of the fixed points $y^\ast_n(\cdot)$. Note that \cite{shen2022single} further assumes smoothness of the fixed points, namely, $y_n^\ast(\cdot)$ are differentiable and $\nabla y_n^\ast(\cdot)$ are Lipschitz continuous, in addition to Lipschitz continuity of the fixed points.
\end{remark}

\begin{assumption}[Lipschitz continuity of $v$]\label{asp-lc-main} There exists a positive constant $\ell_0$ such that for any $\hat x$, $\bar x$ and $y_{1:N}$, it holds that
    \begin{align}
        &\|v(\hat x,y^\diamond_{1:N}(\hat x))-v(\bar x,y_{1:N})\|\notag\\
        \le\ &\ell_0(\|\hat x-\bar x\|+\sum_{n=1}^N\|y_n-y_n^\diamond(\hat x)\|).\notag
    \end{align}
\end{assumption}

\begin{remark}\label{rmk-Lv}
    Assumption \ref{asp-lc-main} is weaker than assuming Lipschitz continuity of $v(x,y_{1:N})$ w.r.t. any variable among $x$ and $y_{1:N}$. Together with Assumptions \ref{asp-sub-sm} and \ref{asp-lc-sub}, it implies $v(x,y_{1:N}^\diamond(x))$ is $L_v$-Lipschitz continuous w.r.t. variable $x$, see Appendix \ref{sec-apx-aux} for the formal statement.
\end{remark}

\begin{assumption}[Martingale difference sequences of noise]\label{asp-noise}
    $\{\xi^k\}_k$ and $\{\psi^k_n\}_k$ for any $n\in[N]$ are martingale difference sequences, whose variances are bounded as follows:
    \begin{subequations}\label{eq-bound-var}
        \begin{align}
            \operatorname{Var}[\psi_n^k|\cF^k]\leq&\ \omega_0\|h_n(x^k,y_{1:n}^k)\|^2+\sigma^2,\quad\forall n\in[N]\\
            \operatorname{Var}[\xi^k|\cF^k]\leq&\ \omega_1\|v(x^k,y_{1:N}^k)\|^2+\sigma^2\notag\\
            &+\omega_2\sum_{n=1}^N\|h_n(x^k,y_{1:n}^k)\|^2,
        \end{align}
    \end{subequations}
    where $\omega_0,\,\omega_1,\,\omega_2$ and $\sigma$ are positive constants.
\end{assumption}

Assumptions \ref{asp-sub-sm}--\ref{asp-lc-main} are common in the literature \cite{shen2022single,doan2022nonlinear,doan2024fast}. Note that with Assumption \ref{asp-noise}, any two different elements in $\{\xi^k\}_k$ or $\{\psi^k_n\}_k$ for all $n\in[N]$ are uncorrelated. Besides, the requirement of \eqref{eq-bound-var} is weaker than the uniformly bounded variance assumption made in \cite{shen2022single,doan2022nonlinear,doan2024fast} (that is, \eqref{eq-bound-var} with {$\omega_0,\omega_1,\omega_2=0$}).

\subsection{Challenges in Analyzing MSSA}

To show the challenges in analyzing MSSA, we consider the simplest TSSA case of $N=1$. We expect that the pairs $(x^K,y_1^K)$ generated by \eqref{MSSA} converge to $(x^\ast,y_1^\ast(x^\ast))$ satisfying \eqref{solution} at a rate of $\cO(\frac{1}{K})$. Nevertheless, this is nontrivial because: (i) When $(x^k,y_1^k)$ is close to $(x^\ast,y_1^\ast(x^\ast))$ for any $k\in[K]$, $v(x^k,y_1^k)$ is small such that the random noise $\xi^k$ dominates the increment of $x^k$. This inevitably causes the oscillations of $x^k$ if we choose a sufficiently large step size $\alpha^k$ that enables fast convergence. The same argument also applies to $y_1^k$. (ii) Since the sequences $\{x^k\}_k$ and $\{y_1^k\}_k$ are coupled, the oscillations of $x^k$ enlarge the oscillations of $y_1^k$, and vice versa.

One remedy is to introduce an auxiliary pair $(x^k,y_1^\ast(x^k))$ to bridge $(x^k,y_1^k)$ and $(x^\ast,y_1^\ast(x^\ast))$. Then we need to show: (i) $y^k$ can track the time-varying fixed point $y_1^\ast(x^k)$ for any $k\in[K]$; (ii) The tracking error vanishes at a sufficiently fast rate so that $y^K$ can converge to $y^\ast(x^K)$ at a rate of $\cO(\frac{1}{K})$. If $y^\ast(x^k)$ varies abruptly, such a tracking task becomes impossible. To ensure that the fixed point variation $y^\ast(x^{k+1})-y^\ast(x^{k})$ is sufficiently small, \cite{mokkadem2006convergence,doan2022nonlinear} use a very small step size $\alpha^k=o(\beta_1^k)$, which eventually leads to a rate slower than $\cO(\frac{1}{K})$. An alternative is to assume smoothness of the fixed point $y^\ast(\cdot)$, which allows a properly large step size with $\alpha^k=\Theta(\beta_1^k)$ \cite{shen2022single}. However, this assumption does not necessarily hold in many applications such as bilevel optimization.

Our work directly characterizes the distance between {$y_1^K$ and $y_1^\ast(x^K)$}, as demonstrated in Figure \ref{fig:analysis} and detailed in Section \ref{subsec:analysis-y}. Therefore, we no longer require the fixed point variation $y_1^\ast(x^k)-y_1^\ast(x^{k-1})$ to be sufficiently small, which enables us to provide a single-timescale analysis without assuming smoothness of the fixed points.

\begin{figure}
    \centering
    \includegraphics[width=0.99\linewidth]{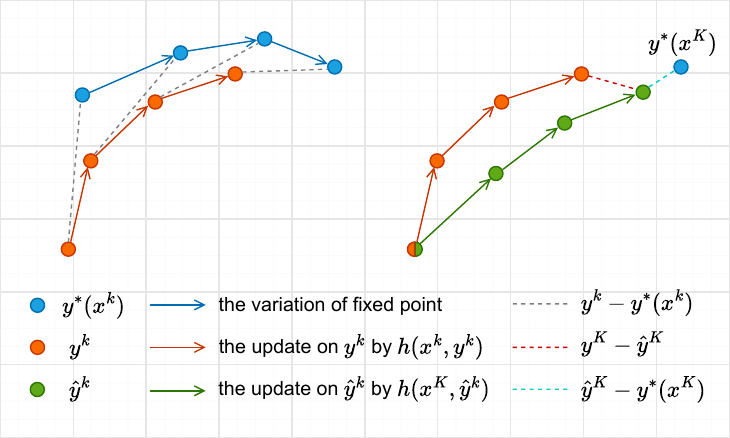}
    \caption{Analytic ideas are illustrated for two sequences with 3 iterations ($N=1,\, K=3$). Previous works \cite{doan2022nonlinear,shen2022single} traces the distance between $y^k$ and time-varying fixed point $y^\ast(x^k)$ (left). Our work characterizes the convergence error $y^K-y^\ast(x^K)$ based on an auxiliary sequence \eqref{eq:SSSA} (right). Here we use $\hat y^k$ to denote the sequence generated by \eqref{eq:SSSA}.}
    \label{fig:analysis}
\end{figure}

\subsection{Analysis for Secondary Sequences}\label{subsec:analysis-y}
At time $K$, the goal of MSSA $h_n(x^K,y_{1:n}^K)=0$ implies that $y^K=y^\ast_n(x^K,y_{1:n-1}^K)$. Therefore, we can use $\|y^K-y^\ast_n(x^K,y_{1:n-1}^K)\|^2$ to measure the convergence of secondary sequence. We will present an upper bound of this convergence metric in this section.

For any $n\in[N]$, given any fixed $x^K,\, y_1^K\dots,\, y_{n-1}^K$, it is well known that single sequence SA
\begin{align}\label{eq:SSSA}
    y_n^{k}=y_n^{k-1}-\beta_n^k(h_n(x^K,y_{1:n-1}^K,y_n^k)+\psi_n^k)
\end{align}
can converge to the fixed point $y^\ast_n(x^K,y_{1:n-1}^K)$ under the assumptions we make \cite{moulines2011non}. The main difference between the sequences $\{y_n^k\}_{k=0}^K$ generated by MSSA \eqref{MSSA} and SSSA \eqref{eq:SSSA} is that at time $k<K$, $y_n^k$ in MSSA is updated by $h_n(x^k,y^k_{1:n})$ instead of $h_n(x^K,y^K_{1:n-1},y^k_n)$.
The bias from $h_n(x^k,y^k_{1:n})$ will affect the convergence of $y_n^k$ to $y^\ast_n(x^K,y_{1:n-1}^K)$. However, thanks to the Lipschitz continuity of $h_n(\cdot)$, this bias can be bounded by the distance between $(x^k,\,y_{1:n-1}^k)$ and $(x^K,\,y_{1:n-1}^K)$. Therefore, the distance between $y^K_n$ and $y^{\ast}_n(x^K,y_{1:n-1}^K)$ can be bounded as follows.
\begin{lemma}[Convergence of secondary sequences]\label{lma:conv-sub}
    Suppose Assumptions \ref{asp-sub-sm}--\ref{asp-noise} hold. Consider the MSSA sequences generated by \eqref{MSSA}. If the step sizes are chosen as $\beta_n^k=\beta_n\le\frac{\mu_n}{\ell_n^2}$ for any $n\in[N]$ and any $K$, then it hold that
    \begin{align}
        &\E\|y^{K}_n-y^{\ast}_n(x^K,y_{1:n-1}^K)\|^2\notag\\
        \le&2\ell_{n}^2\beta_n(\beta_n+\frac{2}{\mu_n})\sum_{k=0}^{K-1}(1-\frac{\mu_n}{2}\beta_n)^{K-k-1}\text{Dist}_n(K,k)\notag\\
        &+2(1-\frac{\mu_n}{2}\beta_n)^K\E\|y_n^0-y^{\ast}_n(x^K,y_{1:n-1}^K)\|^2\notag\\
        &+\sum_{k=0}^{K-1}2(1-\mu_n\beta_n)^{K-1-k}\beta^2_n\operatorname{Var}[\psi_n^k],\label{eq:lma-conv-sub1}
    \end{align}
    where for any $n\in[N]$ and any $K,k\in\bN$, $\text{Dist}_n(K,k)$ is defined as
    \begin{align}
        \text{Dist}_n(K,k):=\E\|x^K-x^k\|^2+\sum_{i=1}^{n-1}\E\|y_i^K-y_i^{k}\|^2.\label{eq:lma-conv-sub2}
    \end{align}
\end{lemma}
Lemma \ref{lma:conv-sub} shows the convergence of $y_n^K$ to $y^{\ast}_n(x^K,y_{1:n-1}^K)$. The first term at the right-hand side of \eqref{eq:lma-conv-sub1} represents the effect of the bias from $h_n(x^k,y^k_{1:n})$ to $h_n(x^K,y^K_{1:n-1},y^k_n)$ as we discussed above. The second term represents the initial bias, the distance between the initial point $y_n^0$ and the convergence target $y_n^\ast(x^K,y_{1:n-1}^K)$.
The last term of \eqref{eq:lma-conv-sub1} stands for the effect of random noises $\psi_n^k$.

\subsection{Analysis for Strongly Monotone Main Operator}\label{sec-sm}
Now we analyze the scenario where the main operator is strongly monotone.
\begin{assumption}[Strong monotonicity of $v$]\label{asp:main-sm}
    There exists a unique fixed point $x^\ast$ such that $v(x^\ast,y^\diamond_{1:N}(x^\ast))=0$. Further, there exists a positive constant $\mu_0$ s.t. for any $x$, it holds that
    \begin{gather}
        \langle v(x,y^\diamond_{1:N}(x)),\,x-x^\ast\rangle\geq\mu_0\|x-x^\ast\|^2. \notag
    \end{gather}
\end{assumption}
Assumption \ref{asp:main-sm} is common in the literature \cite{shen2022single,mokkadem2006convergence,doan2022nonlinear,doan2024fast}. Compared to Assumption \ref{asp-sub-sm} in which the strong monotonicity is defined on two arbitrary points, Assumption \ref{asp:main-sm} is defined on an arbitrary point $x$ and the fixed point $x^\ast$ that is unique.

As we have discussed in Remark \ref{rmk-lip-fp}, the goal of MSSA is equivalent to finding $x^\ast$ and $y^\diamond_n(x^\ast)$ for all $n \in[N]$ such that $v(x^\ast,y^\diamond_{1:N}(x^\ast))=0$. Consider the following SSSA
\begin{align}\label{eq-sm-SSSA}
    x^{k+1}=x^k-\alpha^k(v(x^k,y^\diamond_{1:N}(x^k))+\xi^k).
\end{align}
The convergence property, saying that $\{x^k\}_k$ generated by \eqref{eq-sm-SSSA} converges to $x^\ast$ satisfying $v(x^\ast,y^\diamond_{1:N}(x^\ast))=0$, has been well studied under the strong monotonicity of $v$ \cite{moulines2011non}. The only difference between the main sequence in MSSA \eqref{MSSA} and \eqref{eq-sm-SSSA} is that MSSA updates $x^k$ by $v(x^k,y_{1:N}^k)$ instead of $v(x^k,y^\diamond_{1:N}(x^k))$. Therefore, when following the analysis on \eqref{eq-sm-SSSA} to analyze the convergence of the main sequence in MSSA, the only additional requirement is to characterize the bias from $v(x^k,y_{1:N}^k)$ to $v(x^k,y^\diamond_{1:N}(x^k))$, which can be controlled by the Lipschitz continuity of $v$ and $y^\diamond_{1:N}$, and intuitively, is related to the convergence of the secondary sequences. Inspired by these facts, we prove the convergence of the main sequence as follows.

\begin{lemma}\label{lma-sm-main}
    Consider the MSSA sequences generated by \eqref{MSSA}. Suppose that Assumptions \ref{asp-sub-sm}--\ref{asp:main-sm} hold. For any $K$, if the step size is chosen to satisfy $\alpha^K\le\frac{\mu_0}{2L_v^2}$, it holds that
    \begin{align}
        &\E\|x^{K+1}-x^\ast\|^2\notag\\
        \leq&(1-{\mu_0\alpha^K})\E\|x^K-x^\ast\|^2+(\alpha^K)^2\operatorname{Var}[\xi^K]\notag\\
        +&\alpha^K(\alpha^K+\frac{2}{\mu_0})\sum_{n=1}^N\ell_{v,n}^2\E\|y_n^K-y^{\ast}_n(x^K,y_{1:n-1}^K)\|^2,\label{eq-lma-sm}
    \end{align}
    where $L_v$ is a constant also found in Remark \ref{rmk-Lv}, $\ell_{v,n}$ is a constant for any $n\in[N]$.
\end{lemma}

At the right-hand side of \eqref{eq-lma-sm}, the first term indicates a contraction property of $\|x^K-x^\ast\|^2$, while the second term is the error caused by random noise $\xi^K$. The last term is due to the bias from $v(x^K,y_{1:N}^K)$ to $v(x^K,y^\diamond_{1:N}(x^K))$, and can be bounded as in Lemma \ref{lma:conv-sub}.
Based on Lemmas \ref{lma:conv-sub} and \ref{lma-sm-main}, we finally obtain the convergence result of MSSA as follows.
% As shown in Lemma \ref{lma:conv-sub} and \ref{lma-sm-main}, the convergence of sequences in MSSA are dependent on each other and have a complicated time-coupling. We design a new Lyapunov function to tackle this complicated coupling see the details in Appendix \ref{apx-thm-sm}. And we present the convergence result of MSSA as follows.

\begin{theorem}\label{thm-MSSA-sm}
    Consider the MSSA sequences generated by \eqref{MSSA} for $k\in[K]$. Suppose that Assumptions \ref{asp-sub-sm}--\ref{asp:main-sm} hold. If the step sizes are chosen as $\alpha^k=\Theta\left(\frac{\log K}{K}\right)$ and $\beta^k_n=\Theta\left(\frac{\log K}{K}\right)$, then it holds that
    \begin{gather*}
        \E\|x^{K}-x^\ast\|^2+\sum_{n=1}^N\E\|y_n^K-y_n^\ast(x^K,y_{1:n-1}^K)\|^2=\Tilde{\cO}\left(\frac{1}{K}\right),
    \end{gather*}
    which implies that
    \begin{gather}\label{eq-thm-sm}
        \E\|x^{K}-x^\ast\|^2+\sum_{n=1}^N\E\|y_n^K-y_n^\diamond(x^\ast)\|^2=\Tilde{\cO}\left(\frac{1}{K}\right).
    \end{gather}
    % where $\Tilde{\cO}(\cdot)$ hides problem dependent constants of a polynomial of $N$.
\end{theorem}
Theorem \ref{thm-MSSA-sm} shows that the mean square error converges to zero at a rate of $\Tilde{\cO}(K^{-1})$. The order of the convergence rate is independent on $N$, and optimal when ignoring the logarithm term, and matches that of SSSA with the $\Theta(\frac{\log K}{K})$ step size \cite[Theorem 1]{moulines2011non}.

\subsection{Analysis for Non-strongly Monotone Main Operator}

Assumption \ref{asp:main-sm} is too strict in many applications. One example is that when we analyze the actor-critic methods, the main operator is not necessarily strongly monotone \cite{konda1999actor}. Following \cite{shen2022single,karimi2019non}, we make an alternative assumption.

\begin{assumption}[Existence and lower-boundedness of primitive function of $v$]\label{asp:main-nsm}
    There exists a function $\Phi:\R^{d_0}\to\R$ lower-bounded by $C_F$ such that $\nabla \Phi(x)=v(x,y^\diamond_{1:N}(x))$.
\end{assumption}

From the optimization perspective, the primitive function $\Phi(\cdot)$ stands for the cost function to minimize. If $v(\cdot)$ has a primitive function $\Phi(\cdot)$, then Assumption \ref{asp:main-sm} is stronger than Assumption \ref{asp:main-nsm}.

The convergence of \eqref{eq-sm-SSSA} under Assumption \ref{asp:main-nsm} is also well studied \cite{ghadimi2013stochastic}. $\Phi(x^k)$ can be considered as Lyapunov function in the analysis of SSSA \eqref{eq-sm-SSSA}, and the convergence of \eqref{eq-sm-SSSA} can be derived from the descent of Lyapunov function. We first provide the descent property on $\Phi(x^k)$ as follows.
\begin{lemma}\label{lma-nsm}
    Consider the MSSA sequences generated by \eqref{MSSA}. Suppose that Assumptions \ref{asp-sub-sm}--\ref{asp-noise} and \ref{asp:main-nsm} hold. If the step size is chosen to satisfy $\alpha^K\le\frac{1}{2L_v}$, then it holds that
    \begin{align}
        &\E[\Phi(x^{K+1})]\le\ \E[\Phi(x^K)]-\frac{\alpha^K}{2}\E\|v(x^K,y^{\diamond}_{1:N}(x^K))\|^2\notag\\
        &~~~~~~~~+\frac{\alpha^K}{2}\sum_{n=1}^N\ell_{v,n}\E\|y_n^K-y_n^{\ast}(x^K,y_{1:n-1}^K)\|^2\notag\\
        &~~~~~~~~-\frac{\alpha^K}{4}\E\|v(x^K,y^K_{1:N})\|^2+\frac{L_v}{2}(\alpha^K)^2\operatorname{Var}[\xi^K].\label{eq-lma-nsm}
    \end{align}
    where $\ell_{v,n}$ is a constant for any $n\in[N]$.
\end{lemma}
Similar to the strongly monotonic scenario, there are two error terms caused by random noise $\xi^k$ and the bias from $v(x^k,y_{1:N}^k)$ to $v(x^k,y^\diamond_{1:N}(x^k))$ at the right-hand side of \eqref{eq-lma-nsm}. The bias term can be bounded as in Lemma \ref{lma:conv-sub}. With Lemmas \ref{lma:conv-sub} and \ref{lma-nsm}, we finally obtain the convergence result as follows.

\begin{theorem}\label{thm-MSSA-nsm}
    Consider the MSSA sequences generated by \eqref{MSSA} for $k\in[K]$. Suppose that Assumptions \ref{asp-sub-sm}--\ref{asp-noise} and \ref{asp:main-nsm} hold. If the step sizes are chosen as $\alpha^k=\Theta\left(\frac{1}{\sqrt{K}}\right)$ and $\beta^k_n=\Theta\left(\frac{1}{\sqrt{K}}\right)$, then it holds that
    \begin{align*}
        &\frac{1}{K}\sum_{k=1}^K\E\|v(x^k,y^\diamond_{1:N}(x^k))\|^2={\cO}\left(\frac{1}{\sqrt{K}}\right),\notag\\
        &\frac{1}{K}\sum_{k=1}^K\sum_{n=1}^N\E\|y_n^k-y_n^\ast(x^k,y_{1:n-1}^k)\|^2={\cO}\left(\frac{1}{\sqrt{K}}\right).
    \end{align*}
\end{theorem}

Without the strong monotonicity assumption, $v(\cdot)$ may have more than one fixed point. Therefore, we replace $\E\|x^k-x^\ast\|^2$ by $\E\|v(x^k,y_{1:N}^\diamond(x^k))\|^2$ in the performance metric. Theorem \ref{thm-MSSA-nsm} shows that the mean square error converges to zero at a rate of $\cO(K^{-\frac{1}{2}})$, which matches the convergence rate of SSSA \cite[Corollary 2.2]{ghadimi2013stochastic}.

\section{Application to Bilevel Optimization}
\label{sec:BO}
In this section, we apply our theoretical results on MSSA to bilevel optimization. Bilevel optimization naturally appears in various research topics, including hyper-parameter learning \cite{feurer2019hyperparameter}, neural architecture search \cite{cai2023bhe}, and active learning \cite{borsos2021semi}. Besides, it is also tightly related to meta learning \cite{ji2021bilevel} and reinforcement learning \cite{konda1999actor}.

A stochastic bilevel optimization problem is formulated as
\begin{subequations}\label{BO}
\begin{align}
    \min_{x\in\R^{d_x}} ~ & ~F(x):=f(x,y_1^\ast(x)),\\
    ~~s.t. ~ & ~y_1^\ast(x):=\arg\min_{y_1\in\R^{d_y}}g(x,y_1),
\end{align}
\end{subequations}
where the upper-level function $f(x,y_1):={\E}_{\zeta\sim\cD_f}[f(x,y_1;\zeta)]$ and the lower-level function $g(x,y_1):= {\E}_{\phi\sim\cD_g}[g(x,y_1;\phi)]$ with $\zeta$ and $\phi$ being random variables following distributions $\cD_f$ and $\cD_g$, respectively. It is common to suppose that $g(x,y_1)$ is strongly convex in $y_1$ for any $x$. Therefore, given any $x$, minimizing the lower-level function $g(x,y_1)$ yields a unique solution $y_1^\ast(x)$.

In recent years, gradient-based methods for bilevel optimization, which approximate the gradient of $F(x)$ and perform gradient descent on $x$, have become popular. There are several methods to approximate $\nabla F(x)$, for example, approximate implicit differentiation\cite{pedregosa2016hyperparameter,grazzi2020iteration,ji2021bilevel}, iterative differentiation\cite{franceschi2018bilevel,shaban2019truncated,ji2021bilevel,grazzi2020iteration}, Neumann series\cite{ghadimi2018approximation,hong2023two,chen2021closing}, etc. Nevertheless, these methods require an inner loop to estimate lower-level solution $y^\ast(x)$ or the Hessian inverse of $g(x,y_1)$, which brings high computational overhead. To overcome this challenge, \cite{dagreou2022framework} proposes the StOchastic Bilevel optimization Algorithm (SOBA) which has no inner loop. However, \cite{dagreou2022framework} makes the high-order Lipschitz continuity assumptions on $f$ and $g$ to guarantee the convergence of SOBA, which are not required in previous works and may not hold in practice. We find that SOBA falls into the MSSA framework, and according to the convergence analysis of MSSA, one can show that without the high-order Lipschitz continuity assumptions, SOBA still converges at the same rate as in \cite{dagreou2022framework}. We discuss the details in the following.

\subsection{Stochastic Bilevel Optimization Algorithm (SOBA)}
% \red{Why do we only focus on SOBA? Can our analysis be used to other algorithms?}\blue{Have explained it above. MA-SOBA also falls into our framework, while other kind of graident based methods require modification on assumption on noise.}

% A state-of-the-art approach to solving \eqref{BO} is the StOchastic Bilevel optimization Algorithm (SOBA) , which is gradient-based \cite{dagreou2022framework}.
First, we introduce SOBA and show that it falls into the MSSA framework.
According to the implicit function theorem \cite{ghadimi2018approximation}, $\nabla F(x)$ takes a form of
\begin{align}\label{eq-hyper-gradient}
\hspace{-1.0em}    \nabla F(x)=\nabla_xf(x,y_1^\ast(x))+\nabla_{xy}^2g(x,y_1^\ast(x))y_2^\ast(x,y_1^\ast(x)),
\end{align}
where vector function $y_2^\ast(\cdot):\R^{d_x}\times\R^{d_y}\to\R^{d_y}$ is defined as
\begin{align}
    y_2^\ast(x,y_1):=-[\nabla_{y_1y_1}^2g(x,y_1)]^{-1}\nabla_{y_1}f(x,y_1),
\end{align}
$\nabla_xf(x,y_1)$ and $\nabla_{y_1}f(x,y_1)$ respectively represent the gradients of $f$ w.r.t. $x$ and $y_1$, $\nabla_{y_1y_1}^2g(x,y_1)$ represents the Hessian of $g$ w.r.t. $y_1$, while $\nabla_{xy_1}^2g(x,y_1)$ represents the Jacobian of $g$. Given any $x$, computing the gradient $\nabla F(x)$ requires the lower-level solution $y_1^\ast(x)$ as well as the Hessian inverse $[\nabla_{y_1y_1}^2g(x,y_1^\ast(x))]^{-1}$, and hence is expensive. To address this issue, SOBA applies stochastic gradient descent on the lower-level problem to approximate $y_1^\ast(x)$, and on
\begin{gather}\label{sub-problem}
    \min_{y_2\in\R^{d_y}}~\frac{1}{2}\langle\nabla^2_{y_1y_1}g(x,y_1)y_2,~y_2\rangle+\langle\nabla_{y_1}f(x,y_1),~y_2\rangle
\end{gather}
to approximate $y_2^\ast(x,y_1)$.

At time $k$, SOBA updates as
\begin{subequations}\label{alg:SOBA}
    \begin{align}
        y_1^{k+1} = y_1^k & -\beta^k_1\nabla_{y_1}g(x^k,y_1^k;\phi_1^k),\\
        y_2^{k+1} = y_2^k & -\beta^k_2\big(\nabla_{y_1}f(x^k,y_1^k;\zeta^k_2) \notag\\
                   & +\nabla^2_{y_1y_1}g(x^k,y_1^k;\phi^k_2)y_2^k\big),\\
        x^{k+1}   = x^k & -\alpha^k\big(\nabla_xf(x^k,y_1^k;\zeta^k_0) \notag\\
                   & +\nabla^2_{xy_1}g(x^k,y_1^k;\phi^k_0)y_2^k\big),
    \end{align}
\end{subequations}
where $\zeta^k_0$ and $\zeta^k_2$ are independently sampled from $\cD_f$, while $\phi^k_0$, $\phi^k_1$ and $\phi^k_2$ are independently sampled from $\cD_g$. The goal of \eqref{alg:SOBA} is to find $x^\ast$, $y_1^\ast$ and $y_2^\ast$ satisfying
%To reformulate \eqref{alg:SOBA} into \eqref{MSSA} with $N=1$, let $y_1,y_2$ denote $y,u$ respectively, and the generic mappings are defined as:
\begin{subequations}\label{eq-bi-mapp}
    \begin{align}
\hspace{-1.2em}        h_1(x^\ast,y_1^\ast) & :=\nabla_{y_1}g(x^\ast,y_1^\ast) =0,\\
\hspace{-1.2em}        h_2(x^\ast,y_1^\ast,y_2^\ast) & :=\nabla_{y_1}f(x^\ast,y_1^\ast)\!+\!\nabla^2_{y_1y_1}g(x^\ast,y_1^\ast)y_2^\ast =0,\\
\hspace{-1.2em}        v(x^\ast,y_1^\ast,y_2^\ast) & :=\nabla_xf(x^\ast,y_1^\ast)+\nabla^2_{xy_1}g(x^\ast,y_1^\ast)y_2^\ast =0,
    \end{align}
\end{subequations}
which imply $\nabla F(x^\ast)=0$ and $y_1^\ast=y^\ast(x^\ast)$, namely, $(x^\ast, y_1^\ast)$ is an optimal solution to \eqref{BO}.

Observe that the updates \eqref{alg:SOBA} fall into the MSSA framework \eqref{MSSA} with $N=2$. {The operators $v,h_1,h_2$ are shown in \eqref{eq-bi-mapp}. The random noise at time $k$ comes from the random sampling processes, given by
\begin{subequations}\label{eq-soba-noise}
    \begin{align}
        \psi^k_1:=&\nabla_{y_1}g(x^k,y_1^k;\phi_1^k)-\nabla_{y_1}g(x^k,y_1^k),\\
        \psi^k_2:=&\nabla_{y_1}f(x^k,y_1^k;\zeta^k_2)-\nabla_{y_1}f(x^k,y_1^k)\notag\\
        &+\nabla^2_{y_1y_1}g(x^k,y_1^k;\phi^k_2)y_2^k-\nabla^2_{y_1y_1}g(x^k,y_1^k)y_2^k,\\
        \xi^k:=&\nabla_xf(x^k,y_1^k;\zeta^k_0)-\nabla_xf(x^k,y_1^k)\notag\\
        &+\nabla^2_{xy_1}g(x^k,y_1^k;\phi^k_0)y_2^k-\nabla^2_{xy_1}g(x^k,y_1^k)y_2^k.
    \end{align}
\end{subequations}
}

\subsection{Improved Analysis for SOBA}
Before applying our theoretical results to SOBA, we first verify the assumptions made for MSSA.
\begin{lemma}[Verifying assumptions of MSSA]\label{lma-asp-SOBA}
    Consider the following conditions: \\
    $(a)$ $f(x,y_1)$ and $\nabla f(x,y_1)$ are $\ell_{f}$-Lipschitz continuous. $\nabla_{y_1}$ $g(x,y_1)$, $\nabla^2_{xy_1}g(x,y_1)$ and $\nabla^2_{y_1y_1}g(x,y_1)$ are all $\ell_{g}$-Lipschitz continuous. \\
    $(b)$ For any $x$, $g(x,y_1)$ is $\mu_g$-strongly convex in $y_1$. \\
    $(c)$ $\nabla f(x,y_1;\zeta), $ $\nabla_{y_1}g(x,y_1;\phi)$, $\nabla^2_{xy_1}g(x,y_1;\phi)$, and $\nabla^2_{y_1y_1}$ $g(x,y_1;\phi)$ are unbiased estimators of $\nabla f(x,y_1)$, $\nabla_{y_1}g(x,y_1)$, $\nabla^2_{xy_1}g(x,y_1)$, and $\nabla^2_{y_1y_1}g(x,y_1)$, respectively, and the estimation variances are bounded {by $\hat\sigma^2$}. \\
    $(d)$ $F(x)$ is $\mu_x$-strongly convex. \\
    $(e)$ $F(x)$ is lower-bounded by $C_F$. \\
   {Under conditions $(a)$--$(d)$ (resp. conditions $(a)$--$(c)$ and $(e)$), the operators and the random noise involved with the SOBA updates \eqref{alg:SOBA} satisfy Assumptions \ref{asp-sub-sm}--\ref{asp:main-sm} (resp. Assumptions \ref{asp-sub-sm}--\ref{asp-noise} and \ref{asp:main-nsm}).} In particular, $(b)$ $\Rightarrow$ Assumption \ref{asp-sub-sm}; $(a)$--$(b)$
    $\Rightarrow$ Assumptions \ref{asp-lc-sub}--\ref{asp-lc-main}; $(a)$--$(c)$ $\Rightarrow$ Assumption \ref{asp-noise}; $(b)$, $(d)$ $\Rightarrow$ Assumption \ref{asp:main-sm}; and $(b)$, $(e)$
    $\Rightarrow$ Assumption \ref{asp:main-nsm}.
\end{lemma}
The conditions in Lemma \ref{lma-asp-SOBA} are all standard assumptions in bilevel optimization \cite{dagreou2022framework,ghadimi2018approximation,chen2021tighter,akhtar2022projection}. Note that these conditions do not ensure the Lipschitz continuity of $h_2$ but Condition (b) in Assumption \ref{asp-lc-sub} on $h_2$. In addition, smoothness of the fixed point {$y_2^\ast(\cdot)$} is not necessarily valid under the conditions in Lemma \ref{lma-asp-SOBA}. In fact, \cite{dagreou2022framework} assumes Lipschitz continuity of $\nabla^2f$ and $\nabla^3g$ to guarantee smoothness of the fixed point {$y_2^\ast(\cdot)$}. According to Theorems \ref{thm-MSSA-sm} and \ref{thm-MSSA-nsm}, such assumptions are not needed in establishing the convergence rates of SOBA, as stated in the following corollary.
\begin{corollary}\label{cor:SOBA}
    Consider the SOBA sequences generated by \eqref{alg:SOBA} for $k\in[K]$. Under Conditions (a)--(d) in Lemma \ref{lma-asp-SOBA}, selecting the step sizes as $\alpha^k=\Theta\left(\frac{\log K}{K}\right)$, $\beta^k_1=\Theta\left(\frac{\log K}{K}\right)$ and $\beta^k_2=\Theta\left(\frac{\log K}{K}\right)$, we have
    \begin{align*}
        &\E\|x^K-x^\ast\|^2+\E\|y_1^K-y_1^\ast(x^K)\|^2
        = \Tilde{\cO}\left(\frac{1}{K}\right).
    \end{align*}
    Alternatively, under Conditions (a)--(c) and (e) in Lemma \ref{lma-asp-SOBA}, selecting the step sizes as $\alpha^k=\Theta\left(\frac{1}{\sqrt{K}}\right)$, $\beta^k_1=\Theta\left(\frac{1}{\sqrt{K}}\right)$ and $\beta^k_2$ $=\Theta\left(\frac{1}{\sqrt{K}}\right)$, we have
    \begin{align*}
        &\frac{1}{K}\sum_{k=1}^K\left(\E\|\nabla F(x^k)\|^2+\E\|y_1^k-y_1^\ast(x^k)\|^2\right) = {\cO}\left(\frac{1}{\sqrt{K}}\right).
    \end{align*}
\end{corollary}
\begin{remark}[Comparisons with the existing works]
    Our theoretical results in Corollary \ref{cor:SOBA} exactly match those in \cite{dagreou2022framework}, but we do not need to assume Lipschitz continuity of $\nabla^2f$ and $\nabla^3g$. Momentum acceleration has been recently combined with SOBA in \cite{chen2024optimal,li2022fully} to obtain the same convergence rates without assuming high-order Lipschitz continuity. Our analysis shows that momentum acceleration is not critical for such improvement.
\end{remark}

\section{Application to Communication-efficient Distributed Learning}\label{sec:compress}
In this section, we present the application of MSSA in com- munication-efficient distributed learning. Distributed learning enables collaboration among multiple nodes to train a common model with distributed data via communication with a server. It is often formulated as a stochastic optimization problem in the form of
\begin{align}\label{prob:DO}
    \min_{x\in\R^d}\ F(x):=\frac{1}{N}\sum_{n=1}^Nf_n(x),
\end{align}
where $N$ is the number of nodes, $f_n(x):=\E_{\zeta_n\sim\cD_n}[f(x;\zeta_n)]$ is the local loss function of node $n$, and $\zeta_n$ is a random variable following the local data distribution $\cD_n$.

% and become a fundamental tool for large-scale learning.
Nevertheless, as $N$ scales up, the communication between the nodes and the server becomes the bottleneck of distributed learning. This overhead is particularly remarkable when the trained model is high-dimensional. To address this issue, we can consider: (i) reduce the number of communication rounds and/or (ii) reduce the volume of transmitted messages per communication round. To reduce the number of communication rounds, popular approaches include local update that allows each node to perform multiple computation rounds within one communication round \cite{stich2019local,liconvergence}, and momentum acceleration \cite{polyak1964some} that reduces the number of iterations for reaching a target solution accuracy \cite{zheng2019communication,yu2019linear,liu2020accelerating}. On the other hand, to reduce the volume of transmitted messages per communication round, communication compression techniques, such as quantization \cite{alistarh2017qsgd,micikevicius2018mixed,magnusson2020maintaining} and sparsification \cite{stich2018sparsified,wangni2018gradient,safaryan2022uncertainty}, are widely used. These communication-compression techniques can be classified to unbiased ones \cite{mishchenko2019distributed,qian2021error,he2024unbiased} and biased but contractive ones \cite{he2023lower,subramaniam2024adaptive}. We focus on unbiased compressors, denoted as random operators $\cC$, satisfying both unbiasedness $\E[\cC(x)] = x$ and $\omega$-bounded variation $\E[\|\cC(x)-x\|^2]\leq\omega\|x\|^2$, where $x$ is the model to compress.

\subsection{Communication-compressed and Momentum-accelerated Distributed Learning}

We consider solving \eqref{prob:DO} in a communication-efficient manner, both reducing the number of communication rounds and reducing the volume of transmitted messages per communication round. To be specific, we consider a communication-compressed and momentum-accelerated distributed learning algorithm. At round $k$, each node $n$ calculates a stochastic gradient at the current model $x^k$, updates a momentum vector $y_n^k$ as \eqref{eq-CM-y}, and sends a compressed momentum vector $\cC_n(y_n^k)$ to the server. After receiving $\cC_n(y_n^k)$ from all nodes, the server updates the model as \eqref{eq-CM-x} and sends it to all nodes.

The communication-compressed and momentum-acceler-ated distributed learning algorithm is given by
\begin{subequations}\label{alg:CMomentum}
    \begin{align}
        y_n^{k+1}=&y_n^k-\beta_n^k(y_n^k-\nabla f_n(x^k;\zeta_n^k)),\quad\forall n\in[N],\label{eq-CM-y}\\
        x^{k+1}=&x^k-\alpha^k\frac{1}{N}\sum_{n=1}^N\cC_n(y_n^k),\label{eq-CM-x}
    \end{align}
\end{subequations}
in which $\cC_n(\cdot)$ is the compressor used in node $n$. The goal of the algorithm is to find an optimal point $x^\ast$ such that $\nabla F(x^\ast)=0$. On the other hand, the algorithm also falls into the MSSA framework of \eqref{MSSA}, which aims at finding stationary point $x^\ast,\,y_1^\ast,\,\dots,\,y_N^\ast$ satisfying
\begin{subequations}\label{eq:DO-mapp}
    \begin{align}
        h_n(x^\ast,y_{1:n}^\ast):=&y_n^\ast-\nabla f_n(x^\ast)=0,\quad\forall n\in[N],\\
        v(x^\ast,y_{1:N}^\ast):=&\frac{1}{N}\sum_{n=1}^Ny_n^\ast=0.
    \end{align}
\end{subequations}
Obviously, \eqref{eq:DO-mapp} implies $\nabla F(x^\ast)=0$.
In addition, the random noise comes from the stochastic gradients and compressors, formulated as
\begin{subequations}
    \begin{align}
        \psi_n^k:=&\nabla f_n(x^k;\zeta_n^k)-\nabla f_n(x^k),\\
        \xi^k:=&\frac{1}{N}\sum_{n=1}^N\left(\cC_n(y_n^k)-y_n^k\right).
    \end{align}
\end{subequations}
% which imply $\nabla F(x^\ast)=0$, namely, $x^\ast$ is an stationary point to problem \eqref{prob:DO}.

% Compressed Momentum falls into the MSSA framework of \eqref{MSSA}. The operators $v,h_1,\dots,h_n$ are shown in \eqref{eq:DO-mapp}, while the random noises at iteration $k$ are from stochastic gradient noise and compressors, formulated as
% \begin{subequations}
%     \begin{align}
%         \psi_n^k:=&\nabla f_n(x^k)-\nabla f_n(x^k;\zeta_n^k),\\
%         \xi^k:=&\frac{1}{N}\sum_{n=1}^N\left(\cC_n(y_n^k)-y_n^k\right).
%     \end{align}
% \end{subequations}

\subsection{Convergence Analysis}
Before applying our theoretical results to the communica-tion-compressed and momentum-accelerated distributed learning algorithm, we first verify the assumptions made for MSSA.
\begin{lemma}[Verifying assumptions of MSSA]\label{lma-asp-CM}
    Consider the following conditions:\\
    $(a)$ For any $n\in[N]$, $\nabla f_n(x)$ is $\ell_f$-Lipschitz continuous.\\
    $(b)$ For any $n\in[N]$ and any $x$, $\nabla f_n(x;\zeta)$ is an unbiased estimator of $\nabla f_n(x)$ with its variance bounded by $\sigma_f^2$.\\
    $(c)$ There exists positive constant $a,\,b$ such that
    \begin{align}
        \frac{1}{N}\sum_{n=1}^N\|\nabla f_n(x)\|^2\le a\|\nabla F(x)\|^2 + b^2,\quad \forall x\in\R^d.
    \end{align}
    $(d)$ For any $n\in[N]$, the compression operator $\cC_n$ satisfies:
    \begin{align}
        \E[\cC_n(x)]=x,\ \E[\|\cC_n(x)-x\|^2]\le\omega\|x\|^2,\quad\forall x\in\R^d.
    \end{align}
    $(e)$ $F(x)$ is $\mu$-strongly convex in $x$, where constant $\mu>0$.\\
    $(f)$ $F(x)$ is lower bounded by a constant $C_F$.\\
    {Under conditions $(a)$--$(e)$ (resp. conditions $(a)$--$(d)$ and $(f)$), the operators and noise involved with the communication-compressed and momentum-accelerated distributed learning algorithm \eqref{alg:CMomentum} satisfy Assumptions \ref{asp-sub-sm}--\ref{asp:main-sm} (resp. Assumptions \ref{asp-sub-sm}--\ref{asp-noise} and \ref{asp:main-nsm}).} In particular, $(a)$ $\Rightarrow$ Assumption \ref{asp-lc-sub}, $(b)$--$(d)$ $\Rightarrow$ Assumption \ref{asp-noise}, $(e)$ $\Rightarrow$ Assumption \ref{asp:main-sm}, $(f)$ $\Rightarrow$ Assumption \ref{asp:main-nsm}, Assumptions \ref{asp-sub-sm} and \ref{asp-lc-main} hold naturally.
\end{lemma}
The conditions stated in Lemma \ref{lma-asp-CM} are common in analyzing algorithms with communication compression \cite{stich2018sparsified,mishchenko2019distributed,he2024unbiased}.
According to Theorems \ref{thm-MSSA-sm} and \ref{thm-MSSA-nsm}, we can establish the convergence for the communication-compressed and momentum-accelerated distributed learning algorithm as follows.
\begin{corollary}
    Consider the sequences generated by the communication-compressed and momentum-accelerated distributed learning algorithm \eqref{alg:CMomentum} for $k\in[K]$. Under Conditions $(a)$--$(e)$ in Lemma \ref{lma-asp-CM}, if selecting the step sizes as $\alpha^k$ $=\Theta\left(\frac{\log K}{K}\right)$ and $\beta_n^k=\Theta\left(\frac{\log K}{K}\right)$ for any $n\in[N]$, it holds that for any $n\in[N]$,
    \begin{subequations}
        \begin{align}
            &\E[\|x^K-x^\ast\|^2]=\Tilde\cO\left(\frac{1}{K}\right),\\
            &\E[\|y_n^K-\nabla f_n(x^K)\|^2]=\Tilde\cO\left(\frac{1}{K}\right),\ \forall n\in[N].
        \end{align}
    \end{subequations}
    Alternatively, under Conditions $(a)$--$(d)$ and $(f)$ in Lemma \ref{lma-asp-CM}, if selecting the step sizes as $\alpha^k=\Theta\left(\frac{1}{\sqrt{K}}\right)$ and $\beta_n^k=\Theta\left(\frac{1}{\sqrt{K}}\right)$ for any $n\in[N]$, then it holds that for any $n\in[N]$,
    \begin{subequations}
        \begin{align}
            &\frac{1}{K}\sum_{k=0}^{K-1}\E[\|\nabla F(x^k)\|^2]=\cO\left(\frac{1}{\sqrt{K}}\right),\\
            &\frac{1}{K}\sum_{k=0}^{K-1}\E[\|y_n^k-\nabla f_n(x^k)\|^2]=\cO\left(\frac{1}{\sqrt{K}}\right).
        \end{align}
    \end{subequations}
\end{corollary}
The above convergence rates match the ones of the existing distributed algorithms with unbiased compressors \cite{mishchenko2019distributed,stich2020communication,he2023lower}. It is worth noting that to our knowledge, there is no theoretical analysis for stochastic momentum acceleration with an unbiased compressor before. The algorithm most similar to ours is DINIA proposed in \cite{mishchenko2019distributed}. However, DINIA compresses the difference between the stochastic gradient with the local gradient estimator on each node, and then performs momentum acceleration at the server. In contrast, our algorithm applies unbiased compression to the local momentum vector at each nodes. Besides, DINIA requires nearly twofold memory storage (the model, the momentum vector and the compressed difference) at the server compared to ours (the model and the compressed momentum vector).

\section{Numerical Experiments}
To validate the theoretical findings, we conduct numerical experiments on both bilevel optimization and communication-efficient distributed learning.

\subsection{Bilevel Optimization}
First, we consider a data hyper-cleaning task, which can be formulated as a bilevel optimization problem \cite{ji2021bilevel,franceschi2017forward}. Data hyper-cleaning aims to train a classifier on a polluted training dataset, with the aid of a clean validation dataset. To do so,
every training sample is assigned a variable to represent its probability of being uncontaminated and the cost function on the training dataset is defined as
\begin{gather}
    L(\lambda,\omega):=\frac{1}{n_{tr}}\sum_{i=1}^{n_{tr}}\sigma(\lambda_i)\ell(\omega;x_{tr}^{(i)},y_{tr}^{(i)})+\frac{\mu}{2}\|\omega\|^2.
\end{gather}
Therein, $\ell(\cdot)$ is the loss, $\omega$ is the classifier to train, $x_{tr}^{(i)}$ and $y_{tr}^{(i)}$ are respectively the $i$-th sample's feature and label in the training dataset, $n_{tr}$ stands for the number of samples in the training dataset, $\mu$ is the regularization parameter, $\lambda_i$ is the $i$-th element of the hyperparameter $\lambda$, $\sigma(\cdot)$ is the sigmoid function, and $\sigma(\lambda_i)$ represents the probability of the $i$-th sample being uncontaminated.
Then, the data hyper-cleaning task can be formulated as a bilevel optimization problem in the form of
\begin{subequations}
    \begin{align}\label{HC:validation loss}
        \min_{\lambda\in\R^{n_{tr}},\omega^\ast} & ~\frac{1}{n_{val}}\sum_{i=1}^{n_{val}}\ell(\omega^\ast;x_{val}^{(i)},y_{val}^{(i)}),\\
        s.t.~ & ~ \omega^\ast=\arg\min_{\omega}~L(\lambda,\omega),
    \end{align}
\end{subequations}
where $n_{val}$ is the number of samples in the validation dataset, while $x_{val}^{(i)}$ and $y_{val}^{(i)}$ are respectively the $i$-th sample's feature and label in the validation dataset.

In the numerical experiments, we compare three baselines, two-timescale SOBA (TT-SOBA), MA-SOBA \cite{chen2024optimal} and FSLA \cite{li2022fully}, with single-timescale SOBA (ST-SOBA) \cite{dagreou2022framework} in Section \ref{sec:BO} on the MNIST dataset. TT-SOBA uses two-timescale step sizes within the SOBA updates, whereas both MA-SOBA and FSLA combine momentum acceleration with the SOBA update of the upper-level variable $x$. The step sizes in ST-SOBA \eqref{alg:SOBA} are all set as $\cO(k^{-\frac{1}{2}})$, while $\alpha^k$ is set as $\cO(k^{-\frac{3}{5}})$, $\beta_{1}^k$ and $\beta_{2}^k$ are set as $\cO(k^{-\frac{2}{5}})$ in TT-SOBA. As for MA-SOBA and FSLA, the step sizes are chosen as the same as those in ST-SOBA, and the momentum coefficient is set as $\cO(k^{-\frac{1}{2}})$, consistent with the theoretical setting in \cite{chen2024optimal,li2022fully}.

We randomly split 30,000 training samples in the MNIST dataset into training, validation and test datasets. Every dataset has 10,000 samples, and $40\%$ labels in the training samples are randomly changed. The batch size is $100$.

As shown in Figure \ref{fig:HC}, for both linear regression and two-layer MLP (multi-layer perception) losses, ST-SOBA converges as fast as MA-SOBA and FSLA, and outperforms TT-SOBA. With particular note, when using the MLP losses, the fixed point of SOBA is Lipschitz continuous but not Lipschitz smooth. Table \ref{tab:ACC} shows that after 5,000 iterations, ST-SOBA, MA-SOBA and FSLA all have similar solution qualities, and are superior to TT-SOBA. These observations corroborate our theoretical results in Section \ref{sec:BO}.

% \begin{figure}[tb]
%     \centering
%     % \subfigcapskip=0pt
%     \subfloat[Linear Regression]{\includegraphics[width=40mm]{fig/Hyper-Cleaning-f-linear.pdf}}
%     \subfloat[Two-layer MLP]{\includegraphics[width=40mm]{fig/Hyper-Cleaning-f-MLP.pdf}}
%     \caption{Comparisons on validation loss with different classifiers, averaged over 10 runs. }\label{fig:HC}
% \end{figure}
\begin{figure}[thb]
    \centering
    % \subfigcapskip=0pt
    \subfigure[Linear Regression]{\includegraphics[width=40mm]{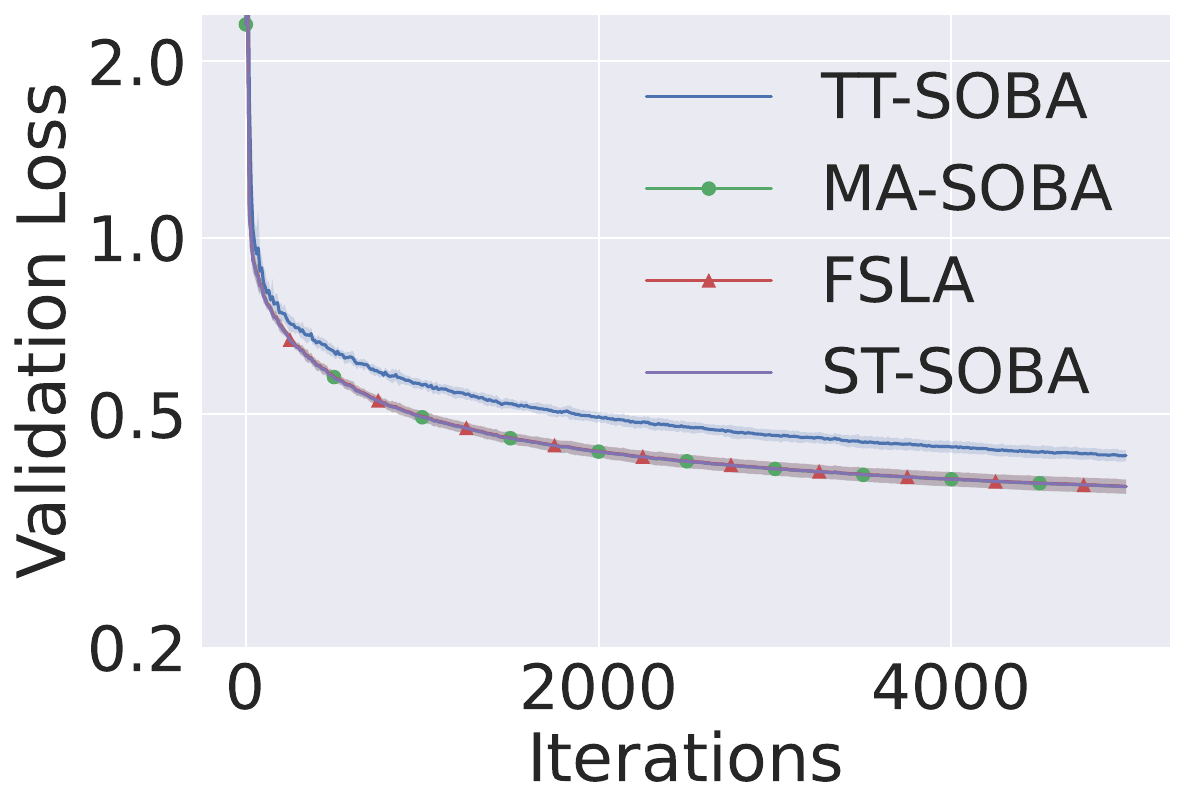}}
    \subfigure[Two-layer MLP]{\includegraphics[width=40mm]{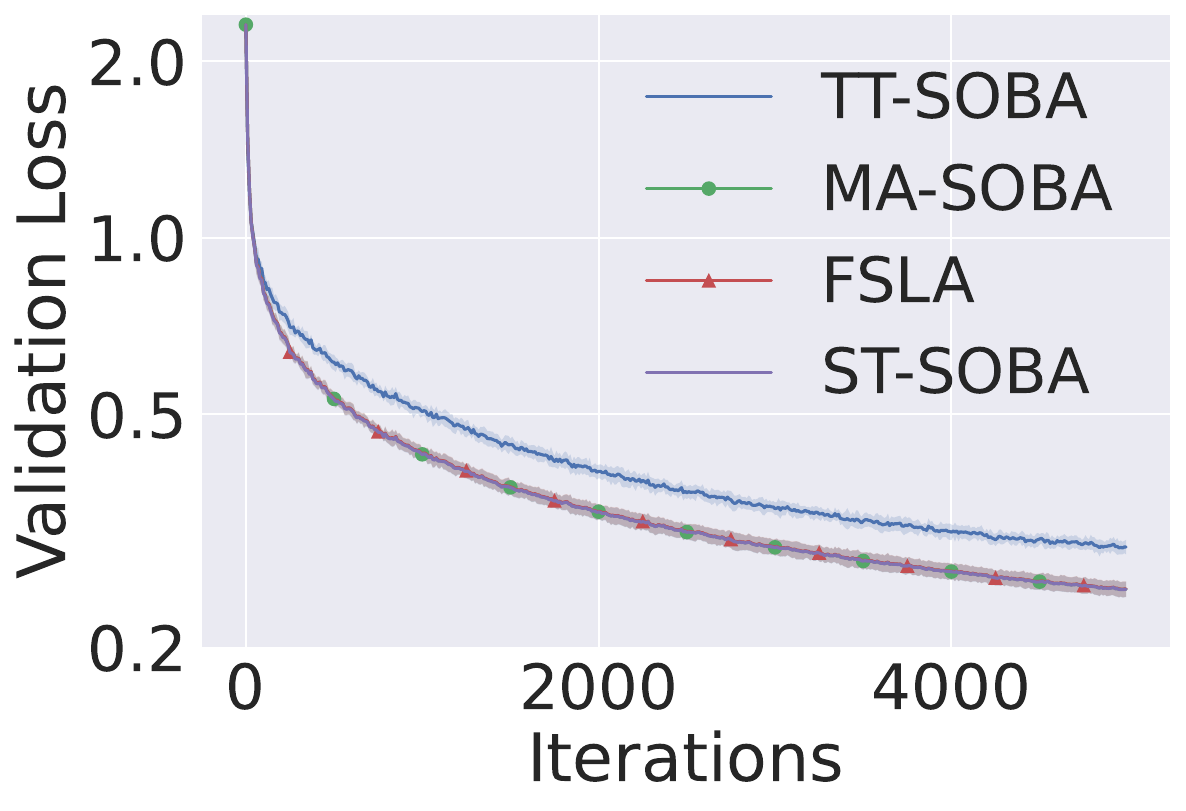}}
    \caption{Validation loss with different classifiers, averaged over 10 runs. }\label{fig:HC}
\end{figure}
\begin{table}[htb]
 \vspace{-1em}
 \small
 \centering
 \caption{Solution quality in terms of accuracy. Acc1: Accuracy of trained classifier on test dataset. Acc2: Accuracy of identifying contaminated samples.}
\label{tab:ACC}
\begin{tabular}{ |c||p{1cm}<{\centering}|p{1cm}<{\centering}|p{1cm}<{\centering}|p{1cm}<{\centering}| }
 \hline
 \multirow{2}{*}{} & \multicolumn{2}{c|}{Linear Regression} & \multicolumn{2}{c|}{Two-layer MLP} \\
  \cline{2-5}
  & Acc1 & Acc2 & Acc1 & Acc2 \\
 \hline
 ST-SOBA & $90.37$ & $91.84$ & $93.38$ & $92.21$ \\
 \hline
 TT-SOBA & $90.13$ & $91.43$ & $93.38$ & $91.83$ \\
 \hline
 MA-SOBA & $90.37$ & $91.84$ & $93.36$ & $92.21$ \\
 \hline
 FSLA & $90.36$ & $91.84$ & $93.37$ & $92.21$\\
 \hline
\end{tabular}
% \vspace{-0.1cm}
% \vspace{-1.8em}
\end{table}
% \vspace{-0.5em}

\subsection{Communication-efficient Distributed Learning}
Second, we evaluate the performance of the communication-compressed and momentum-accelerated distributed learning algorithm. We consider an unbiased sparsification compressor $\cC_p(\cdot)$ with compression rate $p\in(0,1]$. Taking a vector $x\in\R^d$ as the input, the $i$-th dimension of the output $\cC_p(x)$, denoted as $(\cC_p(x))^{(i)}$, is given by
$$
(\cC_p(x))^{(i)}:=\left\{
\begin{aligned}
    &\frac{x^{(i)}}{p},&\quad \text{with probability }p,\\
    &0,&\quad \text{otherwise}.
\end{aligned}
\right.
$$
\subsubsection{Distributed $\ell_2$-Support Vector Machine (SVM)}
Consider a distributed $\ell_2$-SVM problem, which can be formulated to \eqref{prob:DO} with the local loss function at node $n$ being
\begin{align}
    f_n(w,b):=&\frac{1}{M_n}\sum_{m=1}^{M_n}\left(\max\{1-(s_{n}^{(m)T}w+b)t_{n}^{(m)},\,0\}\right)^2\notag\\
    &+\frac{\lambda}{2}\|w\|^2,
\end{align}
where $w\in\R^d$ and $b\in\R$ are model parameters, $s_{n}^{(m)}\in\R^d$ and $t_{n}^{(m)}\in\{1,\,-1\}$ are respectively the $m$-th sample's feature and label in the local training dataset $\cD_n$, $M_n$ is the number of samples in $\cD_n$ and $\lambda>0$ is a regularized parameter. The loss function of $\ell_2$-SVM is strongly convex. Note that for each $n$, $\nabla f_n$ is not Lipschitz smooth; that is, the fixed point is not Lipschitz smooth. To be precise, $\nabla^2f_n$ is undefined on the set $\bigcup_{m=1}^{M_n}S_n^m$ where $S_n^m:=\{(w,\,b)\,|\,1-(s_{n}^{(m)T}w+b)t_{n}^{(m)}=0\}$. Thus, when applying \eqref{alg:CMomentum} to solve this problem, the Lipschitz smoothness assumption on the fixed point does not hold.

In the numerical experiments, the training dataset is generated as follows. First, we independently generate each sample's feature $s_{n}^{(m)}$ from $\cN(0,\,\text{Diag}(v)^2)$, in which each dimension of vector $v\in\R^d$ is generated from $\text{Uniform}([0,1])$. Second, we generate $(w^\ast,\,b^\ast)\in\R^{d+1}$ from $\text{Uniform}([-0.5,0.5]^{d+1})$ and generate the corresponding label $t_n^{(m)}$ for $s_{n}^{(m)}$ by $t_n^{(m)}=\sign(s_{n}^{(m)T}w^\ast+b^\ast+0.2r_n^{m})$ with $r_n^{(m)}\sim\cN(0,\,1)$.
We set $d$ $=200$, $N=10$ and $\lambda=0.5$.

We apply the communication-compressed and momentum-accelerated distributed learning algorithm to solve this problem, with batch size $10$ and step size $\alpha^k=\cO(k^{-1})$, $\beta_n^k=\cO(k^{-1})$. As shown in Figure \ref{fig:L2SVM}, the proposed algorithm is able to converge at a rate of $\cO(k^{-1})$, which matches our theoretical results in Section \ref{sec:compress}.
\begin{figure}[thb]
    \centering
    % \subfigcapskip=0pt
    \includegraphics[width=40mm]{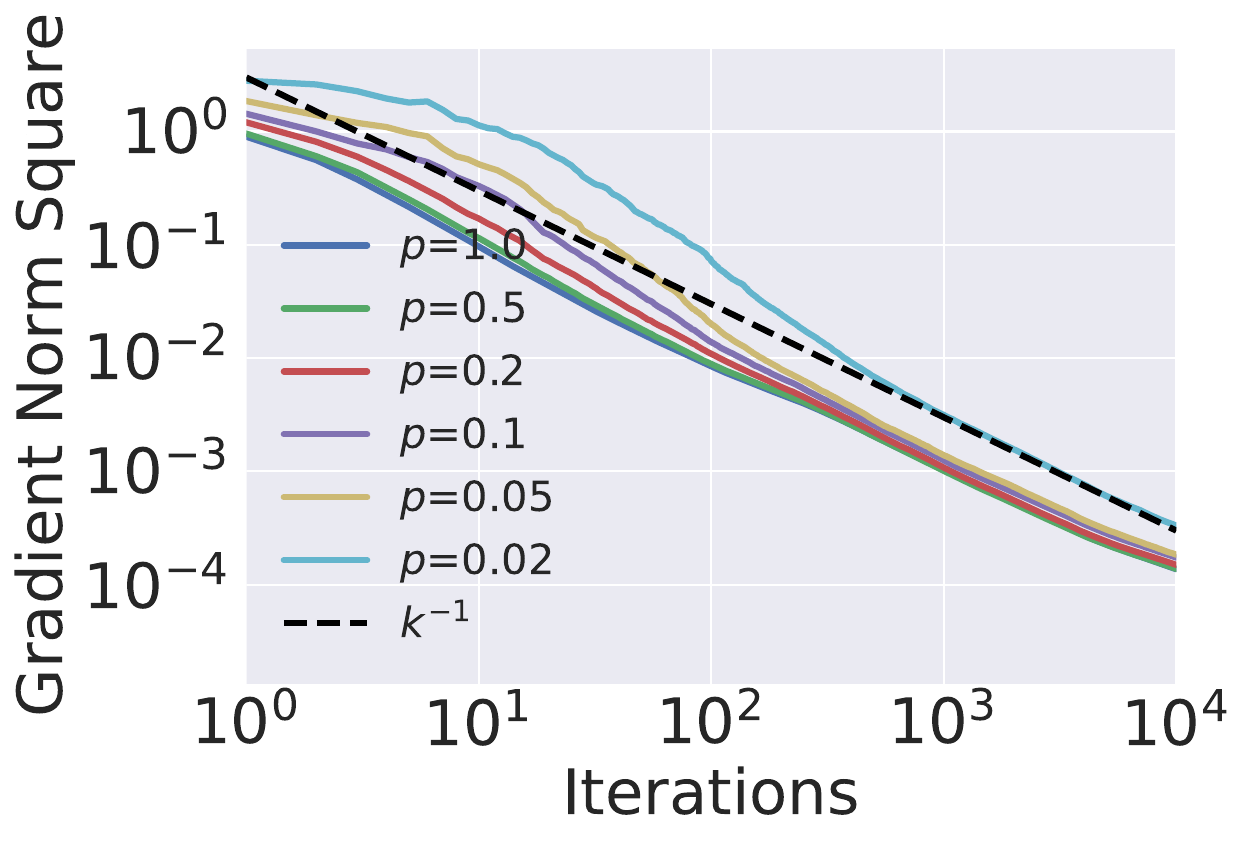}
    \includegraphics[width=40mm]{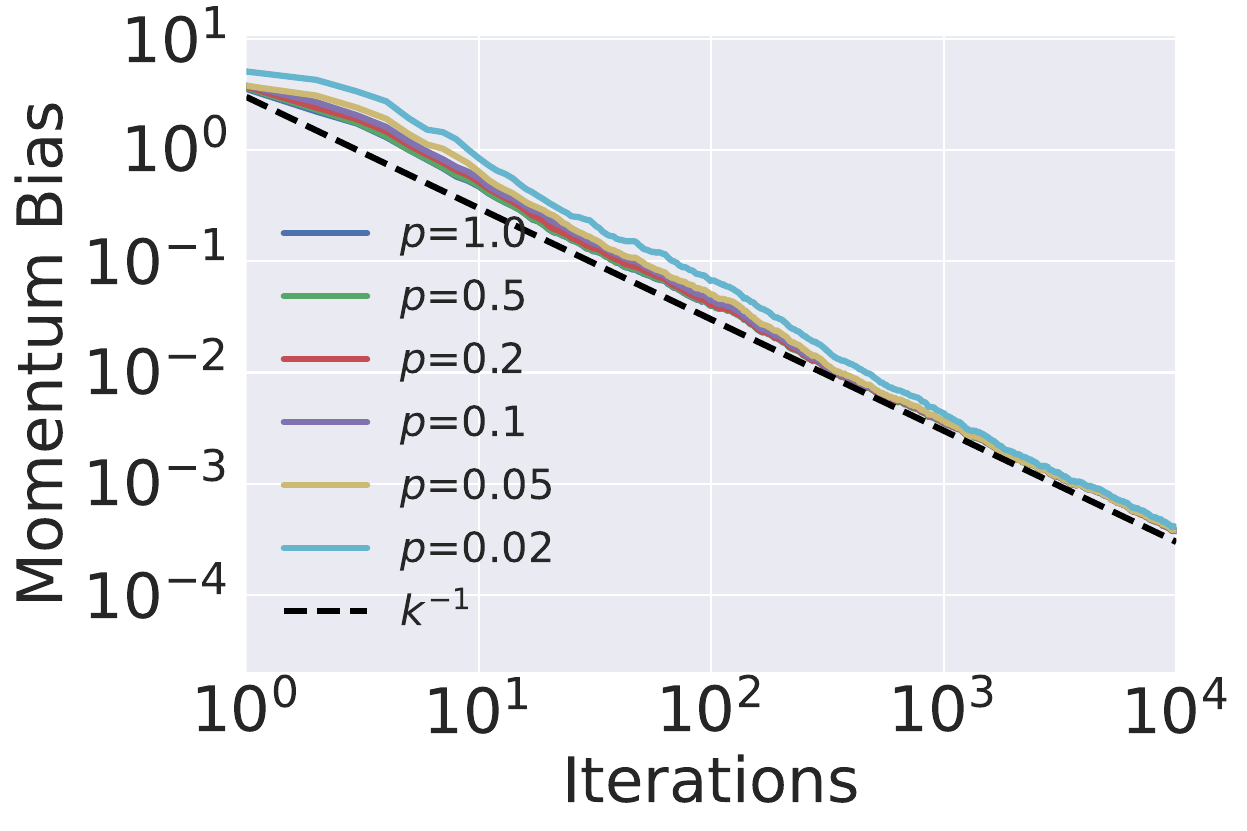}
    \caption{$\ell_2$-SVM with different compression rates $p$. Gradient Norm Squared: $\|\frac{1}{N}\sum_{n=1}^N\nabla f_n(x^k)\|^2$. Momentum Bias: $\frac{1}{N}\sum_{n=1}^N\|y_n^k-\nabla f_n(x^k)\|^2$. }\label{fig:L2SVM}
\end{figure}

\subsubsection{Classification on Fashion MNIST}
\begin{figure}[htb]
    \centering
    % \subfigcapskip=0pt
    \includegraphics[width=0.4\textwidth]{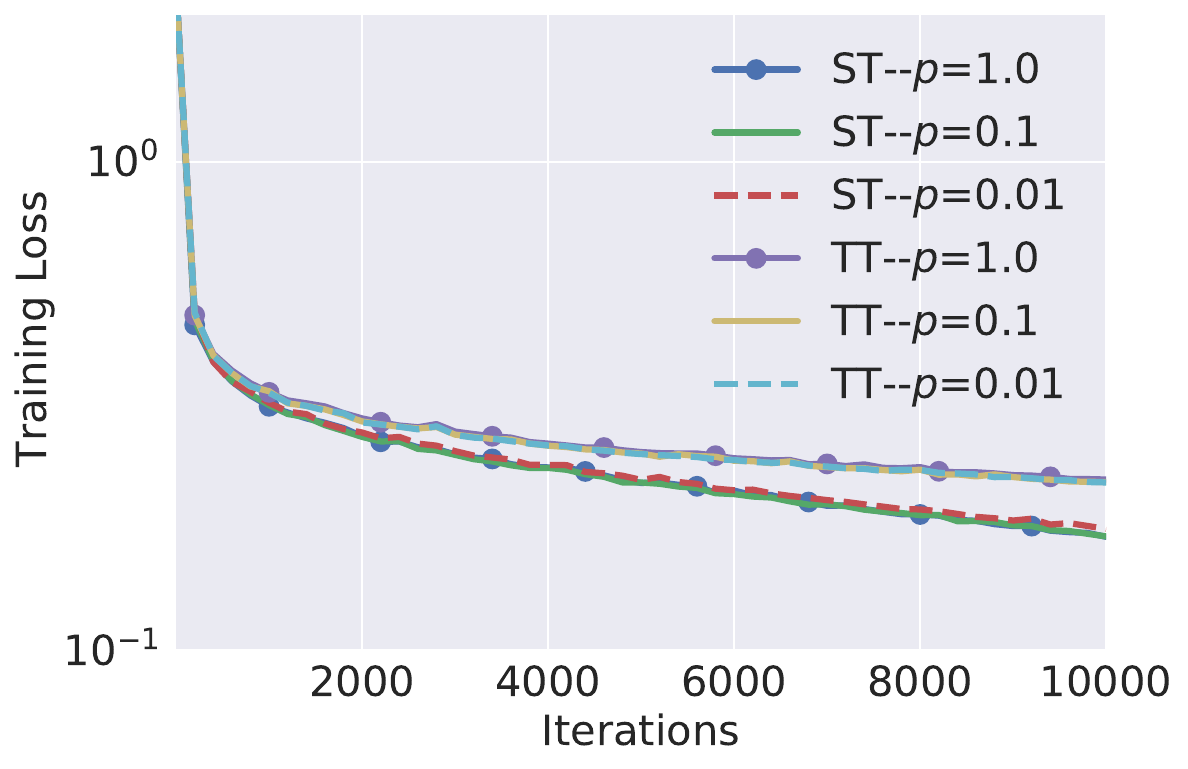}
    \centering
    \caption{Training loss with different compression rates $p$. ST and TT represent single-timescale and two-timescale step sizes, respectively.}\label{fig:FM-TL}
\end{figure}
% \begin{figure*}[htb]
%     \centering
%     % \subfigcapskip=0pt
%     \includegraphics[width=45mm]{fig/ST-FM-TL.pdf}
%     \includegraphics[width=45mm]{fig/TT-FM-TL.pdf}
%     \includegraphics[width=42mm]{fig/ST-FM-TA.pdf}
%     \includegraphics[width=42mm]{fig/TT-FM-TA.pdf}
%     \centering
%     \caption{Compressed Momentum with different timescale setting.}\label{fig:FM}
% \end{figure*}
We also test the communication-compressed and momentum-accelerated distributed learning algorithm by training a convolutional neural network (CNN) on the Fashion MNIST dataset. In the numerical experiments, we set the number of nodes $N=10$ and randomly distribute the 60,000 Fashion MNIST training samples into 10 nodes. We train the CNN model by the proposed algorithm, with different compression rates and different timescales. For the single-timescale (ST) algorithm, we set the step size $\alpha^k=\cO(k^{-\frac{1}{2}})$ and the momentum parameter $\beta_n^k=\cO(k^{-\frac{1}{2}})$. For the two-timescale algorithm, we set $\alpha^k=\cO(k^{-\frac{3}{5}})$ and $\beta_n^k=\cO(k^{-\frac{2}{5}})$.
Fig \ref{fig:FM-TL} reports the loss function on the training dataset, while TABLE \ref{tab:CM-ACC} presents the accuracy on the testing dataset. The results show that the ST algorithm outperforms the TT algorithm, and is consistently well when the compression rate $p$ decreases.

\begin{table}[htb]
 \vspace{-1em}
 \small
 \centering
 \caption{Test accuracies of single-timescale and two-timescale step sizes.}
\label{tab:CM-ACC}
\begin{tabular}{ |c||p{1cm}<{\centering}|p{1cm}<{\centering}|p{1cm}<{\centering}|p{1cm}<{\centering}| }
 \hline
 $p$ & $1.0$ & $0.1$& $0.01$ \\
 \hline
 Single-timescale & $90.44$ & $90.62$ & $90.13$ \\
 \hline
 Two-timescale & $89.53$ & $89.75$ & $89.67$  \\
 \hline
\end{tabular}
% \vspace{-0.1cm}
% \vspace{-1.8em}
\end{table}

\section{Conclusions}
This paper provides a single-timescale analysis for MSSA without assuming the smoothness of the fixed points. We show that when all involved operators are strongly monotone, MSSA converges at a rate of $\Tilde{\cO}(K^{-1})$. Under a weaker assumption that all involved operators are strongly monotone except for the main one, MSSA converges at a rate of $O(K^{-\frac{1}{2}})$. We apply these theoretical results to the bilevel optimization algorithm SOBA, removing the high-order Lipschitz continuity assumptions but establishing the same convergence rates. We also show that the MSSA analytical framework can be applied to communication-efficient distributed learning. Numerical experiments verify our theoretical results.

\appendices
\section{Proof of Lemma \ref{lma:Lip-fp}}
(i) The case with condition (a) in Assumption \ref{asp-lc-sub}:
For any $n\in[N]$, given $x$ and $y_{1:n-1}$, consider the following operator $T_n:\R^{d_n}\to\R^{d_n}$:
\begin{align}
    T(y):=y-\eta_nh_n(x,\,y_{1:n-1},\,y),\quad\forall y\in\R^{d_n}.
\end{align}
where $\eta_n=\frac{\mu_n}{\ell_n^2}$
According to Assumption \ref{asp-sub-sm} and condition (a) in Assumption \ref{asp-lc-sub}, we have
\begin{align}
    &\|T(y)-T(\hat{y})\|^2\notag\\
    =&\|y-\eta_nh_n(x,\,y_{1:n-1},\,y)-\hat{y}+\eta_nh_n(x,\,y_{1:n-1},\,\hat{y})\|^2\notag\\
    =&\|y-\hat{y}\|^2+\eta_n^2\|h_n(x,\,y_{1:n-1},\,y)-h_n(x,\,y_{1:n-1},\,\hat{y})\|^2\notag\\
    &-2\eta_n\langle h_n(x,\,y_{1:n-1},\,y)-h_n(x,\,y_{1:n-1},\,\hat y),\,y-\hat y\rangle\notag\\
    \le&(1+\ell_n^2\eta^2-2\mu_n\eta_n)\|y-\hat{y}\|^2.
\end{align}
Since $1+\ell_n^2\eta^2-2\mu_n\eta_n<1$, $T_n$ is a contractive operator on $\R^{d_n}$, which implies existence and uniqueness of fixed point of $T$. In other words, there exists an unique point $y_n^\ast(x,\,y_{1:n-1})$ satisfying $T_n(y_n^\ast(x,\,y_{1:n-1}))=y_n^\ast(x,\,y_{1:n-1}),$
which leads to
\begin{align}
    h(x,\,y_{1:n-1},\,y^\ast(x,\,y_{1:n-1}))=0.
\end{align}

Given any $x,\hat x, y_{1:n-1}, \hat y_{1:n-1}$, for simplicity, we use $y_n^\ast$ to denote $y^\ast_n(x,\,y_{1:n-1})$, and $\hat y^\ast_n$ to denote $y_n^\ast(\hat x,\,\hat y_{1:n-1})$. Then we have
\begin{align}\label{lma1-eq1}
    &\langle h(x,\,y_{1:n-1},\,\hat y^\ast_n)-h(\hat x,\,\hat y_{1:n-1},\,y^\ast_n),\hat y^\ast_n-y^\ast_n\rangle\notag\\
    =&\langle h(x,\,y_{1:n-1},\,\hat y^\ast_n)-h(x,\,y_{1:n-1},\,y^\ast_n),\hat y^\ast_n-y^\ast_n\rangle\notag\\
    &+\langle h(\hat x,\,\hat y_{1:n-1},\,\hat y^\ast_n)-h(\hat x,\,\hat y_{1:n-1},\,y^\ast_n),\hat y^\ast_n-y^\ast_n\rangle\notag\\
    \ge&\mu_n\|\hat y^\ast_n-y^\ast_n\|^2+\mu_n\|\hat y^\ast_n-y^\ast_n\|^2,
\end{align}
where the last inequality is from Assumption \ref{asp-sub-sm}, and
\begin{align}\label{lma1-eq2}
    &\langle h(x,\,y_{1:n-1},\,\hat y^\ast_n)-h(\hat x,\,\hat y_{1:n-1},\,y^\ast_n),\hat y^\ast_n-y^\ast_n\rangle\notag\\
    =&\langle h(x,\,y_{1:n-1},\,\hat y^\ast_n),\hat y^\ast_n-y^\ast_n\rangle+\langle -h(\hat x,\,\hat y_{1:n-1},\,y^\ast_n),\hat y^\ast_n-y^\ast_n\rangle\notag\\
    \le&\|h(x,\,y_{1:n-1},\,\hat y^\ast_n)\|\|\hat y^\ast_n-y^\ast_n\|\!+\!\|h(\hat x,\,\hat y_{1:n-1},\,y^\ast_n)\|\|\hat y^\ast_n-y^\ast_n\|\notag\\
    =&\|h(x,\,y_{1:n-1},\,\hat y^\ast_n)-h(\hat x,\,\hat y_{1:n-1},\,\hat y^\ast_n)\|\|\hat y^\ast_n-y^\ast_n\|\notag\\
    &+\|h(\hat x,\,\hat y_{1:n-1},\,y^\ast_n)-h(x,\,y_{1:n-1},\,y^\ast_n)\|\|\hat y^\ast_n-y^\ast_n\|\notag\\
    \le&2\ell_n(\|x-\hat x\|+\|y_{1:n-1}-\hat y_{1:n-1}\|)\|\hat y^\ast_n-y^\ast_n\|.
\end{align}
From \eqref{lma1-eq1} and \eqref{lma1-eq2}, we have
\begin{align}
    \|\hat y^\ast_n-y^\ast_n\|\le\frac{\ell_n}{\mu_n}(\|x-\hat x\|+\|y_{1:n-1}-\hat y_{1:n-1}\|),
\end{align}
which implies $y_n^\ast(\cdot)$ is $\frac{\ell_n}{\mu_n}$-Lipschitz continuous.

(ii) The case with condition (b) in Assumption \ref{asp-lc-sub}: For any $n\in[N]$, given $x$ and $y_{1:n-1}$, according to Assumption \ref{asp-sub-sm} and \ref{asp-lc-sub}(b), we have $A_n(x,y_{1:n-1})\succeq\mu_n\bI$. Then the fixed point of $h_n$ exists and is unique, which is given as
\begin{align}
    y^\ast(x,y_{1:n-1})=-A_n^{-1}(x,y_{1:n-1})b_n(x,y_{1:n-1}).
\end{align}
Given any $x,\hat x, y_{1:n-1}, \hat y_{1:n-1}$, we have
\begin{align}
    &\|y^\ast(x,y_{1:n-1})-y^\ast(\hat x,\hat y_{1:n-1})\|\notag\\
    =&\|A_n^{-1}(x,y_{1:n-1})b_n(x,y_{1:n-1})\!-\!A_n^{-1}(\hat x,\hat y_{1:n-1})b_n(\hat x,\hat y_{1:n-1})\|\notag\\
    \le&\|A_n^{-1}(x,y_{1:n-1})b_n(x,y_{1:n-1})\!-\!A_n^{-1}(x,y_{1:n-1})b_n(\hat x,\hat y_{1:n-1})\|\notag\\
    +&\|A_n^{-1}(x,y_{1:n-1})b_n(\hat x,\hat y_{1:n-1})\!-\!A_n^{-1}(\hat x,\hat y_{1:n-1})b_n(\hat x,\hat y_{1:n-1})\|\notag\\
    \le&\frac{1}{\mu_n}\|b_n(x,y_{1:n-1})-b_n(\hat x,\hat y_{1:n-1})\|\notag\\
    &+\ell_{b,n}^\prime\|A_n^{-1}(x,y_{1:n-1})-A_n^{-1}(\hat x,\hat y_{1:n-1})\|\notag\\
    \overset{(a)}{\le}&\left(\frac{\ell_{b,n}}{\mu_n}+\frac{\ell_{b,n}^\prime\ell_{A,n}}{\mu_n^2}\right)(\|x-\hat x\|+\|y_{1:n-1}-\hat y_{1:n-1}\|),
\end{align}
where (a) is from Lipschitz continuity of $b_n$ and
\begin{align}
    &\|A_n^{-1}(x,y_{1:n-1})-A_n^{-1}(\hat x,\hat y_{1:n-1})\|\notag\\
    =&\|A_n^{-1}(x,y_{1:n-1})A_n(\hat x,\hat y_{1:n-1})A_n^{-1}(\hat x,\hat y_{1:n-1})\notag\\
    &-A_n^{-1}(x,y_{1:n-1})A_n(x,y_{1:n-1})A_n^{-1}(\hat x,\hat y_{1:n-1})\|\notag\\
    \le&\frac{1}{\mu_n^2}\|A_n(\hat x,\hat y_{1:n-1})-A_n(x,y_{1:n-1})\|\notag\\
    \le&\frac{\ell_{b,n}^\prime\ell_{A,n}}{\mu_n^2}(\|x-\hat x\|+\|y_{1:n-1}-\hat y_{1:n-1}\|).
\end{align}

\section{Proof of Lemma \ref{lma:conv-sub}}
We use $y^{k,\ast}_n$ to denote $y^\ast(x^k,y^k_{1:n-1})$ and $z^k_n$ to collect $x^k$ and $y_{1:n}^k$. For any $n\in[N]$, we define that
\begin{align}
    &\hat{y}^{k+1}_n=\hat{y}^k_n-\beta_{n}h_n(x^k,y^k_{1:n-1},\hat y^k_n),\ \hat{y}^0_n=y^0_n.
\end{align}
When condition (b) holds in Assumption \ref{asp-lc-sub} instead of condition (a), we define that $\ell_n:=\frac{\ell_{A,n}\ell_{b,n}^\prime}{\mu_n}+\ell_{b,n}$.
By selecting $\beta_n\le\frac{\mu_n}{\ell_n^2}$, we bound the difference between $y^k_n$ and $\hat y^k_n$ as
\begin{align}
    &\E\|y^{k}_n-\hat{y}^k_n\|^2\notag\\
    =&\E\|y^{k-1}_n-\beta_{n}(h_n(x^{k-1},y^{k-1}_{1:n})+\psi^{k-1}_n)-\hat{y}^{k-1}_n\notag\\
    &+\beta_nh_n(x^{k-1},y^{k-1}_{1:n-1},\hat y^{k-1}_n)\|^2\notag\\
    =&\E\|y^{k-1}_n-\beta_nh_n(x^{k-1},y^{k-1}_{1:n})-\hat{y}^{k-1}_n\notag\\
    &+\beta_nh_n(x^{k-1},y^{k-1}_{1:n-1},\hat y^{k-1}_n)\|^2+\beta_n^2\E\|\psi_n^{k-1}\|^2\notag\\
    \leq&(1-\mu_n\beta_n)\E\|y^{k-1}_n-\hat{y}^{k-1}_n\|^2+\beta_n^2\E\|\psi_n^{k-1}\|^2\notag\\
    \leq&\sum_{i=0}^{k-1}(1-\mu_n\beta_n)^i\beta^2_n\E\|\psi_n^{i}\|^2.\label{conv-sub-result1}
\end{align}

Before bounding difference between $\hat y^k_n$ and $y^{k,\ast}_n$, we show the following statements are valid.
(i) For any $i$, under condition (a) in Assumption \ref{asp-lc-sub}, it holds
\begin{align}
    &\|h_n(x^{k},y^{k}_{1:n-1},\hat y^{i}_n)-h_n(x^{i},y^{i}_{1:n-1},\hat y^{i}_n)\|\notag\\
    \le&\ell_n\|z_{n-1}^k-z_{n-1}^{i}\|.
\end{align}
(ii) Under condition (b) in Assumption \ref{asp-lc-sub}, by setting $y_n^0=0$, it holds
\begin{align}
    &\|h_n(x^{k},y^{k}_{1:n-1},\hat y^{i}_n)-h_n(x^{i},y^{i}_{1:n-1},\hat y^{i}_n)\|\notag\\
    =&\|A_n(x^{k},y^{k}_{1:n-1})\hat y^{i}_n+b_n(x^{k},y^{k}_{1:n-1})\notag\\
    &-A_n(x^{i},y^{i}_{1:n-1})\hat y^{i}_n-b_n(x^{i},y^{i}_{1:n-1})\|\notag\\
    \le&\|(A_n(x^{k},y^{k}_{1:n-1})-A_n(x^{i},y^{i}_{1:n-1}))\hat y^{i}_n\|\notag\\
    &+\|b_n(x^{k},y^{k}_{1:n-1})-b_n(x^{i},y^{i}_{1:n-1})\|\notag\\
    \le&\|A_n(x^{k},y^{k}_{1:n-1})-A_n(x^{i},y^{i}_{1:n-1})\|\cdot\|\hat y^{i}_n\|\notag\\
    &+\ell_{b,n}\|z_{n-1}^k-z_{n-1}^{i}\|\notag\\
    \le&(\ell_{A,n}\|\hat y^{i}_n\|+\ell_{b,n})\|z_{n-1}^k-z_{n-1}^{i}\|\notag\\
    \overset{(a)}{\le}&\ell_n\|z_{n-1}^k-z_{n-1}^{i}\|,
\end{align}
where (a) is from $\ell_n=\frac{\ell_{A,n}\ell_{b,n}^\prime}{\mu_n}+\ell_{b,n}$ and
\begin{align}
    &\|\hat y^{i}_n\|\!=\!\|(\bI-\beta_nA_n(x^{i-1},y^{i-1}_{1:n-1}))\hat y^{i-1}_n-\beta_nb_n(x^{i-1},y^{i-1}_{1:n-1})\|\notag\\
    &\le\|(\bI-\beta_nA_n(x^{i-1},y^{i-1}_{1:n-1}))\hat y^{i-1}_n\|+\beta_n\ell_{b,n}^\prime\notag\\
    &\le(1-\mu_n\beta_n)\|\hat y^{i-1}_n\|+\beta_n\ell_{b,n}^\prime\notag\\
    &\le\sum_{j=1}^i(1-\mu_n\beta_n)^{i-j}\beta_n\ell_{b,n}^\prime\le\frac{\ell_{b,n}^\prime}{\mu_n}.
\end{align}

For any $0\le i<k$, we have
\begin{align}
    &\E\|\hat y^{i+1}_n-y^{k,\ast}_n\|^2=\E\|\hat y^{i}_n-\beta_nh_n(x^{i},y^{i}_{1:n-1},\hat y^{i}_n)-{y}^{k,\ast}_n\|^2\notag\\
    =&\E\|\hat y^{i}_n-\beta_nh_n(x^{k},y^{k}_{1:n-1},\hat y^{i}_n)-{y}^{k,\ast}_n\notag\\
    &+\beta_nh_n(x^{k},y^{k}_{1:n-1},\hat y^{i}_n)-\beta_nh_n(x^{i},y^{i}_{1:n-1},\hat y^{i}_n)\|^2\notag\\
    \overset{(a)}{\le}&(1+\frac{\mu_n}{2}\beta_n)\E\|\hat y^{i}-\beta_nh_n(x^{k},y^{k}_{1:n-1},\hat y^{i}_n)-{y}^{k,\ast}_n\|^2\notag\\
    &+\beta_n(\beta_n\!+\!\frac{2}{\mu_n})\E\|h_n(x^{k},y^{k}_{1:n-1},\hat y^{i}_n)\!-\!h_n(x^{i},y^{i}_{1:n-1},\hat y^{i}_n)\|^2\notag\\
    \overset{(b)}\le&(1-\frac{\mu_n}{2}\beta_n)\E\|\hat y^{i}_n-{y}^{k,\ast}_n\|^2\notag\\
    &+\beta_n(\beta_n+\frac{2}{\mu_n})\ell_{h,n}^2\E\|z_{n-1}^k-z_{n-1}^{i}\|^2,\label{conv-sub-eq2}
\end{align}
where (a) is from Young's inequality and (b) is from
\begin{align}
    &\E\|\hat y^{i}-\beta_nh_n(x^{k},y^{k}_{1:n-1},\hat y^{i}_n)-{y}^{k,\ast}_n\|^2\notag\\
    \le&\E\|\hat y^{i}-{y}^{k,\ast}_n\|^2+\beta_n^2\E\|h_n(x^{k},y^{k}_{1:n-1},\hat y^{i}_n)\|^2\notag\\
    -&2\beta_n\E\langle h_n(x^{k},y^{k}_{1:n-1},\hat y^{i}_n),\hat y^{i}-{y}^{k,\ast}_n\rangle\notag\\
    \overset{(a)}{=}&\E\|\hat y^{i}-{y}^{k,\ast}_n\|^2\notag\\
    +&\beta_n^2\E\|h_n(x^{k},y^{k}_{1:n-1},\hat y^{i}_n)-h_n(x^{k},y^{k}_{1:n-1},y^{k,\ast}_n)\|^2\notag\\
    -&2\beta_n\E\langle h_n(x^{k},y^{k}_{1:n-1},\hat y^{i}_n)-h_n(x^{k},y^{k}_{1:n-1},y^{k,\ast}_n),\hat y^{i}-{y}^{k,\ast}_n\rangle\notag\\
    \overset{(b)}{\le}&(1+\beta_n^2\ell_n^{2}-2\mu_n\beta_n)\E\|\hat y^{i}-{y}^{k,\ast}_n\|^2\notag\\
    \overset{(c)}{\le}&(1-\mu_n\beta_n)\E\|\hat y^{i}-{y}^{k,\ast}_n\|^2
\end{align}
where (a) is from $h_n(x^{k},y^{k}_{1:n-1},y^{k,\ast}_n)=0$, (b) is from Assumptions \ref{asp-sub-sm} and \ref{asp-lc-sub}, and (c) is from $\beta_n\le\frac{\mu_n}{\ell_n^2}$. From \eqref{conv-sub-eq2}, we have
\begin{align}
    &\E\|\hat y^{k}_n-y^{k,\ast}_n\|^2\notag\\
    \le&\ell_{n}^2\beta_n(\beta_n+\frac{2}{\mu_n})\sum_{i=0}^{k-1}(1-\frac{\mu_n}{2}\beta_n)^{k-i-1}\E\|z_{n-1}^k-z_{n-1}^{i}\|^2\notag\\
    &+(1-\frac{\mu_n}{2}\beta_n)^k\E\|y^0_n-{y}^{k,\ast}_n\|^2.\label{conv-sub-result2}
\end{align}
Summing up \eqref{conv-sub-result1} and \eqref{conv-sub-result2}, we have
\begin{align*}
    &\E\|y^{k}_n-y^{k,\ast}_n\|^2\le2\E\|y^{k}_n-\hat y^k_n\|^2+2\E\|\hat y^k_n-y^{k,\ast}_n\|^2\notag\\
    \overset{(a)}{\le}&2\ell_{n}^2\beta_n(\beta_n+\frac{2}{\mu_n})\sum_{i=0}^{k-1}(1-\frac{\mu_n}{2}\beta_n)^{k-i-1}\E\|z_{n-1}^k-z_{n-1}^{i}\|^2\notag\\
    +&2(1\!-\!\frac{\mu_n}{2}\beta_n)^k\E\|y_n^0\!-\!y^{k,\ast}_n\|^2
    \!+\!\sum_{i=0}^{k-1}2(1\!-\!\mu_n\beta_n)^{k\!-\!1\!-\!i}\beta^2_n\E\|\psi_n^i\|^2,
\end{align*}
where (a) completes the proof of Lemma \ref{lma:conv-sub}.

\section{Auxiliary Lemma}\label{sec-apx-aux}
\begin{lemma}\label{lma-aux-1}
    Suppose Assumptions \ref{asp-sub-sm} and \ref{asp-lc-main} hold. For any $n\in[N]$, $y_n^\diamond(x)$ is $L_{y,n}^\prime$-Lipschitz continuous, where $L_{y,1}^\prime:=L_{y,1}$ and $L_{y,n}^\prime:=L_{y,n}(1+\sum_{i=1}^{n-1}L_{y,i}^\prime)$. Additionally, $v(x,y_{1:N}^\diamond(x))$ is $L_v$-Lipschitz continuous, where $L_v:=\ell_0+\sum_{n=1}^NL_{y,n}^\prime$. Furthermore, for any $x,y_{1:N}$, it holds that
    \begin{align*}
        &\|v(x,y^\diamond_{1:N}(x))\!-\!v(x,y_{1:N})\|\!\le\!\sum_{n=1}^N\ell_{v,n}\|y_n\!-\!y^\ast_n(x,y_{1:n-1})\|,
    \end{align*}
    where $\ell_{v,n}:=\ell_0\prod_{i=n+1}^{N}(1+L_{y,n})$ and $\ell_{v,n}:=\ell_0$.
\end{lemma}
\begin{proof}
    Since $y^\diamond_1=y^\ast_1$, according to Lemma \ref{lma:Lip-fp}, $y^\diamond_1$ is $L_{y,1}$-Lipschitz continuous.
    Suppose that for any $i<n$, $y_i^\diamond$ is $L_{y,i}^\prime$-Lipschitz continuous. Given any $x,\Tilde x\in\R^{d_0}$, we have
    \begin{align}
        &\|y_n^\diamond(x)-y_n^\diamond(\Tilde x)\|=\|y^\ast_n(x,y_{1:n-1}^\diamond(x))-y^\ast_n(\Tilde x,y_{1:n-1}^\diamond(\Tilde x))\|\notag\\
        \le&L_{y,n}(\|x-\Tilde x\|+\|y_{1:n-1}^\diamond(x)-y_{1:n-1}^\diamond(\Tilde x)\|)\notag\\
        \le&L_{y,n}(1+\sum_{i=1}^{n-1}L_{y,i}^\prime)\|x-\Tilde x\|.
    \end{align}
    Thus $y_n^\diamond$ is $L_{y,n}^\prime$-Lipschitz continuous, where $L_{y,n}^\prime=L_{y,n}(1+\sum_{i=1}^{n-1}L_{y,i}^\prime).$
    % \begin{align}
    %     L_{y,n}^\prime=&L_{y,n}(1+\sum_{i=1}^{n-1}L_{y,i}^\prime)\notag\\
    %     =&L_{y,n}(1+\sum_{i=1}^{n-1}L_{y,i}\Pi_{j=i+1}^{n-1}(1+L_{y,j})).
    % \end{align}
    Combining with Lipschitz continuity of $v$, given any $x,\Tilde x\in\R^{d_0}$, we have
    \begin{align}
        &\|v(x,y_{1:N}^\diamond(x))-v(\Tilde x,y_{1:N}^\diamond(\Tilde x))\|\notag\\
        \le&\ell_0(\|x-\Tilde x\|+\|y_{1:N}^\diamond(x)-y_{1:N}^\diamond(\Tilde x)\|)\notag\\
        \le&(\ell_0+\sum_{n=1}^NL_{y,n}^\prime)\|x-\Tilde x\|.
    \end{align}

    According to Assumption \ref{asp-lc-main}, for any $x,y_{1:N}$, it holds:
    \begin{align}\label{AuxLma-eq1}
        \|v(x,y^\diamond_{1:N}(x))-v(x,y_{1:N})\|\le\ell_0\|y_{1:N}-y^\diamond_{1:N}(x)\|.
    \end{align}
    For any $n\in[N]$, by Lemma \ref{lma:Lip-fp}, we have
    \begin{align}
        &\|y_n-y^\diamond_{n}(x)\|\notag\\
        \le&\|y_n-y^\ast_n(x,y_{1:n-1})\|+\|y^\ast_n(x,y_{1:n-1})-y^\diamond_{n}(x)\|\notag\\
        \le&\|y_n-y^\ast_n(x,y_{1:n-1})\|+L_{y,n}\sum_{i=1}^{n-1}\|y_{i}-y^\diamond_{i}(x)\|,
    \end{align}
    which implies
    \begin{align}
        \|y_{1:n}-y^\diamond_{1:n}(x)\|\le&\|y_n-y^\ast_n(x,y_{1:n-1})\|\notag\\
        +&(1+L_{y,n})\|y_{1:n-1}-y^\diamond_{1:n-1}(x)\|.
    \end{align}
    Thus, it holds that
    \begin{align}
        &\|y_{1:N}-y^\diamond_{1:N}(x)\|\le\sum_{n=1}^N\ell_{v,n}^\prime\|y_n-y^\ast_n(x,y_{1:n-1})\|,\label{AuxLma-eq2}
    \end{align}
    where $\ell_{v,n}^\prime:=\prod_{i=n+1}^{N}(1+L_{y,n})$ and $\ell_{v,n}^\prime=1$. Combining \eqref{AuxLma-eq1} and \eqref{AuxLma-eq2} completes the proof.
\end{proof}

\begin{lemma}\label{apx-aux-lma-sub}
    Suppose Assumptions \ref{asp-sub-sm}-\ref{asp-noise} hold. If the step sizes are chosen as $\beta_n^k=\beta_n\le\frac{\mu_n}{\ell_n^2}$ for any $n\in[N]$ and any $K$, then it holds that
    \begin{align}
        \E\|y^{k}_n-y^{k,\ast}_n\|^2&\le4(1-\frac{\mu_n}{2}\beta_n)^k\E\|y_n^0-y^{0,\ast}_n\|^2+C_{n,\sigma}^{k}\alpha\sigma^2\notag\\
        &\!+\!\sum_{l=1}^{N}\sum_{j=0}^{k-1}(C_{n,l,h}^{k-j-1}+D_{n,l,h}^{k})\E\|h_l(x^j,y_{1:l}^j)\|^2\notag\\
        &\!+\!\sum_{j=0}^{k-1}(C_{n,v}^{k-j-1}\!+\!D_{n,v}^{k})\E\|v(x^j,y_{1:N}^j)\|^2,\label{apx-aux-lma-result}
    \end{align}
    where constants $C_{n,\sigma}^{k},\,C_{n,v}^{k-j-1},\,D_{n,v}^{k},\,C_{n,l,h}^{k-j-1},\,D_{n,l,h}^{k}$ are defined in \eqref{apx-aux-lma-para}.
\end{lemma}
\begin{proof}
Combining \eqref{lma:conv-sub}, Assumption \ref{asp-noise} and
\begin{align}
    &\E\|y^0_n-{y}^{k,\ast}_n\|^2\le2\E\|y^0_n-y^{0,\ast}_n\|^2+2\E\|y^{0,\ast}_n-y^{k,\ast}_n\|^2\notag\\
    \le&2\E\|y^0_n-y^{0,\ast}_n\|^2+2L_{y,n}^2\E\|z^k_{n-1}-z^0_{n-1}\|^2,
\end{align}
we have
\begin{align}
    &\E\|y^{k}_n-y^{k,\ast}_n\|^2\notag\\
    {\le}&2\ell_{n}^2\beta_n(\beta_n+\frac{2}{\mu_n})\sum_{i=0}^{k-1}(1-\frac{\mu_n}{2}\beta_n)^{k-i-1}\E\|z_{n-1}^k-z_{n-1}^{i}\|^2\notag\\
    +&4(1-\frac{\mu_n}{2}\beta_n)^k(\E\|y_n^0-y^{0,\ast}_n\|^2+L_{y,n}^2\E\|z^k_{n-1}-z^0_{n-1}\|^2)\notag\\
    +&\sum_{i=0}^{k-1}2(1-\mu_n\beta_n)^{k-1-i}\beta^2_n(\omega_0\|h_n(x^i,y_{1:n}^i)\|^2+\sigma^2).\label{conv-sub-final}
\end{align}
For simplicity, we define that
\begin{align}
    &\Tilde{y}_n^{k+1}=\Tilde{y}_n^k-\beta_nh(x^k,y_{1:n}^k),\ \Tilde{y}_n^{0}=y_n^0,\\
    &\Tilde{x}^{k+1}=\Tilde{x}^k-\alpha v(x^k,y_{1:N}^k),\ \Tilde{x}^{0}=x^0.
\end{align}
and denote $(\Tilde x^k, \Tilde{y^k})$ as $\Tilde z^k$. Thus, for any $n\in[N]$ and $i<k$, we have $x^k-x^i=\Tilde x^k-\Tilde x^i-\alpha\sum_{j=i}^{k-1}\alpha\xi^j$, and so is $y_n^k$. 
For any $n\in[N]$ and $i<k$, it holds that
\begin{align}
    &\E\|z_{n}^k-z_{n}^i\|^2\notag\\
    \le&2\E\|\tilde z_{n}^k-\tilde z_n^i\|^2+2\E\|\sum_{j=i}^{k-1}\alpha\xi^j\|^2+2\sum_{l=1}^{n}\E\|\sum_{j=i}^{k-1}\beta_l\psi_l^j\|^2\notag\\
    \overset{(a)}\le&2(k-i)\sum_{j=i}^{k-1}\E\|\tilde z_{n}^{j+1}-\tilde z_n^j\|^2+2\omega_1\alpha^2\sum_{j=i}^{k-1}\|v(x^j,y_{1:N}^j)\|^2\notag\\
    &+2\alpha^2\sum_{j=i}^{k-1}(\omega_2\sum_{l=1}^N\|h_l(x^j,y_{1:l}^j)\|^2+\sigma^2)\notag\\
    &+2\sum_{l=1}^{n}\sum_{j=i}^{k-1}\beta_l^2(\omega_0\E\|h_l(x^j,y_{1:l}^j)\|^2+\sigma^2)\notag\\
    =&2\alpha^2(k-i+\omega_1)\sum_{j=i}^{k-1}\E\|v(x^j,y_{1:N}^j)\|^2\notag\\
    &+\sum_{l=1}^{N}2\omega_1\alpha^2\sum_{j=i}^{k-1}\E\|h_l(x^j,y_{1:l}^j)\|^2\!+\!2(k-i)(\alpha^2\!+\!\sum_{l=1}^n\beta_l^2)\sigma^2\notag\\
    &+\sum_{l=1}^{n}2\beta_l^2(k-i+\omega_0)\sum_{j=i}^{k-1}\E\|h_l(x^j,y_{1:l}^j)\|^2,\label{apx-thm-sm-eq5}
\end{align}
where (a) is from Assumption \ref{asp-noise}. Then we have
\begin{align}
    &\sum_{i=0}^{k-1}(1-\frac{\mu_n}{2}\beta_n)^{k-i-1}\E\|z_{n-1}^k-z_{n-1}^{i}\|^2\notag\\
    \le&\sum_{i=0}^{k-1}(1-\frac{\mu_n}{2}\beta_n)^{k-i-1}\Big[2\alpha^2(k-i+\omega_1)\sum_{j=i}^{k-1}\E\|v(x^j,y_{1:N}^j)\|^2\notag\\
    &+\sum_{l=1}^{N}2\omega_1\alpha^2\sum_{j=i}^{k-1}\E\|h_l(x^j,y_{1:l}^j)\|^2\!+\!2(k-i)(\alpha^2\!+\!\sum_{l=1}^n\beta_l^2)\sigma^2\notag\\
    &+\sum_{l=1}^{n-1}2\beta_l^2(k-i+\omega_0)\sum_{j=i}^{k-1}\E\|h_l(x^j,y_{1:l}^j)\|^2\Big].
\end{align}
Rearranging the coefficient of each term in the above inequality, we have
\begin{align}
    &\sum_{i=0}^{k-1}(1-\frac{\mu_n}{2}\beta_n)^{k-i-1}\E\|z_{n-1}^k-z_{n-1}^{i}\|^2\notag\\
    \le&\frac{\alpha^2}{\beta_n}\sum_{j=0}^{k-1}c_{n,0,v}^{k-j-1}\E\|v(x^j,y_{1:N}^j)\|^2+\frac{c_{n,2}\sigma^2}{\beta_n}(\alpha^2+\sum_{l=1}^n\beta_l^2)\notag\\
    &+\sum_{l=1}^{N}\sum_{j=0}^{k-1}\frac{\alpha^2}{\beta_n}c_{n,1}^{k-j-1}\E\|h_l(x^j,y_{1:l}^j)\|^2\notag\\
    &+\sum_{l=1}^{n-1}\frac{\beta_l^2}{\beta_n}\sum_{j=0}^{k-1}c_{n,0,h}^{k-j-1}\E\|h_l(x^j,y_{1:l}^j)\|^2,\label{apx-thm-sm-eq6}
\end{align}
where $c_{n,0}^{j},c_{n,1}^{j},c_{n,2}$ are defined as
\begin{subequations}
    \begin{align}
        &c_{n,0,v}^{j}:=\frac{4}{\mu_n}(1-\frac{\mu_n}{2}\beta_n)^{j}(j+1+\omega_1+\frac{2}{\mu_n\beta_n}),\\
        &c_{n,0,h}^{j}:=\frac{4}{\mu_n}(1-\frac{\mu_n}{2}\beta_n)^{j}(j+1+\omega_0+\frac{2}{\mu_n\beta_n}),\\
        &c_{n,1}^{j}:=\frac{4}{\mu_n}(1-\frac{\mu_n}{2}\beta_n)^{j},\ c_{n,2}:=\frac{4}{\mu_n}(1+\frac{2}{\mu_n\beta_n}).
    \end{align}
\end{subequations}
Substituting \eqref{apx-thm-sm-eq5} and \eqref{apx-thm-sm-eq6} into \eqref{conv-sub-final}, we have
\begin{align}
    &\E\|y^{k}_n-y^{k,\ast}_n\|^2\notag\\
    \le&2\ell_{n}^2(\beta_n+\frac{2}{\mu_n}){\alpha^2}\sum_{j=0}^{k-1}c_{n,0,v}^{k-j-1}\E\|v(x^j,y_{1:N}^j)\|^2\notag\\
    &+2\ell_{n}^2(\beta_n+\frac{2}{\mu_n})\sum_{l=1}^{N}\alpha^2\sum_{j=0}^{k-1}c_{n,1}^{k-j-1}\E\|h_l(x^j,y_{1:l}^j)\|^2\notag\\
    &+2\ell_{n}^2(\beta_n+\frac{2}{\mu_n})\sum_{l=1}^{n-1}{\beta_l^2}\sum_{j=0}^{k-1}c_{n,0,h}^{k-j-1}\E\|h_l(x^j,y_{1:l}^j)\|^2\notag\\
    &+2\ell_{n}^2(\beta_n+\frac{2}{\mu_n}){c_{n,2}\sigma^2}(\alpha^2+\sum_{l=1}^n\beta_l^2)\notag\\
    &+4(1-\frac{\mu_n}{2}\beta_n)^k\E\|y_n^0-y^{0,\ast}_n\|^2\notag\\
    &+8(1-\frac{\mu_n}{2}\beta_n)^kL_{y,n}^2\alpha^2(k+\omega_1)\sum_{j=0}^{k-1}\E\|v(x^j,y_{1:N}^j)\|^2\notag\\
    &+8(1-\frac{\mu_n}{2}\beta_n)^kL_{y,n}^2\sum_{l=1}^{N}\omega_1\alpha^2\sum_{j=0}^{k-1}\E\|h_l(x^j,y_{1:l}^j)\|^2\notag\\
    &+8(1-\frac{\mu_n}{2}\beta_n)^kL_{y,n}^2\sum_{l=1}^{n-1}\beta_l^2(k+\omega_0)\sum_{j=0}^{k-1}\E\|h_l(x^j,y_{1:l}^j)\|^2\notag\\
    &+8(1-\frac{\mu_n}{2}\beta_n)^kL_{y,n}^2k(\alpha^2\!+\!\sum_{l=1}^n\beta_l^2)\sigma^2+\frac{2\beta_n\sigma^2}{\mu_n}\notag\\
    &+\sum_{j=0}^{k-1}2(1-\mu_n\beta_n)^{k-1-j}\beta^2_n\omega_0\E\|h_n(x^{j},y^{j}_{1:n})\|^2.\label{apx-thm-sm-eq7}
\end{align}

Next, we rearrange the above inequality. Define that $\Tilde c_{n,0,v}^{j}:=2\ell_{n}^2(\beta_n+\frac{2}{\mu_n})c_{n,0,v}^{j}$, $\Tilde c_{n,0,h}^{j}:=2\ell_{n}^2(\beta_n+\frac{2}{\mu_n})c_{n,0,h}^{j}$, $\Tilde c_{n,1}^{j}:=2\ell_{n}^2(\beta_n+\frac{2}{\mu_n})c_{n,1}^{j}$, $\Tilde c_{n,2}:=2\ell_{n}^2(\beta_n+\frac{2}{\mu_n}){c_{n,2}}$, $\Tilde c_{n,3}^k:=8(1-\frac{\mu_n}{2}\beta_n)^kL_{y,n}^2$ and
\begin{subequations}\label{apx-aux-lma-para}
    \begin{align}
        &C_{n,l,h}^{j}:=\left\{
        \begin{aligned}
            &\Tilde c_{n,1}^{j}\alpha^2+\Tilde c_{n,0,h}^{j}\beta_l^2,\quad&l<n\\
            &\Tilde c_{n,1}^{j}\alpha^2+2(1-\mu_n\beta_n)^{j}\beta^2_n\omega_0,&l=n\\
            &\Tilde c_{n,1}^{j}\alpha^2,&l>n
        \end{aligned}
        \right.,\\
        &D_{n,l,h}^k:=\left\{
        \begin{aligned}
            &\Tilde c_{n,3}^k\omega_1\alpha^2+\Tilde c_{n,3}^k\beta_l^2(k+\omega_0),\quad&l<n\\
            &\Tilde c_{n,3}^k\omega_1\alpha^2,&l\ge n
        \end{aligned}
        \right.,\\
        &C_{n,v}^{j}:=\Tilde c_{n,0,v}^{j}\alpha^2,\ D_{n,v}^{k}:=\Tilde c_{n,3}^k\alpha^2(k+\omega_1)\\
        &C_{n,\sigma}^{k}:=(\Tilde c_{n,2}+k\Tilde c_{n,3}^k)(\alpha\!+\!\sum_{l=1}^n\beta_l^2/\alpha)+\frac{2\beta_n}{\mu_n\alpha}.\label{apx-sm-Cnsigma}
    \end{align}
\end{subequations}
Then \eqref{apx-thm-sm-eq7} can be rearranged as
\begin{align}
    &\E\|y^{k}_n-y^{k,\ast}_n\|^2\notag\\
    \le&4(1-\frac{\mu_n}{2}\beta_n)^k\E\|y_n^0-y^{0,\ast}_n\|^2+C_{n,\sigma}^{k}\alpha\sigma^2\notag\\
    &+\sum_{j=0}^{k-1}(C_{n,v}^{k-j-1}+D_{n,v}^{k})\E\|v(x^j,y_{1:N}^j)\|^2\notag\\
    &+\sum_{l=1}^{N}\sum_{j=0}^{k-1}(C_{n,l,h}^{k-j-1}+D_{n,l,h}^{k})\E\|h_l(x^j,y_{1:l}^j)\|^2,\label{conv-sub-finally}
\end{align}
which completes the proof.
\end{proof}

\section{Proof of Lemma \ref{lma-sm-main} and Theorem \ref{thm-MSSA-sm}}\label{apx-thm-sm}

Based on the update of $x^k$, we have
\begin{align}
    &\E\|x^{k+1}-x^\ast\|^2\notag\\
    =&\E\|x^k-\alpha(v(x^k,y^k_{1:N})+\xi^k)-x^\ast\|^2\notag\\
    =&\E\|x^k-\alpha v(x^k,y^k_{1:N})-x^\ast\|^2+\alpha^2\E\|\xi\|^2\notag\\
    \leq&(1+\frac{\mu_0\alpha}{2})\E\|x^k-\alpha v(x^k,y_{1:N}^\diamond(x^k))-x^\ast\|^2\notag\\
    &+\alpha(\alpha+\frac{2}{\rho_0})\E\|v(x^k,y_{1:N}^\diamond(x^k))-v(x^k,y^k_{1:N})\|^2\notag\\
    &+\alpha^2(\omega_1\|v(x^k,y_{1:N}^k)\|^2+\omega_2\sum_{n=1}^N\|h_n(x^k,y_{1:n}^k)\|^2+\sigma^2)\notag\\
    \leq&(1+\frac{\mu_0\alpha}{2})\E\|x^k-\alpha v(x^k,y_{1:N}^\diamond(x^k))-x^\ast\|^2\notag\\
    &+\alpha((1+2\omega_1)\alpha+\frac{2}{\rho_0})\E\|v(x^k,y_{1:N}^\diamond(x^k))-v(x^k,y^k_{1:N})\|^2\notag\\
    &+2\omega_1\alpha^2\|v(x^k,y_{1:N}^\diamond(x^k))\|^2+\omega_2\alpha^2\sum_{n=1}^N\|h_n(x^k,y_{1:n}^k)\|^2\notag\\
    &+\alpha^2\sigma^2.\label{apx-sm-eq1}
\end{align}

Because of the Lipschitz continuity and strong monotonicity of $v$, we have
\begin{align}
    &\|x^k-\alpha v(x^k,y_{1:N}^\diamond(x^k))-x^\ast\|^2\notag\\
    =&\|x^k-x^\ast\|^2+(1+2\omega_1-2\omega_1)\alpha^2\|v(x^k,y_{1:N}^\diamond(x^k))\|^2\notag\\
    &-2\alpha\langle v(x^k,y_{1:N}^\diamond(x^k)),x^k-x^\ast\rangle\notag\\
    \le&(1+(1+2\omega_1)L_v^2\alpha^2-2\mu_0\alpha)\|x^k-x^\ast\|^2\notag\\
    &-2\omega_1\alpha^2\|v(x^k,y_{1:N}^\diamond(x^k))\|^2\notag\\
    \le&(1-\mu_0\alpha)\|x^k-x^\ast\|^2-2\omega_1\alpha^2\|v(x^k,y_{1:N}^\diamond(x^k))\|^2,\label{apx-sm-eq2}
\end{align}
where the last inequality is from $\alpha\le\frac{\mu_0}{(1+2\omega_1)L_v^2}$.
According to Lemma \ref{lma-aux-1}, for any $k$, we have
\begin{align}
    &\|v(x^k,y_{1:N}^\diamond(x^k))-v(x^k,y^k_{1:N})\|^2\notag\\
    \le&\sum_{n=1}^N\ell_{v,n}^2\|y_n^k-y^\ast_n(x^k,y_{1:n-1}^k)\|^2.\label{apx-sm-eq3}
\end{align}
Since $h_n$ is $\ell_n$-Lipschitz continuous, for any $k$, we have
\begin{align}
    \|h_n(x^k,y_{1:n}^k)\|^2=&\|h_n(x^k,y_{1:n}^k)-h_n(x^k,y_{1:n-1}^k,y^{\ast,k}_n)\|^2\notag\\
    \le&\ell_n^2\|y_n^k-y^{\ast,k}_n\|^2.\label{apx-sm-eq4}
\end{align}
Substituting \eqref{apx-sm-eq2}-\eqref{apx-sm-eq4} into \eqref{apx-sm-eq1}, we have
\begin{align}
    &\E\|x^{k+1}-x^\ast\|^2\notag\\
    \leq&(1-\frac{\mu_0\alpha}{2})\E\|x^k-x^\ast\|^2\!+\!\omega_2\alpha^2\sum_{n=1}^N\ell_n^2\|y_n^k-y^{\ast,k}_n\|^2\!+\!\alpha^2\sigma^2\notag\\
    &+\alpha((1+2\omega_1)\alpha+\frac{2}{\mu_0})\sum_{n=1}^N\ell_{v,n}^2\|y_n^k-y^{\ast,k}_n\|^2.\label{apx-main-eq1}
\end{align}

Now we define a Lyapunov function:
\begin{align}
    &\V^k:=\E\|x^k-x^\ast\|^2+\sum_{j=0}^{k-1}(a_{0}^{k-j-1}+b_0^{k-1})\alpha\E\|v(x^j,y^j_{1:N})\|^2\notag\\
    &+\sum_{j=0}^{k-1}\left(1-\frac{\mu_0\alpha}{2}\right)^{k-j-1}\frac{(k-j-1)\alpha^2+\alpha}{512L_v^2/\mu_0}\E\|v(x^j,y^j_{1:N})\|^2\notag\\
    &+\sum_{n=1}^N\!\sum_{j=0}^{k-1}\left(1-\frac{\mu_0\alpha}{2}\right)^{k-j-1}\!((k\!-\!j\!-\!1)\alpha^2\!+\!\alpha)\E\|h_n(x^j,y^j_{1:n})\|^2\notag\\
    &+\sum_{n=1}^N\sum_{j=0}^{k-1}(a_{n}^{k-j-1}+b_n^{k-1})\alpha\E\|h_n(x^j,y^j_{1:n})\|^2,
\end{align}
where we guarantee sequence $\{b_n^{k}\}_k$ is a decreasing sequence for any $0\le n\le N$. The definitions of $a_n^k,b_n^k$ are provided later. For simplicity, define that
\begin{align}
    &\ell_v:=\max_{n\in[N]}\{\ell_{v,n}\},\ \ell_h:=\max_{n\in[N]}\{\ell_n\},\\
    &B\!:=\!\max_{n\in[N]}\left\{(\omega_2\ell_n^2\!+\!(1+2\omega_1)\ell_{v,n}^2)\alpha+\frac{2\ell_{v,n}^2}{\mu_0}+\ell_n^2+\frac{\mu_0\ell_{v,n}^2}{256L_v^2}\right\},\\
    &B^k:=B+2(a_0^0+b_0^{k})\ell_v^2+\max_{n\in[N]}(a_n^0+b_n^{k})\ell_h^2.
\end{align}
According to \eqref{apx-main-eq1} and the definition of $\V^k$, we have
\begin{align}
    &\V^{k+1}\le\V^k-\frac{\mu_0\alpha}{2}\E\|x^k-x^\ast\|^2\notag\\
    &+\alpha\sum_{n=1}^N\left((\omega_2\ell_n^2+(1+2\omega_1)\ell_{v,n}^2)\alpha+\frac{2\ell_{v,n}^2}{\mu_0}\right)\E\|y_n^k-y^{\ast,k}_n\|^2\notag\\
    &-\!\frac{\mu_0\alpha}{2}\sum_{j=0}^{k-1}\!\left(1\!-\!\frac{\mu_0\alpha}{2}\right)^{k\!-j\!-1}\!\!((k\!-j\!-\!1)\alpha^2\!+\!\alpha)\Big(\frac{\E\|v(x^j,y^j_{1:N})\|^2}{512L_v^2/\mu_0}\notag\\
    &+\!\sum_{n=1}^N\E\|h_n(x^j,y^j_{1:n})\|^2\Big)\!+\!\sum_{n=1}^N\sum_{j=0}^{k-1}((1\!-\!\frac{\mu_0\alpha}{4})^{k-j}\alpha+a_n^{k-j}\notag\\
    &+b_{n}^{k}-a_n^{k-j-1}-b_{n}^{k-1})\alpha\E\|h_n(x^j,y^j_{1:n})\|^2\!+\!\sum_{j=0}^{k-1}(a_0^{k-j}+b_{0}^{k}\notag\\
    &+\!(1\!-\!\frac{\mu_0\alpha}{4})^{k-j}\frac{\mu_0\alpha}{512L_v^2}\!-\!a_0^{k-j-1}\!-\!b_{0}^{k-1})\alpha\E\|v(x^j,y^j_{1:N})\|^2\notag\\
    &+(a_0^0+b_{0}^{k}+\frac{\mu_0}{512L_v^2})\alpha\E\|v(x^k,y^k_{1:N})\|^2\notag\\
    &+\sum_{n=1}^N(a_n^0+b_n^{k}+1)\alpha\E\|h_n(x^k,y^k_{1:n})\|^2+\alpha^2\sigma^2.
\end{align}
Furthermore, we can bound $\|h_n(x^k,y^k_{1:n})\|^2$ by \eqref{apx-thm-sm-eq6} and bound $\|v(x^k,y^k_{1:N})\|^2$ as follows:
\begin{align}
    &\|v(x^k,y^k_{1:N})\|^2\notag\\
    \le&2\|v(x^k,y^k_{1:N})-v(x^k,y_{1:N}^\diamond(x^k))\|^2\notag\\
    &+2\|v(x^k,y_{1:N}^\diamond(x^k))-v(x^\ast,y_{1:N}^\diamond(x^\ast))\|^2\notag\\
    \le&2\sum_{n=1}^N\ell_{v,n}^2\|y_n^k-y^\ast_n(x^k,y_{1:n-1}^k)\|^2+2L_v^2\|x^k-x^\ast\|^2
\end{align}
where the last inequality is from \eqref{apx-sm-eq3} and Lemma \ref{lma-aux-1}. By doing so, we obtain
\begin{align}
    &\V^{k+1}\le(1-\frac{\mu_0\alpha}{4})\V^k+\alpha\sum_{n=1}^NB^k\E\|y_n^k-y^{k,\ast}_n\|^2\notag\\
    &-(\frac{\mu_0}{4}-2L_v^2(a_0^0+b_0^k)-\frac{\mu_0}{256})\alpha\E\|x^k-x^\ast\|^2\notag\\
    &+\!\sum_{n=1}^N\sum_{j=0}^{k-1}\Bigg(a_n^{k-j}+b_{n}^{k}-\left(1-\frac{\mu_0\alpha}{4}\right)(a_n^{k-j-1}+b_{n}^{k-1})\notag\\
    &+\alpha\left(1\!-\!\frac{\mu_0\alpha}{2}\right)^{k-j}\Bigg)\alpha\E\|h_n(x^j,y^j_{1:n})\|^2\notag\\
    &+\sum_{j=0}^{k-1}\Bigg(a_0^{k-j}+b_{0}^{k}-(1\!-\!\frac{\mu_0\alpha}{4})(a_0^{k-j-1}\!+\!b_{0}^{k-1})\notag\\
    &+\frac{\mu_0\alpha}{512L_v^2}(1\!-\!\frac{\mu_0\alpha}{2})^{k-j}\Bigg)\alpha\E\|v(x^j,y^j_{1:N})\|^2+\alpha^2\sigma^2.\label{apx-sm-eq-V1}
\end{align}
Then with Lemma \ref{apx-aux-lma-sub}, we use \eqref{apx-aux-lma-result} to bound $\E\|y_n^k-y^{\ast,k}_n\|^2$ in \eqref{apx-sm-eq-V1}, and obtain :
\begin{align}
    &\V^{k+1}\le(1-\frac{\mu_0\alpha}{4})\V^k+(\sum_{n=1}^NB^kC_{n,\sigma}^k+1)\alpha^2\sigma^2\notag\\
    &+4(1-\frac{\mu_n}{2}\beta_n)^k\sum_{n=1}^N\E\|y_n^0-y^{0,\ast}_n\|^2\notag\\
    &+\!\sum_{n=1}^N\sum_{j=0}^{k-1}\Bigg(\alpha(1\!-\!\frac{\mu_0\alpha}{2})^{k-j}-(1-\frac{\mu_0\alpha}{4})(a_n^{k-j-1}+b_{n}^{k-1})\notag\\
    &+a_n^{k-j}+b_{n}^{k}+\sum_{l=1}^NB^k(C_{l,n,h}^{k-j-1}\!+\!D_{l,n,h}^{k})\Bigg)\alpha\E\|h_n(x^j,y^j_{1:n})\|^2\notag\\
    &+\sum_{j=0}^{k-1}\Bigg(\frac{\mu_0\alpha}{512L_v^2}(1\!-\!\frac{\mu_0\alpha}{2})^{k-j}-(1\!-\!\frac{\mu_0\alpha}{4})(a_0^{k-j-1}\!+\!b_{0}^{k-1})\notag\\
    &+a_0^{k-j}+b_{0}^{k}+\sum_{n=1}^NB^k(C_{n,v}^{k-j-1}\!+\!D_{n,v}^{k})\Bigg)\alpha\E\|v(x^j,y^j_{1:N})\|^2\notag\\
    &-(\frac{\mu_0}{8}-2L_v^2(a_0^0+b_0^k))\alpha\E\|x^k-x^\ast\|^2,\label{apx-thm-sm-eq8}
\end{align}
Later, we will make sure $\{b_n^k\}_k$ is a decreasing sequence, thus we have $B^k\le B^0$, and the convergence result can be obtained from \eqref{apx-thm-sm-eq8} if the followings hold for any $k\ge1,n\in[N]$:
\begin{subequations}\label{apx-sm-para}
    \begin{align}
        &b_n^k\le(1-\frac{\mu_0\alpha}{4})b_n^{k-1}-\sum_{l=1}^NB^0D_{l,n,h}^k,\label{apx-sm-para1}\\
        &a_n^k\!\le\!(1\!-\!\frac{\mu_0\alpha}{4})a_n^{k-1}\!-\!\sum_{l=1}^NB^0C_{l,n,h}^{k-1}\!-\!(1\!-\!\frac{\mu_0\alpha}{2})^{k-1}\alpha,\label{apx-sm-para2}\\
        &b_0^k\le(1-\frac{\mu_0\alpha}{4})b_0^{k-1}-\sum_{n=1}^NB^0D_{n,v}^k,\label{apx-sm-para3}\\
        &a_0^k\!\le\!(1\!-\!\frac{\mu_0\alpha}{4})a_0^{k-1}\!-\!\sum_{n=1}^NB^0C_{n,v}^{k-1}\!-\!(1\!-\!\frac{\mu_0\alpha}{2})^{k-1}\frac{\mu_0\alpha}{512L_v^2},\label{apx-sm-para4}\\
        &a_0^0+b_0^0\le\frac{\mu_0}{16L_v^2}.\label{apx-sm-para5}
    \end{align}
\end{subequations}

\subsection{Parameter $a_n^k$ and $b_n^k$ with $n\in[N]$}\label{apx-subsec-ank-bnk}
Next we select $a_n^k,b_n^k$ to make sure \eqref{apx-sm-para1} and \eqref{apx-sm-para2} hold. We begin with supposing \eqref{apx-sm-para5} hold. We let
\begin{align}
    &\hat B^0:=B+\mu_0\ell_v^2/(8L_v^2)+\max_{n\in[N]}(a_n^0+b_n^{0})\ell_h^2\ge B^0, \label{apx-sm-eq-hat-B0}\\
    &b_n^k:=\sum_{i=k+1}^\infty(1-\mu_0\alpha/4)^{k-i}\hat B^0\sum_{l=1}^ND_{l,n,h}^i,\ \forall k>0.\label{apx-sm-bnk}
\end{align}
% which ensures \eqref{apx-sm-para1} holds for $k>1$ if the RHS is finite.
According to the definition of $D_{l,n,h}^i$ in \eqref{apx-aux-lma-para}, we have $b_n^k=\hat B^0\hat b_n^k$ for $k>0$, where for any $k$,
\begin{align}
    \hat b_n^k:=\sum_{i=k+1}^\infty\frac{\sum_{l=1}^N\Tilde c_{l,3}^i\omega_1\alpha^2+\sum_{l=n+1}^{N}\Tilde c_{l,3}^i\beta_n^2(i+\omega_0)}{(1-\mu_0\alpha/4)^{i-k}}.
\end{align}
Apparently, $\{\hat b_n^k\}_k$ is a decreasing sequence. Since $\hat B^0$ in \eqref{apx-sm-bnk} is related to $b_n^0$, we can not simply let $b_n^0$ defined as \eqref{apx-sm-bnk}. Next we choose step sizes and $b_n^0$ s.t. $b_n^0\ge \hat B^0\hat b_n^0$, which together with \eqref{apx-sm-bnk} provide a well-defined $b_n^k$ and ensure \eqref{apx-sm-para1} holds under \eqref{apx-sm-para5}. First, we have
\begin{align}
    \hat b_n^0\le&\sum_{i=0}^\infty\frac{\sum_{l=1}^N\Tilde c_{l,3}^i\omega_1\alpha^2+\sum_{l=n+1}^{N}\Tilde c_{l,3}^i\beta_n^2(i+\omega_0)}{(1-\mu_0\alpha/4)^{i}}\notag\\
    =&8\Bigg(\sum_{l=1}^N\frac{\omega_1L_{y,l}^2\alpha^2(4-\mu_0\alpha)}{2\mu_l\beta_l-\mu_0\alpha}+\sum_{l=n+1}^{N}\frac{L_{y,l}^2\beta_n^2(4-\mu_0\alpha)^2}{(2\mu_l\beta_l-\mu_0\alpha)^2}\notag\\
    &+\sum_{l=n+1}^{N}\frac{(\omega_0-1)L_{y,l}^2\beta_n^2(4-\mu_0\alpha)}{2\mu_l\beta_l-\mu_0\alpha}\Bigg).\label{apx-sm-hat-bn0-ub}
\end{align}
By selecting step sizes such that for any $n$
\begin{subequations}\label{apx-sm-con-bn0}
    \begin{align}
        &\alpha\le\min_{n}\{\frac{\mu_n\beta_n}{\mu_0},\frac{\mu_0}{256\omega_1\ell_h^2\sum_{l=1}^NL_{y,l}^2}\},\\
        &\beta_n\le\min_{N\ge l>n}\{\frac{\mu_l\beta_l}{64\sqrt{N-n}L_{y,l}\ell_h},\ \frac{1/(\omega_0-1)_+}{8\sqrt{N-n}L_{y,l}\ell_h}\},
    \end{align}
\end{subequations}
from \eqref{apx-sm-hat-bn0-ub}, we have
\begin{align}
    \hat b_n^0\le{1}/{(2\ell_h^2)},\quad\forall n\in[N],\label{apx-sm-bn0-ub}
\end{align}
and let
\begin{align}
    b_n^0:=2\left(B+\frac{\mu_0\ell_v^2}{8L_v^2}+\max_{n\in[N]}\{a_n^0\}\ell_h^2\right)\max_{n}\{\hat b_n^0\},\label{apx-sm-bn0}
\end{align}
then we have
\begin{align}
    \hat B^0\hat b_n^0=&(B+\frac{\mu_0\ell_v^2}{8L_v^2}+\max_{n\in[N]}\{a_n^0+b_n^0\}\ell_h^2)\hat b_n^0\notag\\
    \le&(B+\frac{\mu_0\ell_v^2}{8L_v^2}+\max_{n\in[N]}\{a_n^0\}\ell_h^2+\max_{n\in[N]}\{b_n^0\}\ell_h^2)\hat b_n^0\notag\\
    \le&(B+\frac{\mu_0\ell_v^2}{8L_v^2}+\max_{n\in[N]}\{a_n^0\}\ell_h^2)\hat b_n^0+\frac{1}{2}\max_{n\in[N]}\{b_n^0\}\notag\\
    \le&b_n^0,
\end{align}
where the penultimate inequality is from \eqref{apx-sm-bn0-ub}, and the last inequality is from $\max_{n\in[N]}\{b_n^0\}=b_n^0$ for any $n$. Under \eqref{apx-sm-para5} and the definition of $b_n^k$ shown above, \eqref{apx-sm-para1} with $k>1$ is directly verified by substituting the definition \eqref{apx-sm-bnk} into \eqref{apx-sm-para1}. In the case with $k=1$, we have
\begin{align}
    (1-\frac{\mu_0\alpha}{4})b_n^{0}-b_n^1\ge&(1-\frac{\mu_0\alpha}{4})\hat B^0\hat b_n^0-b_n^1\notag\\
    \ge&\hat B^0\sum_{l=1}^ND_{l,n,h}^1\ge B^0\sum_{l=1}^ND_{l,n,h}^1.
\end{align}

Similarly, for any $k>0$, we let
\begin{align}\label{apx-sm-ank}
    &a_n^k:=\sum_{i=k}^\infty(1-\frac{\mu_0\alpha}{4})^{k-i-1}(\hat B^0\sum_{l=1}^NC_{l,n,h}^i+{(1-\frac{\mu_0\alpha}{2})^i}\alpha)\notag\\
    &=\sum_{i=k}^\infty(1-\frac{\mu_0\alpha}{4})^{k-i-1}\hat B^0\sum_{l=1}^NC_{l,n,h}^i+\frac{4}{\mu_0}(1-\frac{\mu_0\alpha}{2})^{k},
\end{align}
which ensures \eqref{apx-sm-para2} holds for $k>1$. According to the definition of $C_{l,n,h}^i$ in \eqref{apx-aux-lma-para}, we have $a_n^k=\hat B^0\hat a_n^k+\frac{4}{\mu_0}(1-\frac{\mu_0\alpha}{2})^{k}$, where for any $k$,
\begin{align*}
    \hat a_n^k\!:=\!\sum_{i=k}^\infty\!\frac{\sum_{l=1}^N\Tilde c_{l,1}^i\alpha^2\!+\!\sum_{l=n+1}^{N}\Tilde c_{l,0,h}^i\beta_n^2\!+\!2(1\!-\!\mu_n\beta_n)^{i}\beta_n^2\omega_0}{(1-\mu_0\alpha/4)^{i-k+1}}.
\end{align*}
Next we choose step sizes and $a_n^0$ s.t. $a_n^0\ge \hat B^0\hat a_n^0+\frac{4}{\mu_0}$ which together with \eqref{apx-sm-ank} provide a well-defined $a_n^k$ and ensure \eqref{apx-sm-para2} holds under \eqref{apx-sm-para5}. First, we have
\begin{align}
    \hat a_n^0=&
    \!\sum_{i=0}^\infty\!\frac{\sum_{l=1}^N\Tilde c_{l,1}^i\alpha^2\!+\!\sum_{l=n+1}^{N}\Tilde c_{l,0,h}^i\beta_n^2\!+\!2(1\!-\!\mu_n\beta_n)^{i}\beta_n^2\omega_0}{(1-\mu_0\alpha/4)^{i+1}}\notag\\
    =&\!\sum_{l=n+1}^N\!\frac{32\ell_l^2\beta_n^2(\beta_l\!+\!\frac{2}{\mu_l})}{\mu_l(2\mu_l\beta_l\!-\!\mu_0\alpha)}\!\left(\frac{4-\mu_0\alpha}{2\mu_l\beta_l-\mu_0\alpha}\!+\!\omega_0\!+\!\frac{2}{\mu_l\beta_l}\right)\notag\\
    &+\sum_{l=1}^N\frac{32\ell_l^2\alpha^2(\beta_l+\frac{2}{\mu_l})}{\mu_l(2\mu_l\beta_l-\mu_0\alpha)}+\frac{8\omega_0\beta_n^2}{4\mu_n\beta_n-\mu_0\alpha}.
\end{align}
By selecting step sizes such that for any $n$,
\begin{subequations}\label{apx-sm-con-an0}
    \begin{align}
        &\alpha\le\min_{n}\left\{\frac{\mu_n\beta_n}{\mu_0},\frac{\mu_0}{512\ell_h^2\sum_{l=1}^N(1+{2\ell_l^2}/{\mu_l^2})}\right\},\\
        &\beta_n\le\min_{N\ge l>n}\left\{\frac{\mu_n}{\ell_n^2},\frac{\mu_n}{32\omega_0\ell_h^2},\frac{2}{\omega_0\mu_n},\frac{\mu_l^2\beta_l/\sqrt{N-n}}{64\sqrt{\mu_l^2+2\ell_l^2}\ell_h}\right\},
    \end{align}
\end{subequations}
then under \eqref{apx-sm-con-an0}, we have
\begin{align}
    \hat a_n^0\le{1}/{(4\ell_h^2)},\quad\forall n\in[N],\label{apx-sm-an0-ub}
\end{align}
and let
\begin{align}
    a_n^0:=4(B+\mu_0\ell_{v}^2/(8L_v^2))\max_{n\in[N]}\{\hat a_n^0\}+{8}/{\mu_0},\label{apx-sm-an0}
\end{align}
then with \eqref{apx-sm-eq-hat-B0} and \eqref{apx-sm-bn0}, we have $a_n^0\ge \hat B^0\hat a_n^0+\frac{4}{\mu_0}$ since
\begin{align}
    \hat B^0\hat a_n^0=&(B+\mu_0\ell_v^2/(8L_v^2)+\max_{n\in[N]}\{a_n^0+b_n^{0}\}\ell_h^2)\hat a_n^0\notag\\
    \overset{(a)}{\le}&2(B+\frac{\mu_0\ell_v^2}{8L_v^2}+\max_{n\in[N]}\{a_n^0\}\ell_h^2)\hat a_n^0\notag\\
    \overset{(b)}\le&2(B+\frac{\mu_0\ell_v^2}{8L_v^2})\hat a_n^0+\frac{1}{2}\max_{n\in[N]}\{a_n^0\}\notag\\
    \overset{(c)}\le&\frac{a_n^0}{2}-\frac{4}{\mu_0}+\frac{1}{2}\max_{n\in[N]}\{a_n^0\}\overset{(d)}=a_n^0-\frac{4}{\mu_0},
\end{align}
where (a) is from \eqref{apx-sm-bn0-ub} and \ref{apx-sm-bn0}, (b) is from \eqref{apx-sm-an0-ub}, (c) is from \eqref{apx-sm-an0} and (d) is from $\max_{l\in[N]}\{a_l^0\}=a_n^0$ for any $n\in[N]$. Similar to the case with $b_n^k$, under \eqref{apx-sm-para5}, the definition of $a_n^k$ in \eqref{apx-sm-ank} and \eqref{apx-sm-an0} ensures \eqref{apx-sm-para2} hold.
According to \eqref{apx-sm-bn0-ub}, \eqref{apx-sm-bn0}, \eqref{apx-sm-an0-ub} and \eqref{apx-sm-an0}, we have
\begin{align}\label{apx-sm-an0-bn0}
    a_n^0+b_n^0\le\frac{3}{\ell_h^2}\left(B+\frac{\mu_0\ell_{v}^2}{8L_v^2}\right)+\frac{8}{\mu_0}\sim\cO(1).
\end{align}

In summary, if \eqref{apx-sm-para5} holds, with step sizes satisfying \eqref{apx-sm-con-bn0} and \eqref{apx-sm-con-an0}, the definitions of $a_n^k,b_n^k$ as \eqref{apx-sm-bn0-ub}, \eqref{apx-sm-bn0}, \eqref{apx-sm-an0-ub} and \eqref{apx-sm-an0} ensure that \eqref{apx-sm-para1}, \eqref{apx-sm-para2} and \eqref{apx-sm-an0-bn0} hold.

\subsection{Parameter $a_0^k$ and $b_0^k$}
Next we select $a_0^k,b_0^k$ to satisfy \eqref{apx-sm-para3}, \eqref{apx-sm-para4} and \eqref{apx-sm-para5}. First, we define
\begin{align}
    &\Tilde B^0:=4B+\frac{\mu_0\ell_{v}^2}{L_v^2}+2(a_0^0+b_0^0)\ell_{v}^2,\\
    &\Tilde\ell_v:=\max\left\{\ell_{v}, \frac{8L_v}{\sqrt{\mu_0}}\sqrt{4B+\frac{\mu_0\ell_{v}^2}{L_v^2}}\right\}.\label{apx-sm-TLv}
\end{align}

We let
\begin{align}\label{apx-sm-b0k}
    b_0^k:=\sum_{i=k+1}^\infty(1-\mu_0\alpha/4)^{k-i}\sum_{n=1}^N\Tilde B^0D_{n,v}^i,\ \forall k>0.
\end{align}
According to the definition of $D_{n,v}^i$ in \eqref{apx-aux-lma-para}, we have $b_0^k=\sum_{n=1}^N\Tilde B^0\Tilde b_n^k$ for $k>0$, where
\begin{align}
    \Tilde b_n^k:=\sum_{i=k+1}^\infty\frac{\Tilde c_{n,3}^i\alpha^2(i+\omega_1)}{(1-\mu_0\alpha/4)^{i-k}},\quad\forall k\ge 0.
\end{align}
Next, we choose step sizes and $b_0^0$ s.t. $b_0^0\ge\sum_{n=1}^N\Tilde B^0\Tilde b_n^0$. First,
\begin{align}
    \Tilde b_n^0\le&\sum_{i=0}^\infty\frac{\Tilde c_{n,3}^i\alpha^2(i+\omega_1)}{(1-\mu_0\alpha/4)^{i}}\notag\\
    =&\frac{8L_{y,n}^2\alpha^2(4-\mu_0\alpha)}{2\mu_n\beta_n-\mu_0\alpha}\left(\frac{4-\mu_0\alpha}{2\mu_n\beta_n-\mu_0\alpha}-1+\omega_1\right).\label{apx-sm-b0-aux1}
\end{align}
By selecting step sizes such that for any $n$
\begin{align}\label{apx-sm-con-b00}
    \alpha\le\min_{n}\left\{\frac{\mu_n\beta_n}{\mu_0},\frac{\mu_n\beta_n/64}{\sqrt{N}\Tilde\ell_vL_{y,n}},\frac{\mu_0/256}{N\Tilde\ell_v^2L_{y,n}^2(\omega_1-1)}\right\},
\end{align}
then according to \eqref{apx-sm-b0-aux1}, we have
\begin{align}\label{apx-sm-b00-ub}
    \Tilde b_n^0\le\frac{1}{4N\Tilde\ell_v^2}=\min\left\{\frac{1}{4N\ell_{v}^2},\frac{\mu_0/(256N)}{4BL_v^2+\mu_0\ell_{v}^2}\right\},
\end{align}
and let
\begin{align}
    b_0^0:=\sum_{n=1}^N2\left(4B+\frac{\mu_0\ell_{v}^2}{L_v^2}+2a_0^0\ell_{v}^2\right)\Tilde b_n^0,\label{apx-sm-b00}
\end{align}
thus it holds that
\begin{subequations}\label{apx-sm-b00-bound}
    \begin{align}
        \sum_{n=1}^N\Tilde B^0\Tilde b_n^0\overset{(a)}=&\sum_{n=1}^N(4B+\frac{\mu_0\ell_{v}^2}{L_v^2}+2(a_0^0+b_0^0)\ell_{v}^2)\Tilde b_n^0\notag\\
        \overset{(b)}\le&\sum_{n=1}^N(4B+\frac{\mu_0\ell_{v}^2}{L_v^2}+2a_0^0\ell_{v}^2)\Tilde b_n^0+\frac{b_0^0}{2}\notag\\
        =&b_0^0,\label{apx-sm-b00-bound1}\\
        b_0^0\overset{(c)}\le&\frac{2B}{\ell_{v}^2}+\frac{\mu_0}{2L_v^2}+a_0^0,\label{apx-sm-b00-bound2}
    \end{align}
\end{subequations}
where (a) is from \eqref{apx-sm-TLv}, (b) is from \eqref{apx-sm-b00-ub} and (c) is from \eqref{apx-sm-b00-ub} and \eqref{apx-sm-b00}.

Similarly, for any $k>0$, we let
\begin{align}\label{apx-sm-a0k}
    a_0^k&:=\sum_{i=k}^\infty(1-\frac{\mu_0\alpha}{4})^{k\!-\!i\!-\!1}(\Tilde B^0\sum_{n=1}^NC_{n,v}^i+{(1-\frac{\mu_0\alpha}{2})^i}\frac{\mu_0\alpha}{512L_v^2})\notag\\
    &=\sum_{n=1}^N\!\Tilde B^0\sum_{i=k}^\infty(1\!-\!\frac{\mu_0\alpha}{4})^{k\!-\!i\!-\!1}C_{n,v}^i+\frac{\mu_0(1\!-\!\frac{\mu_0\alpha}{2})^{k}}{128L_v^2}.
\end{align}
Then $a_0^k:=\sum_{n=1}^N\Tilde B^0\Tilde a_n^k+\frac{\mu_0(1-\frac{\mu_0\alpha}{2})^{k}}{128L_v^2}$ for $k>0$, where
\begin{align}
    \Tilde a_n^k:=\sum_{i=k}^\infty\frac{\Tilde c_{n,0,v}^i\alpha^2}{(1-\mu_0\alpha/4)^{i-k+1}},\quad\forall k\ge0.
\end{align}
Next we select step sizes and $a_0^0$ s.t. $a_0^0\ge\sum_{n=1}^N\Tilde B^0\Tilde a_n^0+\frac{\mu_0}{128L_v^2}$, where
\begin{align}
    &\Tilde a_n^0=\sum_{i=0}^\infty\frac{\Tilde c_{0}^i\alpha^2}{(1-\mu_0\alpha/4)^{i+1}}\notag\\
    &=\frac{32\ell_n^2\alpha^2(\beta_n+\frac{2}{\mu_n})}{\mu_n(2\mu_n\beta_n-\mu_0\alpha)}\left(\frac{4-\mu_0\alpha}{2\mu_n\beta_n-\mu_0\alpha}\!+\!\omega_1\!+\!\frac{2}{\mu_n\beta_n}\right).
\end{align}
By selecting step sizes satisfying \eqref{apx-sm-con-b00} and
% \begin{subequations}
    \begin{align}\label{apx-sm-con-a00}
        &\alpha\le\min_{n}\left\{\frac{\mu_n\beta_n}{\mu_0},\frac{\mu_n^2\beta_n}{32\sqrt{2N(\mu_n^2+2\ell_n^2)}\Tilde\ell_v}\right\},\beta_n\le\frac{2}{\mu_n\omega_1},
    \end{align}
% \end{subequations}
then for $n\in[N]$, it holds that
\begin{align}\label{apx-sm-a00-ub}
    \Tilde a_n^0\le\frac{1}{8N\Tilde\ell_v^2}=\min\left\{\frac{1}{8N\ell_{v}^2},\frac{\mu_0/(512N)}{4BL_v^2+\mu_0\ell_{v}^2}\right\},
\end{align}
and let
\begin{align}
    a_0^0:=\sum_{n=1}^N4\Tilde a_n^0\left(4B+\frac{\mu_0\ell_{v}^2}{L_v^2}\right)+\frac{\mu_0}{64L_v^2}.\label{apx-sm-a00}
\end{align}
Then we have
\begin{align}
    \sum_{n=1}^N\Tilde B^0\Tilde a_n^0=&\sum_{n=1}^N\left(4B+\frac{\mu_0\ell_{v,n}^2}{L_v^2}+2(a_0^0+b_0^0)\ell_{v}^2\right)\Tilde a_n^0\notag\\
    \overset{(a)}{\le}&\sum_{n=1}^N\left(8B+\frac{2\mu_0\ell_{v}^2}{L_v^2}+4\ell_{v}^2a_0^0\right)\Tilde a_n^0\notag\\
    \overset{(b)}{\le}&\frac{a_0^0}{2}+\sum_{n=1}^N2\Tilde a_n^0\left(4B+\frac{\mu_0\ell_{v}^2}{L_v^2}\right)\notag\\
    =&a_0^0-\frac{\mu_0}{128L_v^2},
\end{align}
where (a) is from \eqref{apx-sm-b00-bound} and (b) is from \eqref{apx-sm-a00-ub}.

We also have
\begin{align}
    a_0^0+b_0^0=&a_0^0+\sum_{n=1}^N2\Tilde b_n^0\left(4B+\frac{\mu_0\ell_{v}^2}{L_v^2}+2a_0^0\ell_{v}^2\right)\notag\\
    \overset{(a)}{\le}&2a_0^0+\sum_{n=1}^N2\Tilde b_n^0\left(4B+\frac{\mu_0\ell_{v}^2}{L_v^2}\right)\notag\\
    =&\sum_{n=1}^N2(4\Tilde a_n^0+\Tilde b_n^0)\left(4B+\frac{\mu_0\ell_{v}^2}{L_v^2}\right)+\frac{\mu_0}{32L_v^2}\notag\\
    \overset{(b)}\le&\sum_{n=1}^N\frac{3}{2N\Tilde\ell_v^2}\left(4B+\frac{\mu_0\ell_{v}^2}{L_v^2}\right)+\frac{\mu_0}{32\mu_0}\notag\\
    \overset{(c)}\le&\frac{3\mu_0}{128L_v^2}+\frac{\mu_0}{32L_v^2}<\frac{\mu_0}{16L_v^2},
\end{align}
where (a) is from \eqref{apx-sm-b00-ub}, (b) is from \eqref{apx-sm-b00-ub} and \eqref{apx-sm-a00-ub}, and (c) is from \eqref{apx-sm-TLv}.
Therefore, by choosing step sizes such that \eqref{apx-sm-con-b00} and \eqref{apx-sm-con-a00}, \eqref{apx-sm-para5} holds. Then, by the definitions of $a_n^k,b_n^k$ in subsection \ref{apx-subsec-ank-bnk}, we have \eqref{apx-sm-para1}, \eqref{apx-sm-para2} and \eqref{apx-sm-an0-bn0} hold. According to \eqref{apx-sm-an0-bn0}, we have
\begin{align}
    B^0\le\Tilde{B}^0\sim\cO(1),\label{apx-sm-Bn0}
\end{align}
which together with \eqref{apx-sm-con-b00}, \eqref{apx-sm-con-a00} and the definitions of $a_0^k,b_0^k$ imply that for any $k>0$, it holds that
\begin{align}
    &(1-\frac{\mu_0\alpha}{4})b_0^{k-1}-b_0^{k}\ge\sum_{n=1}^N\Tilde{B}^0D_{n,v}^k\ge\sum_{n=1}^N{B}^0D_{n,v}^k,\\
    &(1-\frac{\mu_0\alpha}{4})a_0^{k-1}-a_0^{k}\ge\sum_{n=1}^N\Tilde{B}^0C_{n,v}^{k-1}+\frac{\mu_0\alpha}{512L_v^2}(1-\frac{\mu_0\alpha}{2})^{k-1}\notag\\
    &\ge\sum_{n=1}^N{B}^0C_{n,v}^{k-1}+\frac{\mu_0\alpha}{512L_v^2}(1-\frac{\mu_0\alpha}{2})^{k-1},
\end{align}
which implies \eqref{apx-sm-para3} and \eqref{apx-sm-para4} hold.

According to \eqref{apx-thm-sm-eq8} and \eqref{apx-sm-para}, we have
\begin{align}
    \V^{k+1}
    \le&(1-\frac{\mu_0\alpha}{4})\V^k+\left(\sum_{n=1}^NB^kC_{n,\sigma}^k+1\right)\alpha^2\sigma^2\notag\\
    &+4\alpha\sum_{n=1}^NB^0(1-\frac{\mu_n\beta_n}{2})^k\E\|y_n^0-y_n^{0,\ast}\|^2.
\end{align}
Since \eqref{apx-sm-Bn0}, we have $B\le B^k\le B^0\sim\cO(1)$, and according to the definition of $C_{n,\sigma}^k$ \eqref{apx-sm-Cnsigma}, we have
\begin{align}
    C_{n,\sigma}^k\le\cO\left(\frac{\beta_n}{\alpha}+k\left(1-\frac{\mu_n\beta_n}{2}\right)^k\frac{\beta_n^2}{\alpha}\right).\label{apx-Cnsigma}
\end{align}
We can select step sizes $\alpha=\frac{4}{\mu_0K}\log\Theta(K)$, $\beta_n=\Theta(\frac{\log K}{K})$ satisfying constraints $\alpha\le\frac{\mu_0}{\ell_0^2}$, $\beta_n\le\frac{\mu_n}{\ell_n^2}$, \eqref{apx-sm-con-bn0}, \eqref{apx-sm-con-an0}, \eqref{apx-sm-con-b00} and \eqref{apx-sm-con-a00}. Thus, it holds that
\begin{align}
    \V^K\le&\cO\Bigg(\left(1-\frac{\mu_0\alpha}{4}\right)^K+\sum_{n=1}^N\frac{4\alpha\beta_n}{\mu_0\alpha}\notag\\
    &+\alpha\sum_{n=1}^N\sum_{k=0}^{K-1}(1-\frac{\mu_n\beta_n}{2})^k(1-\frac{\mu_0\alpha}{4})^{K-1-k}\notag\\
    &+{\alpha}{\beta_n^2}\sum_{k=0}^{K-1}k\left(1-\frac{\mu_n\beta_n}{2}\right)^k\Bigg)\notag\\
    \le&\cO\left(e^{-\frac{K\mu_0\alpha}{4}}+\sum_{n=1}^N{\beta_n}+K\alpha(1-\frac{\mu_0\alpha}{4})^K+\alpha\right)\notag\\
    \le&\Tilde{\cO}\left(\frac{1}{K}\right).
\end{align}
Then we have $\E\|x^K-x^\ast\|^2\le\V^K\le\Tilde{\cO}\left(K^{-1}\right)$ and
\begin{subequations}\label{apx-sm-sub}
    \begin{align}
        &\sum_{j=0}^{K-1}\left(1-\frac{\mu_0\alpha}{2}\right)^{K-j-1}{((K-j)\alpha^2+\alpha)}\E\|v(x^j,y^j_{1:N})\|^2\notag\\
        &\le\Tilde{\cO}\left(K^{-1}\right),\label{apx-sm-sub1}\\
        &\sum_{n=1}^N\sum_{j=0}^{K-1}\left(1-\frac{\mu_0\alpha}{2}\right)^{K-j-1}((K-j)\alpha^2+\alpha)\E\|h_n(x^j,y^j_{1:n})\|^2\notag\\
        &\le\Tilde{\cO}\left(K^{-1}\right).\label{apx-sm-sub2}
    \end{align}
\end{subequations}
Next we derive the convergence of auxiliary sequences.
According to the definition of $C_{n,l,h}^k,D_{n,l,h}^k,C_{n,v}^k,D_{n,v}^k$, $\mu_0\alpha\le\mu_n\beta_n$ and $\alpha=\Theta(\beta_n)$, we have
\begin{align*}
    &\sum_{j=0}^{K-1}(C_{n,v}^{K-j-1}+D_{n,v}^{K})\E\|v(x^j,y_{1:N}^j)\|^2\le\notag\\
    &\cO(\sum_{j=0}^{K-1}\left(1-\frac{\mu_0\alpha}{2}\right)^{K-j-1}{((K-j)\alpha^2+\alpha)}\E\|v(x^j,y^j_{1:N})\|^2),\\
    &\sum_{l=1}^{N}\sum_{j=0}^{K-1}(C_{n,l,h}^{K-j-1}+D_{n,l,h}^{K})\E\|h_l(x^j,y_{1:l}^j)\|^2\le\notag\\
    &\cO(\sum_{n=1}^N\!\sum_{j=0}^{K-1}\!\left(1-\frac{\mu_0\alpha}{2}\right)^{K\!-\!j\!-\!1}\!((K\!-\!j)\alpha^2\!+\!\alpha)\E\|h_n(x^j,y^j_{1:n})\|^2).
\end{align*}
Thus, from \eqref{conv-sub-finally}, we have
\begin{align}
    &\E\|y^{K}_n-y^{K,\ast}_n\|^2\notag\\
    \le&4(1-\frac{\mu_n}{2}\beta_n)^K\E\|y_n^0-y^{0,\ast}_n\|^2+C_{n,\sigma}^{K}\alpha\sigma^2+\Tilde{\cO}\left(K^{-1}\right),\notag\\
    \le&\Tilde{\cO}\left(K^{-1}\right).
\end{align}

\section{Proof of Lemma \ref{lma-nsm} and Theorem \ref{thm-MSSA-nsm}}\label{apx-nsm}
With Assumption \ref{asp:main-nsm} and Lemma \ref{lma-aux-1}, we have
\begin{align}
    \E[\Phi(x^{k+1})]\!\le&\E[\Phi(x^k)]+\E\langle v(x^k,y^{\diamond}_{1:N}(x^k)),\,x^{k+1}-x^k\rangle\notag\\
    +&\frac{L_v}{2}\E\|x^{k+1}-x^k\|^2\notag\\
    \le&\E[\Phi(x^k)]-\alpha\E\langle v(x^k,y^{\diamond}_{1:N}(x^k)),\,v(x^k,y^k_{1:N})\rangle\notag\\
    +&\frac{L_v}{2}\alpha^2\E\|v(x^k,y^k_{1:N})\|^2\!+\!\frac{L_v}{2}\alpha^2\E\|\xi^k\|^2.
\end{align}
According to Lemma \ref{lma-aux-1}, we have
\begin{align}
    &-\langle v(x^k,y^{\diamond}_{1:N}(x^k)),\,v(x^k,y^k_{1:N})\rangle\notag\\
    \le&\frac{1}{2}\|v(x^k,y^k_{1:N})-v(x^k,y^{\diamond}_{1:N}(x^k))\|^2-\frac{1}{2}\|v(x^k,y^k_{1:N})\|^2\notag\\
    &-\frac{1}{2}\|v(x^k,y^{\diamond}_{1:N}(x^k))\|^2\notag\\
    \le&\sum_{n=1}^N\frac{\ell_{v,n}^2}{2}\|y_n^k\!-\!y^{\ast,k}_n\|-\frac{1}{2}\|v(x^k,y^k_{1:N})\|^2\notag\\
    &-\frac{1}{2}\|v(x^k,y^{\diamond}_{1:N}(x^k))\|^2.
\end{align}
With step size $\alpha\le\frac{1}{2L_v}$, Lemma \ref{lma-nsm} can be proved by combining the above two inequalities. Next, we prove Theorem \ref{thm-MSSA-nsm}.
Under Assumption \ref{asp-noise}, we have
\begin{align}
    &\E\|\xi^k\|^2\le\omega_1\E\|v(x^k,y^k_{1:N})\|^2\!+\!\omega_2\sum_{n=1}^N\E\|h_n(x^k,y_{1:n}^k)\|^2\!+\!\sigma^2\notag\\
    &\le\omega_1\E\|v(x^k,y^k_{1:N})\|^2\!+\!\omega_2\sum_{n=1}^N\ell_n^2\E\|y_n^k\!-\!y_n^{\ast,k}\|^2\!+\!\sigma^2.
\end{align}
We define $\ell_h:=\max_n\{\ell_n\}$ and
\begin{align}
    &B:={\ell_{v}^2}+\frac{L_v}{2}\omega_2\ell_h^2\alpha.
\end{align}
Summing up the above results, we have
\begin{align}
    \E[\Phi(x^{k+1})]\le&\E[\Phi(x^k)]\!+\!\alpha\sum_{n=1}^N(B-\frac{\ell_{v}^2}{2})\E\|y_n^k-y_n^{\ast,k}\|^2\notag\\
    &-\frac{\alpha}{2}(1-L_v(\omega_1+1)\alpha)\E\|v(x^k,y^k_{1:N})\|^2\notag\\
    &-\frac{\alpha}{2}\E\|v(x^k,y^{\diamond}_{1:N}(x^k))\|^2+\frac{L_v}{2}\alpha^2\sigma^2.\label{apx-nsm-eq-1}
\end{align}

Define a Lyapunov function as
\begin{align}
    \V^k:=&\E[\Phi(x^k)]+\sum_{j=0}^{k-1}(a_{0}^{k-j-1}+b_0^{k-1})\alpha\E\|v(x^j,y^j_{1:N})\|^2\notag\\
    &+\sum_{n=1}^N\sum_{j=0}^{k-1}(a_{n}^{k-j-1}+b_n^{k-1})\alpha\E\|h_n(x^j,y^j_{1:n})\|^2,
\end{align}
where we guarantee $\{b_n^k\}_k$ to be a decreasing sequence. For simplicity, we define that
\begin{align}
    B^k:=B+\max_{n\in[N]}\{a_n^0+b_n^{k}\}\ell_h^2.
\end{align}
With \eqref{apx-nsm-eq-1} and the definition of $\V^k$, we have
\begin{align*}
    &\V^{k+1}\le\V^k\!+\!\alpha\sum_{n=1}^N(B-\frac{\ell_{v}^2}{2})\E\|y_n^k-y_n^{\ast,k}\|^2+\frac{L_v}{2}\alpha^2\sigma^2\notag\\
    &-\frac{\alpha}{2}(1-L_v(\omega_1+1)\alpha-2a_0^0-2b_0^k)\E\|v(x^k,y^k_{1:N})\|^2\notag\\
    &\!-\!\frac{\alpha}{2}\E\|v(x^k,y^{\diamond}_{1:N}(x^k))\|^2\!+\!\sum_{n=1}^N(a_n^0+b_n^{k})\alpha\E\|h_n(x^k,y^k_{1:n})\|^2\notag\\
    &+\sum_{n=1}^N\sum_{j=0}^{k-1}(a_n^{k-j}+b_{n}^{k}-a_n^{k-j-1}-b_{n}^{k-1})\alpha\E\|h_n(x^j,y^j_{1:n})\|^2\notag\\
    &+\sum_{j=0}^{k-1}(a_0^{k-j}+b_{0}^{k}-a_0^{k-j-1}-b_{0}^{k-1})\alpha\E\|v(x^j,y^j_{1:N})\|^2.
\end{align*}
Next we use \eqref{apx-sm-eq4} and \eqref{apx-aux-lma-result} in Lemma \ref{apx-aux-lma-sub} to bound $\|h_n(x^k,y^k_{1:n})\|^2$ in the above inequality, and then obtain
\begin{align}
    &\V^{k+1}\le\V^k\!-\!\frac{\alpha\ell_{v}^2}{2}\sum_{n=1}^N\E\|y_n^k-y_n^{\ast,k}\|^2\!-\!\frac{\alpha}{2}\E\|v(x^k,y^{\diamond}_{1:N}(x^k))\|^2\notag\\
    &-\frac{\alpha}{2}(1-L_v(\omega_1+1)\alpha-2a_0^0-2b_0^k)\E\|v(x^k,y^k_{1:N})\|^2\notag\\
    &+\sum_{n=1}^N\sum_{j=0}^{k-1}(a_n^{k-j}+b_{n}^{k}-a_n^{k-j-1}-b_{n}^{k-1}\!+\!\sum_{l=1}^NB^k(C_{l,n,h}^{k-j-1}\notag\\
    &+D_{l,n,h}^{k}))\alpha\E\|h_n(x^j,y^j_{1:n})\|^2+\sum_{j=0}^{k-1}(a_0^{k-j}+b_{0}^{k}-a_0^{k-j-1}\notag\\
    &-b_{0}^{k-1}+\sum_{n=1}^NB^k(C_{n,v}^{k-j-1}+D_{n,v}^{k}))\alpha\E\|v(x^j,y^j_{1:N})\|^2\notag\\
    &+4\alpha\sum_{n=1}^NB(1-\frac{\mu_n\beta_n}{2})^k\E\|y_n^0-y_n^{0,\ast}\|^2\notag\\
    &+(\frac{L_v}{2}+\sum_{n=1}^NB^kC_{n,\sigma}^k)\alpha^2\sigma^2.\label{apx-thm-nsm-eq1}
\end{align}
Since we will ensure $\{b_n^k\}_k$ to be decreasing, we have $B^k\le B^0$, thus the convergence result can be obtained from \eqref{apx-thm-nsm-eq1} if the followings hold for any $k\ge1,n\in[N]$:
\begin{subequations}\label{apx-nsm-para}
    \begin{align}
        &b_n^k\le b_n^{k-1}-\sum_{l=1}^NB^0D_{l,n,h}^k,\label{apx-nsm-para1}\\
        &a_n^k\le a_n^{k-1}-\sum_{l=1}^NB^0C_{l,n,h}^{k-1},\label{apx-nsm-para2}\\
        &b_0^k\le b_0^{k-1}-\sum_{n=1}^NB^0D_{n,v}^k,\label{apx-nsm-para3}\\
        &a_0^k\le a_0^{k-1}-\sum_{n=1}^NB^0C_{n,v}^{k-1},\label{apx-nsm-para4}\\
        &a_0^0+b_0^0\le\frac{1}{4}.\label{apx-nsm-para5}
    \end{align}
\end{subequations}
\subsection{Parameter $a_n^k$ and $b_n^k$ with $n\in[N]$}\label{apx-nsm-subsec-ank-bnk}
First we select $a_n^k,b_n^k$ to make sure \eqref{apx-nsm-para1} and \eqref{apx-nsm-para2} hold.
We let
\begin{align}\label{apx-nsm-bnk}
    b_n^k:=B^0\sum_{i=k+1}^\infty\sum_{l=1}^ND_{l,n,h}^i,\ \forall k>0,
\end{align}
which ensures \eqref{apx-nsm-para1} holds for $k>1$ and implies $b_n^k=B^0\hat b_n^k$ for $k>0$, where for $k\ge0$, $b_n^k$ is defined as
\begin{align}
    \hat b_n^k:=\sum_{i=k+1}^\infty\left({\sum_{l=1}^N\Tilde c_{l,3}^i\omega_1\alpha^2+\sum_{l=n+1}^{N}\Tilde c_{l,3}^i\beta_n^2(i+\omega_0)}\right).
\end{align}
However, since $B^0$ in \eqref{apx-nsm-bnk} is related to $b_n^0$, we can not simply define $b_n^0$ as \eqref{apx-nsm-bnk}. Next, we choose step sizes and $b_n^0$ s.t. $b_n^0\ge B^0\hat b_n^0$, which together with \eqref{apx-nsm-bnk} provide a well-defined $b_n^k$ and ensure \eqref{apx-nsm-para1}. With the definition of $D_{l,n,h}^i$ in \eqref{apx-aux-lma-para}, we have
\begin{align}
    \hat b_n^0\le&\sum_{i=0}^\infty\left({\sum_{l=1}^N\Tilde c_{l,3}^i\omega_1\alpha^2+\sum_{l=n+1}^{N}\Tilde c_{l,3}^i\beta_n^2(i+\omega_0)}\right)\notag\\
    =&\sum_{l=1}^N\frac{16\omega_1L_{y,l}^2\alpha^2}{\mu_l\beta_l}+\sum_{l=n+1}^{N}\frac{32L_{y,l}^2\beta_n^2}{\mu_l^2\beta_l^2}\notag\\
    &+\sum_{l=n+1}^{N}\frac{16(\omega_0-1)L_{y,l}^2\beta_n^2}{\mu_l\beta_l}.
\end{align}
By selecting step sizes such that for any $n$
\begin{subequations}\label{apx-nsm-con-bn0}
    \begin{align}
        &\alpha\le\min_{n}\{\frac{\mu_n\beta_n}{L_{y,n}},\frac{1}{128\omega_1\ell_n^2\sum_{l=1}^NL_{y,l}}\},\\
        &\beta_n\le\min_{N\ge l>n}\{\frac{\mu_l\beta_l}{16\sqrt{N-n}L_{y,l}\ell_n},\ \frac{2}{\mu_n(\omega_0-1)^+}\},
    \end{align}
\end{subequations}
then we have
\begin{align}
    \hat b_n^0\le{1}/{(2\ell_n^2)},\label{apx-nsm-bn0-ub}
\end{align}
and let
\begin{align}
    b_n^0:=2\left(B+\max_{n\in[N]}\{a_n^0\}\ell_n^2\right)\max_{n\in[N]}\{\hat b_n^0\},\label{apx-nsm-bn0}
\end{align}
then we have
\begin{align}
    B^0\hat b_n^0=&(B+\max_{n\in[N]}\{a_n^0+b_n^0\}\ell_h^2)\hat b_n^0\notag\\
    \overset{(a)}\le&(B+\max_{n\in[N]}\{a_n^0\}\ell_h^2)\hat b_n^0+\max_{n\in[N]}\{b_n^0\}/2\notag\\
    \overset{(b)}=&b_n^0,
\end{align}
where (a) is from \eqref{apx-nsm-bn0-ub} and $\max_{n\in[N]}\{a_n^0+b_n^0\}\le\max_{n\in[N]}\{a_n^0\}+\max_{n\in[N]}\{b_n^0\}$, (b) is from \eqref{apx-nsm-bn0} and $\max_{l\in[N]}\{b_l^0\}=b_n^0$ for any $n\in[N]$. Therefore, by selecting step sizes and $b_n^k$ as we showed above, we have \eqref{apx-nsm-para1} holds for any $n\in[N]$ and $k\ge1$.

Similarly, we let
\begin{align}\label{apx-nsm-ank}
    a_n^k:=B^0\sum_{i=k}^\infty\sum_{l=1}^NC_{l,n,h}^i,\ \forall k>0,
\end{align}
which ensures \eqref{apx-nsm-para2} holds for $k>1$ and implies $a_n^k=B^0\hat a_n^k$, where
\begin{align*}
    \hat a_n^k\!:=\!\sum_{i=k}^\infty\!({\sum_{l=1}^N\Tilde c_{l,1}^i\alpha^2\!+\!\sum_{l=n+1}^{N}\Tilde c_{l,0}^i\beta_n^2\!+\!2(1\!-\!\mu_n\beta_n)^{i}\beta_n^2\omega_0}).
\end{align*}
Next, we choose step sizes and $a_n^0$ s.t. $a_n^0\ge B^0\hat a_n^0$ which together with \eqref{apx-nsm-ank} provide a well-defined $a_n^k$ and ensure \eqref{apx-nsm-para2} holds. First, we have
\begin{align}
    \hat a_n^0=&\!\sum_{i=0}^\infty\!({\sum_{l=1}^N\Tilde c_{l,1}^i\alpha^2\!+\!\sum_{l=n+1}^{N}\Tilde c_{l,0}^i\beta_n^2\!+\!2(1\!-\!\mu_n\beta_n)^{i}\beta_n^2\omega_0})\notag\\
    =&\sum_{l=1}^N\frac{16\ell_l^2\alpha^2(\beta_l+\frac{2}{\mu_l})}{\mu_l^2\beta_l}+\frac{2\omega_0\beta_n}{\mu_n}\notag\\
    &+\!\sum_{l=n+1}^N\!\frac{16\ell_l^2\beta_n^2(\beta_l+\frac{2}{\mu_l})}{\mu_l^2\beta_l}\left(\omega_0\!+\!\frac{4}{\mu_l\beta_l}\right).
\end{align}
By selecting step sizes such that for any $n$,
\begin{subequations}\label{apx-nsm-con-an0}
    \begin{align}
        &\alpha\le\min_{n}\left\{\frac{\mu_n\beta_n}{\ell_h},\frac{1}{256\ell_h^2\sum_{l=1}^N(1/\ell_l+{2\ell_l}/{\mu_l^2})}\right\},\\
        &\beta_n\le\min_{N\ge l>n}\left\{\frac{\mu_n}{32\omega_0\ell_h^2},\frac{\mu_l^2\beta_l}{32\sqrt{(N-n)(\mu_l^2+2\ell_l^2)}\ell_h},\frac{4}{\omega_0\mu_l}\right\},
    \end{align}
\end{subequations}
then under \eqref{apx-nsm-con-bn0} and \eqref{apx-nsm-con-an0}, we have
\begin{align}
    \hat a_n^0\le{1}/{(4\ell_h^2)},\label{apx-nsm-an0-ub}
\end{align}
and let
\begin{align}
    a_n^0:=4B\hat a_n^0,\label{apx-nsm-an0}
\end{align}
then we have
\begin{align}
    B^0\hat a_n^0=&\ (B+\max_{n\in[N]}\{a_n^0+b_n^0\}\ell_h^2)\hat a_n^0\notag\\
    \le&\ 2(B+\max_{n\in[N]}\{a_n^0\}\ell_h^2)\hat a_n^0\notag\\
    \le&\ 2B\hat a_n^0+\frac{1}{2}\max_{n\in[N]}\{a_n^0\}=\,a_n^0.
\end{align}
Thus we have \eqref{apx-nsm-para2} holds for any $n\in[N]$ and $k\ge1$.
According to \eqref{apx-nsm-bn0-ub}, \eqref{apx-nsm-bn0}, \eqref{apx-nsm-an0-ub} and \eqref{apx-nsm-an0}, we have
\begin{align}\label{apx-nsm-an0-bn0}
    a_n^0+b_n^0\le\frac{3}{\ell_h^2}B\sim\cO(1).
\end{align}

\subsection{Parameter $a_0^k$ and $b_0^k$}
Next we select $a_0^k,b_0^k$ to satisfy \eqref{apx-nsm-para3}, \eqref{apx-nsm-para4} and \eqref{apx-nsm-para5}. From \eqref{apx-nsm-an0-bn0}, we have
\begin{align}
    &4B\ge B^0\ge B^k,\ \forall n\in[N],\forall k\ge0.
\end{align}
We let
\begin{align}\label{apx-nsm-b0k}
    b_0^k:=\sum_{i=k+1}^\infty\sum_{n=1}^N4 BD_{n,v}^i,\ \forall k\ge0,
\end{align}
which ensures \eqref{apx-nsm-para3} holds for any $k\ge1$ and implies $b_0^k=\sum_{n=1}^N4B\Tilde b_n^k$, where
\begin{align}
    \Tilde b_n^k=\sum_{i=k+1}^\infty{\Tilde c_{n,3}^i\alpha^2(i+\omega_1)}.
\end{align}
Next, we choose step sizes s.t. $b_0^0\le\frac{1}{8}$. First, we have
\begin{align}
    \Tilde b_n^0\le&\sum_{i=0}^\infty{\Tilde c_{n,3}^i\alpha^2(i+\omega_1)}\!=\!\frac{16L_{y,n}^2\alpha^2}{\mu_n\beta_n}\left(\frac{2}{\mu_n\beta_n}\!-\!1\!+\!\omega_1\right).
\end{align}
By selecting step sizes such that for any $n$
\begin{align}\label{apx-nsm-con-b00}
    \alpha\le\min_{n}\left\{\frac{\mu_n\beta_n}{32\sqrt{2NB}L_{y,n}},\frac{1}{16\sqrt{2BN}L_{y,n}(\omega_1-1)^+}\right\},
\end{align}
then we have $\Tilde b_n^0\le\frac{1}{32NB}$ which implies
\begin{align}\label{apx-nsm-b00-ub}
    b_0^0=\sum_{n=1}^N4B\Tilde b_n^0\le\sum_{n=1}^N4B\frac{1}{32NB}=\frac{1}{8}.
\end{align}

Similarly, we let
\begin{align}\label{apx-nsm-a0k}
    a_0^k:=\sum_{i=k}^\infty\sum_{n=1}^N4 BC_{n,v}^i,\ \forall k\ge0,
\end{align}
which ensures \eqref{apx-nsm-para3} holds and $a_0^k=\sum_{n=1}^N4B\Tilde a_n^k$, where
\begin{align}
    \Tilde a_n^k=\sum_{i=k}^\infty\Tilde c_{0}^i\alpha^2.
\end{align}
Next, we choose step sizes s.t. $a_0^0\le\frac{1}{8}$. First, we have
\begin{align}
    \Tilde a_n^0=&\sum_{i=0}^\infty{\Tilde c_{0}^i\alpha^2}\!=\!\frac{16\ell_n^2\alpha^2(\beta_n+\frac{2}{\mu_n})}{\mu_n^2\beta_n}\left(\omega_1\!+\!\frac{4}{\mu_n\beta_n}\right).
\end{align}
By selecting step sizes such that for any $n$
\begin{align}\label{apx-nsm-con-a00}
    \alpha\le\min_{n}\left\{\frac{\mu_n^2\beta_n}{64\sqrt{NB(\mu_n^2+2\ell_h^2)}},\frac{\mu_n}{16\sqrt{NB(\mu_n^2+2\ell_h^2)}\omega_1}\right\},
\end{align}
then with $\beta_n\le\frac{\mu_n}{\ell_n^2}$, we have
\begin{align}\label{apx-nsm-a00-ub}
    a_0^0=\sum_{n=1}^N4B\Tilde a_n^0\le\sum_{n=1}^N4B\frac{1}{32B}=\frac{1}{8}.
\end{align}
According to \eqref{apx-nsm-b00-ub} and \eqref{apx-nsm-a00-ub}, \eqref{apx-nsm-para5} holds.

Selecting parameters $a_n^k,b_n^k,a_0^k,b_0^k$ as we discuss above, we ensure \eqref{apx-nsm-para} hold. Then from \eqref{apx-thm-nsm-eq1} and \eqref{apx-nsm-para}, we have
\begin{align}
    &\V^{k+1}\le\V^k-\frac{\alpha}{2}(\frac{1}{2}-L_v(\omega_1+1)\alpha)\E\|v(x^k,y^k_{1:N})\|^2\notag\\
    &-\!\frac{\alpha}{2}\sum_{n=1}^N\ell_{v,n}\E\|y_n^k-y_n^{\ast,k}\|^2\!-\!\frac{\alpha}{2}\E\|v(x^k,y^{\diamond}_{1:N}(x^k))\|^2\notag\\
    &+4\alpha\sum_{n=1}^NB(1-\frac{\mu_n\beta_n}{2})^k\E\|y_n^0-y_n^{0,\ast}\|^2\notag\\
    &+(\frac{L_v}{2}+\sum_{n=1}^NB^kC_{n,\sigma}^k)\alpha^2\sigma^2.
\end{align}
Selecting $\alpha\le\frac{1}{2L_v(\omega_1+1)}$, we have
\begin{align}
    &\frac{1}{K}\sum_{k=0}^{K-1}\left(\E\|v(x^k,y^{\diamond}_{1:N}(x^k))\|^2+\sum_{n=1}^N\ell_{v}\E\|y_n^k-y_n^{\ast,k}\|^2\right)\notag\\
    &\le\frac{2\V^0-2\V^{K}}{\alpha K}+\frac{8}{K}\sum_{k=0}^{K-1}\sum_{n=1}^NB(1-\frac{\mu_n\beta_n}{2})^k\E\|y_n^0-y_n^{0,\ast}\|^2\notag\\
    &+\frac{2}{K}\sum_{k=0}^{K-1}(\frac{L_v}{2}+\sum_{n=1}^NB^kC_{n,\sigma}^k)\alpha\sigma^2\notag\\
    &\le\frac{2\V^0-2\V^{K}}{\alpha K}+\frac{16}{K}\sum_{n=1}^N\frac{B}{\mu_n\beta_n}\E\|y_n^0-y_n^{0,\ast}\|^2\notag\\
    &+\frac{2}{K}\sum_{k=0}^{K-1}(\frac{L_v}{2}+\sum_{n=1}^NB^kC_{n,\sigma}^k)\alpha\sigma^2.\label{apx-nsm-res1}
\end{align}
According to \eqref{apx-Cnsigma}, we have
\begin{align}
    \sum_{k=0}^{K-1}C_{n,\sigma}^k\le\cO(\frac{K\beta_n}{\alpha}+\frac{1}{\mu_n^2\alpha}).
\end{align}
We can select step sizes satisfying the constraints $\alpha\le\frac{1}{2L_v(\omega_1+1)}$, $\beta_n\le\frac{\mu_n}{\ell_n^2}$, \eqref{apx-nsm-con-bn0}, \eqref{apx-nsm-con-an0}, \eqref{apx-nsm-con-b00}, \eqref{apx-nsm-con-a00} and $\alpha=\Theta(\frac{1}{\sqrt{K}}),\beta_n=\Theta(\frac{1}{\sqrt{K}})$, thus from \eqref{apx-nsm-res1}, we have
\begin{align*}
    &\frac{1}{K}\sum_{k=0}^{K-1}\left(\E\|v(x^k,y^{\diamond}_{1:N}(x^k))\|^2+\sum_{n=1}^N\ell_{v,n}\E\|y_n^k-y_n^{\ast,k}\|^2\right)\notag\\
    &\le\cO(\frac{1}{\alpha K}+\alpha+\sum_{n=1}^N\left(\frac{1}{K\beta_n}+{\beta_n}\right)+\frac{1}{K})\le\cO(\frac{1}{\sqrt{K}}),
\end{align*}
which completes the proof.
\vspace{-0.5em}
\section{Proof of Lemma \ref{lma-asp-SOBA}}
% \vspace{-0.5em}
i) Condition (b) $\Rightarrow$ Assumption \ref{asp-sub-sm}: Since $h_1(x,y_1)=\nabla_{y_1}g(x,y_1)$, $\mu_g$-strong monotonicity of $h_1(x,y_1)$ in $y_1$ follows from condition (b). Under (b), we have $\nabla_{y_1y_1}^2g\!\succeq\!\mu_g\bI$, and then $h_2(x,\!y_1,\!y_2)$ is $\mu_g$-strongly monotone in $y_2$.

ii) Conditions (a)--(b) $\Rightarrow$ Assumptions \ref{asp-lc-sub}--\ref{asp-lc-main}: $h_1(x,y_1)=\nabla_{y_1}g(x,y_1)$ is $\ell_g$-Lipschitz continuous, which satisfies condition (a) in Assumption \ref{asp-lc-sub}. Since $f(x,y_1)$ and $\nabla_{y_1}g(x,y_1)$ are Lipschitz continuous, $\nabla f(x,y_1)$, $\nabla_{xy_1}g(x,y_1)$ and $\nabla_{y_1y_1}g(x,y_1)$ are bounded, which together with Lipschitz continuity of $\nabla f(x,y_1)$ and $\nabla_{y_1y_1}g(x,y_1)$, implies that $h_2(x,y_1,y_2)$ satisfies condition (b) in Assumption \ref{asp-lc-sub}.
Under condition (b), for any $x$, $y^\ast_1(x)$, $y^\ast_2(x,y)$ both exist and are unique. For any $x$, $\|y^\diamond_{2}(x)\|=\|y^\ast_2(x,y^\ast(x))\|=\|[\nabla_{y_1y_1}^2g(x^2,y_1^2)]^{-1}\nabla_{y_1}f(x^2,y_1^2)\|\le\frac{\ell_f}{\mu_g}$. Then for any $\hat x$, $x$ and $y_{1}, y_2$, it holds that
\begin{align*}
    &\|v(\hat x,y^\diamond_{1}(\hat x),y^\diamond_{2}(\hat x))-v(x,y_{1},y_2)\|\\
    \overset{(a)}\le&\|\nabla_xf(\hat x,y^\diamond_{1}(\hat x))-\nabla_xf(x,y_{1})\|\\
    &+\|\nabla^2_{xy_1}g(\hat x,y^\diamond_{1}(\hat x))y^\diamond_{2}(\hat x)-\nabla^2_{xy_1}g(x,y_{1})y_2\|\\
    \overset{(b)}\le&\ell_f(\|\hat x-x\|+\|y^\diamond_{1}(\hat x)-y_1\|)+\|\nabla^2_{xy_1}g(x,y_{1})(y^\diamond_{2}(\hat x)-y_2)\|\\
    &+\|(\nabla^2_{xy_1}g(\hat x,y^\diamond_{1}(\hat x))-\nabla^2_{xy_1}g(x,y_{1}))y^\diamond_{2}(\hat x)\|\\
    \overset{(c)}\le&(\ell_f+\frac{\ell_f}{\mu_g})(\|\hat x-x\|+\|y^\diamond_{1}(\hat x)-y_1\|)+\ell_g\|y^\diamond_{2}(\hat x)-y_2\|,
\end{align*}
where $(a)$ is from the definition of $v$ and triangle inequality, $(b)$ is from the Lipschitz continuity of $\nabla f$ and triangle inequality, and $(c)$ is from the boundedness of $\nabla^2g,y_2^\diamond$ and Lipschitz continuity of $\nabla^2g$. This inequality implies Assumption \ref{asp-lc-main}.

iii) Conditions (a)--(c) $\Rightarrow$ Assumption \ref{asp-noise}: Assumption \ref{asp-noise} on $\psi_1^k$ holds immediately. Under conditions (a)--(c), according to \eqref{eq-soba-noise}, we have
\begin{align*}
    &\E[\|\psi_2^k\|^2|\cF^k]\le\hat\sigma^2(1+\|y_2^k\|^2)\notag\\
    \le&\hat\sigma^2(1+2\|y_2^k-y_2^\ast(x^k,y_1^k)\|^2+2\|y_2^\ast(x^k,y_1^k)\|^2)\notag\\
    \overset{(a_1)}\le&\hat\sigma^2+\frac{2\hat\sigma^2}{\mu_g^2}\|h_2(x^k,y_1^k,y_2^k)\|^2+\frac{2\hat\sigma^2\ell_f^2}{\mu_g^2},
\end{align*}
where $(a_1)$ is from the strongly monotonicity of $h_2$ and the boundedness of $y_2^\ast$ as shown before.
A similar result holds for $\xi^k$. Then Assumption \ref{asp-noise} holds.

iv) Conditions (b), (d) $\Rightarrow$ Assumption \ref{asp:main-sm}, and (b), (e)
    $\Rightarrow$ Assumption \ref{asp:main-nsm}:
Since $y^\diamond_1(x)=y^\ast_1(x)$ and $y^\diamond_2(x)=y_2^\ast(x,y^\diamond_1(x))=y_2^\ast(x,y^\ast_1(x))$, according to \eqref{eq-hyper-gradient}, we have $v(x,y_1^\diamond(x),y^\diamond_2(x))=\nabla F(x)$. Then the statement iv) follows immediately.

\section{Proof of Lemma \ref{lma-asp-CM}}
For any $n$, $h_n$ is $1$-strongly monotone in $y_n$ and $v$ is Lipschitz continuous, which implies Assumptions \ref{asp-sub-sm} and \ref{asp-lc-main}.

i) Condition (a) $\Rightarrow$ Assumption \ref{asp-lc-sub}: Lipschitz continuity of $h_n$ follows immediately from Lipschitz continuity of $\nabla f_n$, thus Condition (a) in Assumption \ref{asp-lc-sub} holds.

ii) Conditions (b)--(d) $\Rightarrow$ Assumption \ref{asp-noise}: Assumption \ref{asp-noise} on $\psi_n^k$ holds immediately under condition (b). In addition, we have
\begin{align*}
    &\E[\|\xi^k\|^2|\cF^k]\le\frac{1}{N}\sum_{n=1}^N\E[\|\cC_n(y_n^k)-y_n^k\|^2|\cF^k]\\
    \le&\frac{\omega}{N}\sum_{n=1}^N\|y_n^k\|^2\le\frac{2\omega}{N}\sum_{n=1}^N\|y_n^k-y_n^{k,\ast}\|^2+\frac{2\omega}{N}\sum_{n=1}^N\|\nabla f_n(x^k)\|^2\\
    \overset{(a_1)}\le&\frac{2\omega}{N}\sum_{n=1}^N\|h_n(x^k,y_{1:n}^k)\|^2+{2a\omega}\|\nabla F(x^k)\|^2+2\omega b^2\\
    \overset{(a_2)}\le&\frac{2\omega(1+2a)}{N}\sum_{n=1}^N\|h_n(x^k,y_{1:n}^k)\|^2+4a\omega\|v(x^k,y_{1:N}^k)\|^2+2\omega b^2,
\end{align*}
where $(a_1)$ is from the definition of $h_n$ and Conditions $(c)$, and $(a_2)$ is from $v(x,y_{1:n})+\frac{1}{N}\sum_{n=1}^Nh_n(x,y_{1:n})=\frac{1}{N}\sum_{n=1}^Ny_n-\frac{1}{N}\sum_{n=1}^N(y_n-\nabla f_n(x))=\nabla F(x)$. Therefore, Assumption \ref{asp-noise} on $\xi^k$ holds.

iii) $(e)$ $\Rightarrow$ Assumption \ref{asp:main-sm} and $(f)$ $\Rightarrow$ Assumption \ref{asp:main-nsm}: Since $v(x^k,y^\diamond_{1:N}(x))=\frac{1}{N}\sum_{n=1}^Ny^\diamond_n(x^k)=\frac{1}{N}\sum_{n=1}^N\nabla f_n(x^k)=F(x^k)$, the statement iii) holds immediately.
\bibliographystyle{IEEEtran}
\bibliography{refs}

\vfill

\end{document}